\documentclass[pdflatex,sn-mathphys-num]{sn-jnl}
\usepackage{graphicx}%
\usepackage{multirow}%
\usepackage{amsmath,amssymb,amsfonts}%
\usepackage{amsthm}%
\usepackage{mathrsfs}%
\usepackage[title]{appendix}%
\usepackage{xcolor}%
\usepackage{textcomp}%
\usepackage{manyfoot}%
\usepackage{booktabs}%
\usepackage{algorithm}%
\usepackage{algorithmicx}%
\usepackage{algpseudocode}%
\usepackage{listings}%

\usepackage{placeins}
\usepackage{subcaption}
\usepackage{siunitx}
\usepackage{tabularx}
\usepackage{enumitem}
\usepackage{svg}
\usepackage[percent]{overpic}
\usepackage[utf8]{inputenc}
\usepackage{hyperref} 

\usepackage{makecell}

\hypersetup{
    colorlinks=true,
    linkcolor=black,     
    citecolor=black,     
    filecolor=black,     
    urlcolor=black       
}
\newtheorem{proposition}{Proposition}

\begin{document}
\title[Article Title]{Learning-Based Dynamic Obstacle Avoidance for a UAV Using Only Three Range Sensors}
\author{\fnm{Mohammad Reza} \sur{Ranjbar Divkoti}}
\author{\fnm{A. Pedro} \sur{Aguiar}}
\affil{\orgdiv{SYSTEC-ARISE Research Center for Systems and Technologies}, \orgname{Faculty of Engineering, University of Porto}, \orgaddress{\postcode{4200-465}, \city{Porto}, \country{Portugal.}}
\\ \textbf{Corresponding author:} Mohammad Reza Ranjbar Divkoti (Email: up202111345@edu.fe.up.pt)}

\abstract{
We present a learning-based approach to kinodynamic online motion planning for an Unmanned Aerial Vehicle (UAV) operating at a fixed altitude in unknown dynamic environments, where real-time avoidance of both static and dynamic obstacles must be achieved under conditions of extreme partial observability. The UAV is controlled with a single degree of freedom (yaw only), resulting in constrained, nonholonomic motion similar to fixed-wing platforms.  

The proposed framework integrates a behavior grid map representation with Deep Reinforcement Learning (DRL), using Proximal Policy Optimization (PPO) for stable policy learning in continuous control. The key idea is the co-design of a state representation and control policy that enables reliable navigation using only three low-cost directional range sensors, without reliance on dense sensing modalities such as LiDAR or vision-based systems. The behavior grid map dynamically aggregates sparse measurements into a structured local representation that supports real-time decision-making for obstacle avoidance and target reaching.  \textcolor{black}{Extensive simulations across environments of varying sizes and obstacle densities demonstrate that the proposed standard and enhanced methods achieve higher success rates than PPO variants and Model Predictive Control (MPC) (94\% vs. 79--90\% in small-scale high-congestion scenarios, and 83\% vs. 62--71\% in large-scale high-congestion scenarios), while maintaining real-time performance. Real-world experiments across four scenarios further confirm practical feasibility, with consistent target-reaching behaviour and no collisions under the tested conditions.}}
\keywords{UAV Online Motion Planning, Deep Reinforcement Learning, Behavior Grid Map, Dynamic Obstacle Avoidance, PPO}


\maketitle

\section{Introduction}\label{sec:introduction}
Recent years have seen enormous advances in autonomous robotic vehicles, with a focus on efficiency and safety. An Unmanned Aerial Vehicle (UAV), commonly referred to as a drone, has become increasingly popular for a wide range of applications, including environmental research, package delivery, agricultural monitoring, and aerial surveillance~\cite{rezaee2024comprehensive}.

The problem of efficient online motion planning has been a major focus in UAV. Various motion-planning techniques have been thoroughly analyzed and categorized by recent studies~\cite{ait2022uav}. For instance in~\cite{lin2021autonomous}, an offline path planning method based on a map-based approach and vision-based obstacle detection for collision avoidance has been proposed. However, effectively addressing dynamic obstacle avoidance in UAV environments remains an open area for future research.

Drones have become increasingly popular over the last ten years due to their numerous applications. Effective path planning is necessary for autonomous drone operation to prevent collisions~\cite{ranjbar2025rcs}. The work in~\cite{gugan2023path} examines path generation and environmental representation methods, evaluating popular approaches and their cons. By analyzing influential publications, the study reveals current trends and offers insights for developing more robust path planners for autonomous drones.

Online motion planning algorithms can be broadly divided into two categories: traditional and learning-based approaches~\cite{xiao2022motion,gonzalez2015review}. Traditional algorithms, lacking a learning component, rely on predefined algorithms and rules to generate robot trajectories. While effective in static environments, these methods lack adaptability to dynamic scenarios. In contrast, learning-based algorithms, particularly those leveraging machine learning and DRL techniques, address these limitations~\cite{kaufmann2023champion, he2021explainable, he2020integrated, angulo2022policy}.

Deep Reinforcement Learning (DRL) has drawn a lot of interest. The application of DRL in this context has been thoroughly investigated by researchers~\cite{aradi2020survey, liu2020path, yan2020towards, chiang2019rl, wang2019autonomous, choi2021robust, hodge2021deep, huang2019deep,wang2020deep, kang2024uav}. DRL can learn challenging online motion planning tasks and adapt to dynamic environments. Researchers are investigating methods that combine sensor fusion~\cite{ali2019path}, optimal control~\cite{zheng2023constrained, gao2023fixed}, and graph theory~\cite{ma2022optimal} to further improve the capabilities of DRL in online motion planning. This integration can enhance the decision-making process and offer a more thorough understanding of the environment. The integration of low-cost sensors and smart software for various vehicle monitoring scenarios, sensor error modeling, and enhancing semantic scene understanding should be the main areas of future research~\cite{gyagenda2022review}.

Several advanced learning algorithms have been introduced, including Asynchronous Advantage Actor-Critic (A3C)~\cite{ryou2018applying, mnih2016asynchronous}, Advantage Actor-Critic (A2C), Trust Region Policy Optimization (TRPO)~\cite{schulman2015trust}, and Proximal Policy Optimization (PPO)~\cite{schulman2017proximal, qi2022uav, zhu2022auv, chikhaoui2022ppo, li2023uav}. To stabilize and speed up convergence, A3C and A2C used multi-thread techniques, while TRPO and PPO added trust region constraints and adaptive penalties. 

Table~\ref{tab:discussion} provides a summary of key research papers and their contributions, offering insights into the current state of the field.

\begin{table*}
    \centering
    \caption{A summary of key research papers and their contributions}
    \label{tab:discussion}
    \small
    \begin{tabularx}{\linewidth}{XXXX}
        \hline
        \textbf{\textcolor{black}{Category}} & \textbf{Pros} & \textbf{Cons} & \textbf{Open Problem} \\
        \hline
        \begin{itemize}[leftmargin=0.1cm, labelsep=0cm, itemsep=0cm]
            \item~\textit{\textcolor{black}{UAV Path Planning and Obstacle Avoidance~\cite{ait2022uav,lin2021autonomous,ranjbar2025rcs,qu2020novel,ge2020path,zhu2021deep}}}
        \end{itemize}
        & 
        \begin{itemize}[leftmargin=0.1cm, labelsep=0cm, itemsep=0cm]
            \item~Comprehensive analyses of path-planning approaches.
            \item~Innovations in obstacle detection for UAV.
        \end{itemize} 
        & 
        \begin{itemize}[leftmargin=0.15cm, labelsep=0cm, itemsep=0cm]
            \item~Challenges in dynamic obstacle avoidance.
            \item~Limitations of traditional online motion planning algorithms.
        \end{itemize} 
        & 
        \begin{itemize}[leftmargin=0.1cm, labelsep=0cm, itemsep=0cm]
            \item~Develop real-time adaptive solutions to dynamic obstacles.
        \end{itemize} 
        \\
        \begin{itemize}[leftmargin=0.1cm, labelsep=0cm, itemsep=0cm]
            \item~\textit{\textcolor{black}{DRL-Based Motion Planning and Navigation~\cite{aradi2020survey,liu2020path,yan2020towards}}}
        \end{itemize}
        & 
        \begin{itemize}[leftmargin=0.1cm, labelsep=0cm, itemsep=0cm]
            \item~Potential for complex online motion planning with DRL.
            \item~Integration of sensor fusion and optimal control.
        \end{itemize} 
        & 
        \begin{itemize}[leftmargin=0.15cm, labelsep=0cm, itemsep=0cm]
            \item~Challenges in accounting for environmental disturbances.
            \item~Issues with poor generalization in DRL.
        \end{itemize} 
        & 
        \begin{itemize}[leftmargin=0.1cm, labelsep=0cm, itemsep=0cm]
            \item~Integrate low-cost hardware sensors with smart software to enhance environmental understanding.
        \end{itemize} 
        \\
        \hline
    \end{tabularx}
\end{table*}

Despite these advances, a significant gap remains in deploying learning-based UAV navigation frameworks in real-world settings with low-cost, low-perception sensors. Many existing approaches assume access to rich sensory data such as 3D LiDAR, dense RGB-D maps, or full environment reconstructions~\cite{huang2025gtrxl, almahamid2024viznav}, which are often expensive and computationally intensive. This reliance makes them impractical for cost-sensitive or resource-constrained UAV applications.

In contrast to these works, this paper investigates a lightweight and cost-effective online motion planning framework that uses only three rangefinder sensors positioned in the front, left, and right directions of the UAV. These inexpensive sensors provide sparse distance readings, and the system must make safe navigation decisions using this limited information. To compensate for the perception constraints, we integrate these readings into a behavior grid map, a compact, local, and dynamically updated spatial representation that feeds a PPO-based policy to generate control actions in real-time.

Although behavior or occupancy maps have been used in DRL settings, they are typically supported by rich and dense observations~\cite{fan2025flying, liu2024flexible}. Our work challenges this assumption by demonstrating that strong learning-based navigation performance is achievable even with minimal and noisy sensory input if combined with a task-adapted behavior grid and structured low-level policy design. While recent works leverage high-resolution sensors and dense maps (see Table~\ref{tab:comparison}), our method uniquely demonstrates DRL-powered obstacle avoidance using sparse, minimal sensing.

\begin{table*}
\centering
\caption{Comparison of DRL-based UAV online motion planning methods under different sensor modalities}
\label{tab:comparison}
\small
\begin{tabularx}{\linewidth}{XXXXXX}
\hline
\textbf{Ref (Year)} & \textbf{Sensors} & \textbf{Perception / Input} & \textbf{Map / Representation} & \textbf{DRL Algorithm} & \textbf{Environment / Hardware}\\
\hline
Fan et al. (2025)~\cite{fan2025flying} & 3D LiDAR & Point-cloud & 2D obstacle map via LiDAR encoder & PPO-/end-to-end & Simulated and real UAV, dynamic environment \\
Liu et al. (2024)~\cite{liu2024flexible} & Airborne LiDAR & 3D scans & Occupancy grid for formation & PPO-style & Simulated, multi-UAV \\
Huang et al. (2025)~\cite{huang2025gtrxl} & Vision, sensors, and LiDAR & Multimodal (RGB, depth) & Transformer-enhanced perception & SAC and GTrXL & Simulated, dynamic environment \\
Almahamid (2024)~\cite{almahamid2024viznav} & RGB and depth \newline camera & Voxel grids & 3D voxel map & Off-policy DRL & 3D environments simulated \\
Wang et al. (2024)~\cite{wang2024end} & 2D LiDAR & Relative distances & End-to-end and potential field & PF-SAC & Wheeled robot, but LiDAR-based\\
\hline
\textbf{This Work (2025)} & 3 rangefinders (front, left, and right) & Sparse range readings & Local behavior grid map & PPO & Simulated and real UAV, 3 sensors \\
\hline
\end{tabularx}
\end{table*}

In addition to the severe perception constraints, the control authority of the UAV is deliberately limited in this work. The system is modeled with a single degree of freedom, where only the yaw rate can be adjusted while the forward velocity remains constant. This constraint results in non-holonomic motion behavior similar to that of a fixed-wing UAV, which cannot perform in-place rotations or lateral movements. As a consequence, the UAV must continuously move forward and commit to its heading decisions in advance, making online obstacle avoidance and target reaching significantly more challenging, especially in cluttered and dynamic environments.

To the best of our knowledge, this is the first work to successfully integrate a behavior grid map with PPO for dynamic online UAV motion planning, considering a single degree-of-freedom model and relying solely on three sparse rangefinder sensors, without the use of LiDAR, stereo vision, or global maps. This setup makes our approach well-suited for low-power, real-time deployment on a resource-limited UAV operating in dynamic, cluttered environments. 

\textcolor{black}{
\subsection{Methodological Rationale}\label{subsec:justification}
This section provides a structured rationale for the proposed framework by positioning it relative to classical navigation approaches and highlighting the key design choices that enable operation under extreme partial observability and sparse sensing conditions.
}

\textcolor{black}{
\textbf{Reactive Methods:}
Classical reactive techniques such as Artificial Potential Fields (APF)~\cite{khatib1986real} and Tangent Bug~\cite{kamon1998tangentbug} are computationally efficient and suitable for real-time applications. However, they are susceptible to local minima and exhibit degraded performance under sparse sensing, where limited range measurements lead to fragmented obstacle representations, often resulting in oscillatory behavior and unreliable avoidance decisions.
}

\textcolor{black}{
\textbf{Predictive Methods:}
Predictive approaches, including the Dynamic Window Approach (DWA)~\cite{fox2002dynamic} and Model Predictive Control (MPC)~\cite{mayne2000constrained}, improve safety by explicitly considering future trajectories. DWA is effective for short-horizon local planning but is sensitive to incomplete perception and may fail to capture the UAV’s geometric footprint under sparse feedback. MPC is adopted as the primary baseline, as it represents a standard optimization-based framework for constrained systems. While MPC provides theoretically optimal control under accurate environment models, it relies on sufficiently rich geometric information. Under sparse sensing, individual range measurements are treated as isolated observations, leading to blind regions between sensing directions and under-approximation of obstacle boundaries. This limitation provides a consistent reference to evaluate the proposed approach under identical perception constraints.
}

\textcolor{black}{
\textbf{DRL Policy (PPO):}
We employ Proximal Policy Optimization (PPO) due to its stability and robustness in continuous control settings~\cite{tang2025deep}. While DRL-based navigation is well established, the contribution of this work lies in the co-design of the policy and the state representation tailored to operate under extreme partial observability, rather than in the learning algorithm itself.
}

\textcolor{black}{
\textbf{Behavior Grid Map:}
Standard Local Occupancy Grid Maps (LOGM)~\cite{thrun2002probabilistic, ren2024rogmap} rely on dense observations to construct coherent obstacle representations. In the sparse sensing regime, such representations degenerate into disconnected measurements, commonly referred to as the \textit{Swiss cheese effect}. The proposed behavior grid map addresses this limitation by introducing a spatial influence kernel that transforms discrete sensor readings into continuous regions encoding approximate obstacle influence and clearance. This representation does not aim to reconstruct exact geometry, but rather provides a structured prior that regularizes sparse observations and enables the policy to reason about obstacle proximity despite missing information. Notably, many existing works consider multi-beam LiDAR (e.g., 16-beam) as a minimal viable sensing setup~\cite{tao2022lidar}, whereas the present work operates with only three directional measurements, representing a significantly more constrained perception regime.
}

\textcolor{black}{
\textbf{Sinusoidal Motion Strategy:}
The sinusoidal motion strategy is formulated as an active perception mechanism~\cite{bajcsy1988active}. While exploratory motion is commonly used for mapping in unknown environments~\cite{placed2023survey}, here it is integrated as a task-driven primitive for real-time obstacle avoidance. By inducing controlled oscillations in the UAV heading, the sensing directions are swept across a wider spatial region over time, effectively increasing angular coverage. This behavior can be interpreted as a form of temporal sensing augmentation, partially mitigating the limitations of sparse measurements without introducing additional hardware or computational overhead.
}

\subsection{Main Contributions}
The main contributions of this paper are summarized as follows:

\begin{itemize}
    \item \textcolor{black}{A learning-based UAV navigation framework designed for operation under extreme partial observability, integrating a behavior grid map with PPO. The approach enables kinodynamic motion planning under a single degree-of-freedom (yaw-only) control model using only three sparse rangefinder measurements, departing from prior works that rely on dense sensing modalities. The framework combines obstacle and target encoding, dynamic grid updating, and compact state extraction to support real-time decision-making with low computational overhead.}

    \item \textcolor{black}{A lightweight and deployable system tailored for resource-constrained platforms, requiring minimal memory and computational resources during execution. The approach relies on low-cost range sensors and demonstrates suitability for onboard implementation without the need for LiDAR or vision-based perception systems.}

    \item \textcolor{black}{A perception-enhancement strategy based on sinusoidal motion, formulated as an active sensing mechanism to improve environmental coverage under sparse observations. The enhanced method introduces controlled oscillations during forward motion to increase effective perception, improving navigation robustness while introducing a measured increase in energy consumption.}

    \item \textcolor{black}{Extensive validation in both simulation and real-world environments. Simulations across multiple environment sizes and obstacle densities show that the proposed standard and enhanced methods achieve higher success rates than PPO variants and MPC baselines (94\% vs.\ 79--90\% in small-scale high-congestion scenarios, and 83\% vs.\ 62--71\% in large-scale high-congestion scenarios), while maintaining real-time performance. Real-world experiments across four scenarios further demonstrate practical feasibility, with consistent target-reaching behaviour and no collisions under the tested conditions.}

\end{itemize}

The remaining paper is organized as follows. The problem is formally described in Section~\ref{sec:problem-definition}. The online motion planning algorithm to solve the problem is presented in Section~\ref{sec:algorithm}. In Section~\ref{sec:result}, the results are discussed. Section~\ref{sec:conclusion} concludes the paper.

\section{Problem Definition}\label{sec:problem-definition}
This paper addresses the problem of planning a feasible kinodynamic trajectory for a UAV while avoiding static and dynamic obstacles. Let $W = \{O, x, y, z\}$ be the inertial frame, where the axes are oriented in the north, west, and upward directions. Similarly, let $B = \{O_b, x_b , y_b , z_b \}$ represent the body-fixed frame. By applying classical Newtonian or Lagrangian methods, the quadcopter's dynamics can be formulated as follows~\cite{mahony2012multirotor}
\begin{equation}
\label{equ:dynamics}
\begin{aligned}
    \dot{\xi} &= R_{WB}(\eta) V_B \\
    m \dot{V}_B &= - m \omega \times V_B + m R_{WB}^T(\eta) g_W + T_B + F_d  \\
    \dot{\eta} &= T(\eta) \omega  \\
        I \dot{\omega} &= - \kappa \omega - \omega \times (I \omega) + \tau_B
\end{aligned}
\end{equation}
where $\xi = [\xi_x, \xi_y, \xi_z]^T$ is the position in the inertial frame $W$, $\eta = [\phi, \theta, \psi]^T$ represents the Euler angles (roll $\phi$, pitch $\theta$, and yaw $\psi$), $R_{WB}(\eta)$ is the rotation matrix from the body-fixed frame to the inertial frame, $T(\eta)$ is the transformation matrix from angular velocities to Euler angle rates, $V_B = [u, v, w]^T$ is the linear velocity in the body-fixed frame, $\omega = [p, q, r]^T$ is the angular velocity in the body-fixed frame, $m$ is the quadcopter mass, $I$ is the body-fixed inertia matrix,
$
T_B =
\begin{bmatrix}
T_x &
T_y &
T_z
\end{bmatrix}^T
=
\begin{bmatrix}
0 &
0 &
k \sum_{i=1}^{4} \Omega_i^2
\end{bmatrix}^T
$
is the total thrust vector acting along the body-fixed $z$-axis, \( k \) is the thrust coefficient, \( \Omega_i \) is the angular velocity of rotor \( i \), 
\(
\tau_B =
\begin{bmatrix}
\tau_\phi \\
\tau_\theta \\
\tau_\psi
\end{bmatrix}
=
\begin{bmatrix}
l k (\Omega_1^2 - \Omega_2^2 - \Omega_3^2 + \Omega_4^2) \\
l k (\Omega_1^2 + \Omega_2^2 - \Omega_3^2 - \Omega_4^2) \\
b (\Omega_1^2 - \Omega_2^2 + \Omega_3^2 - \Omega_4^2)
\end{bmatrix}
\)
represents the control torques, \( l \) is the distance from each rotor to the center of mass, \( b \) is the yaw torque coefficient, $\kappa = \text{diag}([\kappa_1, \kappa_1, \kappa_2])$ is the damping ratio matrix, $g_W = [0, 0, g]^T$ is the gravitational acceleration vector, $F_d = -\Lambda \left( V_B + R_{WB}^T(\eta) V_w \right)$ is the aerodynamic drag force experienced by the quadcopter in the body-fixed frame, $\Lambda = \text{diag}([\Lambda_1, \Lambda_1, \Lambda_2])$ is the drag coefficient matrix, and $V_w = [u_w, v_w, w_w]^T$ is the wind velocity in the inertial frame.

\textcolor{black}{
In this work, the UAV operates under its full nonlinear dynamics, which are regulated by low-level proportional–integral–derivative (PID) controllers. A constant altitude reference (\(z\)) and a constant forward velocity reference (\(u\)) are provided to the controller, while the lateral velocity reference is constrained to \(v = 0\). The proposed algorithm generates the yaw angular velocity command (\(r\)), enabling the UAV to adjust its heading for navigation and obstacle avoidance, while the PID controllers track the commanded references.
}
\\
\\
\textit{The goal of this paper is to design a kinodynamically feasible online motion planning and control strategy that enables a UAV to safely navigate environments with both static and dynamic obstacles, using only sparse, low-cost range sensing.}
\\
\\
\textcolor{black}{The main challenge arises from the severely limited sensing capabilities of the UAV. In particular, for simplicity and to match the experimental hardware configuration, the platform is equipped with only three low-cost rangefinder sensors measuring obstacle distances in the forward, left, and right directions. Unlike dense sensing modalities such as LiDAR, these sensors provide one-dimensional measurements along narrow beams with a maximum range of $c$ meters, resulting in only three scalar observations at each time step. This sparse sensing configuration leads to an extremely limited and fragmented perception of the environment, where obstacle boundaries are under-sampled and may appear as disconnected detections, significantly reducing the ability to reliably infer spatial structure and avoid collisions. In addition, to remain consistent with this sensing setup, we assume that no dynamic obstacles approach the UAV from behind.}

Figure~\ref{fig:problem} illustrates this setup, showing the drone equipped with three rangefinder sensors (forward, left, and right), represented by red lines. A notable limitation is highlighted: while the drone can detect the blue part of the ball as an obstacle, it fails to detect the box positioned in its blind spot. This constraint poses significant challenges for effective obstacle detection and navigation. 

\begin{figure}
    \centering
    \includegraphics[width=0.5\linewidth]{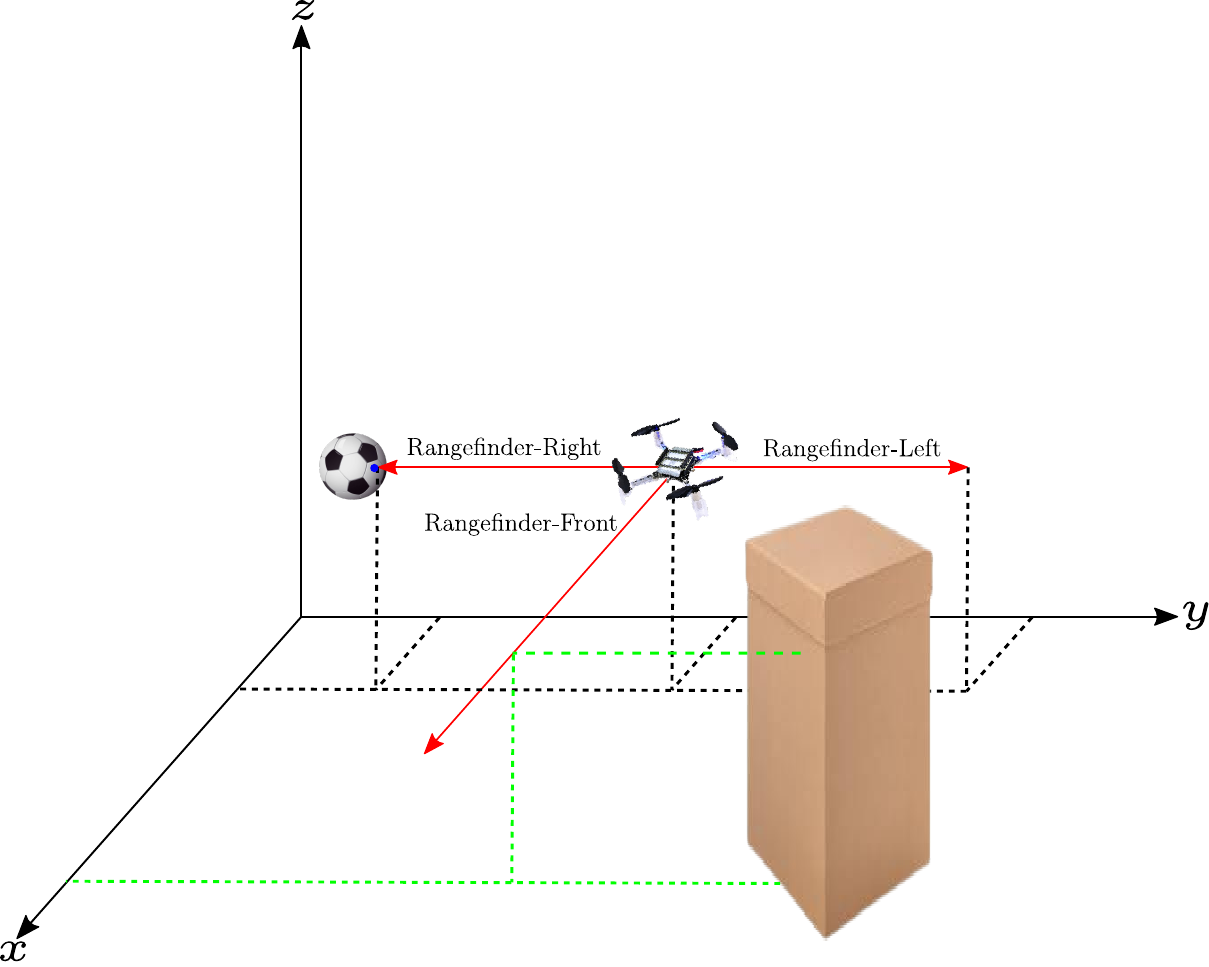}
    \caption{Illustration of the setup: the drone is equipped with three rangefinder sensors (forward, left, and right), shown as red lines. The right sensor detects only the blue part of the ball as an obstacle, while all sensors fail to detect the box, as it is positioned in the drone's blind spot.}
    \label{fig:problem}
\end{figure}

\section{Online Motion Planning Algorithm}\label{sec:algorithm}
This section provides a detailed description of the online motion planning algorithm. To address the problem defined in Section~\ref{sec:problem-definition}, a DRL approach is adopted. DRL is particularly well-suited to the kinodynamic online motion planning problem in dynamic environments where handcrafted rules or classical planners may struggle with real-time decision-making under uncertainty. DRL enables the UAV to learn adaptive navigation strategies directly from interaction with the environment, handling high-dimensional state spaces without requiring an explicit model of the UAV's dynamics or the environment.

The proposed method relies on a model-free DRL framework to handle the problem of incomplete and inaccurate modeling of the environment and obstacle dynamics. Among model-free methods, PPO was selected due to its favorable trade-off between performance and training stability. PPO maintains reliable updates through clipping mechanisms, making it particularly robust for continuous control problems and environments with noisy sensory input, conditions that reflect our UAV’s limited sensing capabilities.
The PPO agent learns a navigation policy that maps behavior grid-based states, derived from low-cost range sensor inputs, to velocity commands, ensuring safe and efficient motion toward the goal. Further implementation details are provided in the following subsections.

\subsection{Behavior Grid Map Algorithm}\label{subsec:behavior_grid_map_algorithm}
Building on the concept of using low-cost hardware with smart software, this paper tackles the challenge of detecting unknown obstacles using only three distance values (forward, left, and right) obtained from low-cost sensors. To address this limitation, a behavior grid map is introduced to determine the robot’s state at each time step $t$ and keep a perception memory of the environment.

To ensure memory efficiency, at each time step $t$, a portion of the environment is represented as a grid map where each cell contains a probability value indicating the likelihood of an obstacle and/or a target. The robot is always located at the center of the grid, and the grid forms around it. The grid size is a fixed value based on the rangefinder's capacity. The grid is updated at each time step according to the robot's movements and sensor readings.

The algorithm for generating the grid map consists of three parts: marking obstacles and the target, updating the grid, and extracting states. These three parts are described in detail below.

\subsection*{Marking Obstacles and the Target}
In this section, the positions of the obstacles and the target are marked in the grid map based on the robot's position and its heading angle.

Consider the driving car perspective (see Figure~\ref{fig:index_obstacle_d}), where the robot is located constantly at position \( \begin{bmatrix} 0 \\ 0 \end{bmatrix} \) in the grid map \( G \in \mathbb{R}^{n \times n} \), where \( n \) is the size of the grid in both dimensions. Let \( \psi_t \) be the robot's heading angle at time step \( t \). Let \( d^{(l)}_{t}, d^{(f)}_{t}, d^{(r)}_{t} \) be the distances measured by the robot's rangefinder sensors to the closest left, forward, and right obstacles at time step \( t \), respectively. Let \( d^{(\iota)}_{t} \) be the distance between the robot and the target at time step \( t \). Define the angles 
$
\gamma^{(l)}_{t} = \psi_t - \bar{\psi}, 
\gamma^{(f)}_{t} = \psi_t, \text{ and }
\gamma^{(r)}_{t} = \psi_t + \bar{\psi}
$ as the angle between the robot's direction and the left, forward, and right closet obstacles at time step \( t \), respectively, and \( \gamma^{(\iota)}_{t} \) as the angle between the robot's direction and the target at time step \( t \). 
With this setup, at each time step \( t \) the map $G$ is constructed by including a Gaussian term $
\mathcal{N}\left(\begin{bmatrix}\mu_i \\ \mu_j\end{bmatrix}^{(k)}_t, \sigma^2\right)
$  with mean \( \begin{bmatrix}\mu_i \\ \mu_j\end{bmatrix}^{(k)}_t\)   (\(k \in \{l, f, r, \iota\} \)) and variance \( \sigma^2 \) at corresponding row and column indices of \( G \) according with the distance measured by the robot’s rangefinder sensors to the obstacles or target. The grid map \( G\) is initialized with all cell values set to \( 0 \), and it is updated for each time step $t$ based on Algorithm~\ref{alg:grid_construction}.

\begin{figure}
    \centering
    \includegraphics[width=0.5\linewidth]{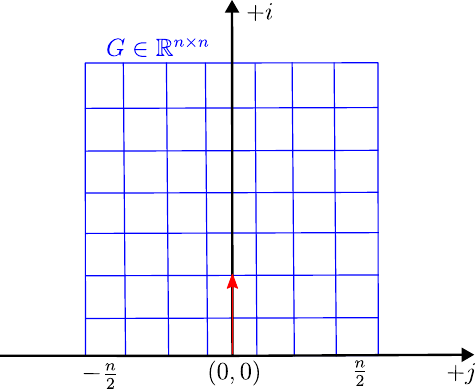}
    \caption{Driving car perspective of the grid map $G$.}
    \label{fig:index_obstacle_d}
\end{figure}

Algorithm~\ref{alg:grid_construction} constructs the behavior grid map by projecting rangefinder measurements from the robot’s local frame into a grid aligned with the robot’s heading. For each detected obstacle and the target, a Gaussian distribution with mean
\( \begin{bmatrix} \mu_i \\ \mu_j \end{bmatrix}^{(k)}_t \) and variance \( \sigma^2 \) is generated and mapped onto the grid. This distribution illustrates the probability of the presence of obstacles and the target. Figure~\ref{fig:mark-obstacle} shows this distribution.

\begin{algorithm}
\small
\caption{Construction of the Behavior Grid Map at Time Step $t$}
\label{alg:grid_construction}
\begin{algorithmic}[1]
\State \textbf{Inputs:} Grid size \( n \), scaling factor \( f \), robot heading \( \psi_t \), variance \( \sigma^2 \),
sensor distances \( d^{(k)}_t \), sensor angles \( \gamma^{(k)}_t \),
 \( k \in \{l \text{ (left)}, f \text{ (front)}, r \text{ (right)}, \iota \text{ (target)}\} \)
\State \textbf{Output:} Behavior grid map \( G_t \in \mathbb{R}^{n \times n} \) at time step $t$

\State Define rotation matrix
\(
R(-\psi_t)=
\begin{bmatrix}
\cos(-\psi_t) & -\sin(-\psi_t) \\
\sin(-\psi_t) & \cos(-\psi_t)
\end{bmatrix}
\)

\For{each \( k \in \{l, f, r, \iota\} \)}
    \State Compute rotated position
    \(
    \begin{bmatrix}
    G_{x} \\
    G_{y}
    \end{bmatrix}^{(k)}_t
    =
    R(-\psi_t)
    \begin{bmatrix}
    d^{(k)}_t \cos(\gamma^{(k)}_t) \\
    d^{(k)}_t \sin(\gamma^{(k)}_t)
    \end{bmatrix}
    \)

    \State Compute grid indices
    \(
    \begin{bmatrix}
    \mu_i \\
    \mu_j
    \end{bmatrix}^{(k)}_t
    =
    \left\lfloor
    f
    \begin{bmatrix}
    G_{x} \\
    G_{y}
    \end{bmatrix}^{(k)}_t
    \right\rfloor
    \)

    \If{\( k \in \{l, f, r\} \)  and \( 0 \leq \mu_i,\mu_j \leq n\)}
        \State Add a positive Gaussian distribution centered at
        \( \begin{bmatrix}
        \mu_i \\
        \mu_j
        \end{bmatrix}^{(k)}_t \) with variance~\( \sigma^2 \)
        \(
        G_t \gets \mathbb{F} (G_t, \mathcal{N}\left(
        \begin{bmatrix}
        \mu_i \\
        \mu_j
        \end{bmatrix}^{(k)}_t, \sigma^2
        \right) )
        \) \Comment{$\mathbb{F}$ is a general cell-wise function that, for example, can be the maximum operator.}
    \EndIf
    \If{\( k = \iota  \)  and \( 0 \leq \mu_i,\mu_j \leq n\)}
        \State Add a negative Gaussian distribution centered at
        \( \begin{bmatrix}
        \mu_i \\
        \mu_j
        \end{bmatrix}^{(k)}_t \) with variance \( \sigma^2 \)
        \(
        G_t \gets \mathbb{F} (G_t, - \mathcal{N}\left(
        \begin{bmatrix}
        \mu_i \\
        \mu_j
        \end{bmatrix}^{(k)}_t, \sigma^2
        \right))
        \)
    \EndIf
\EndFor

\State \Return \( G_t \)
\end{algorithmic}
\end{algorithm}

\begin{figure}
    \centering
    \includegraphics[width=0.5\linewidth]{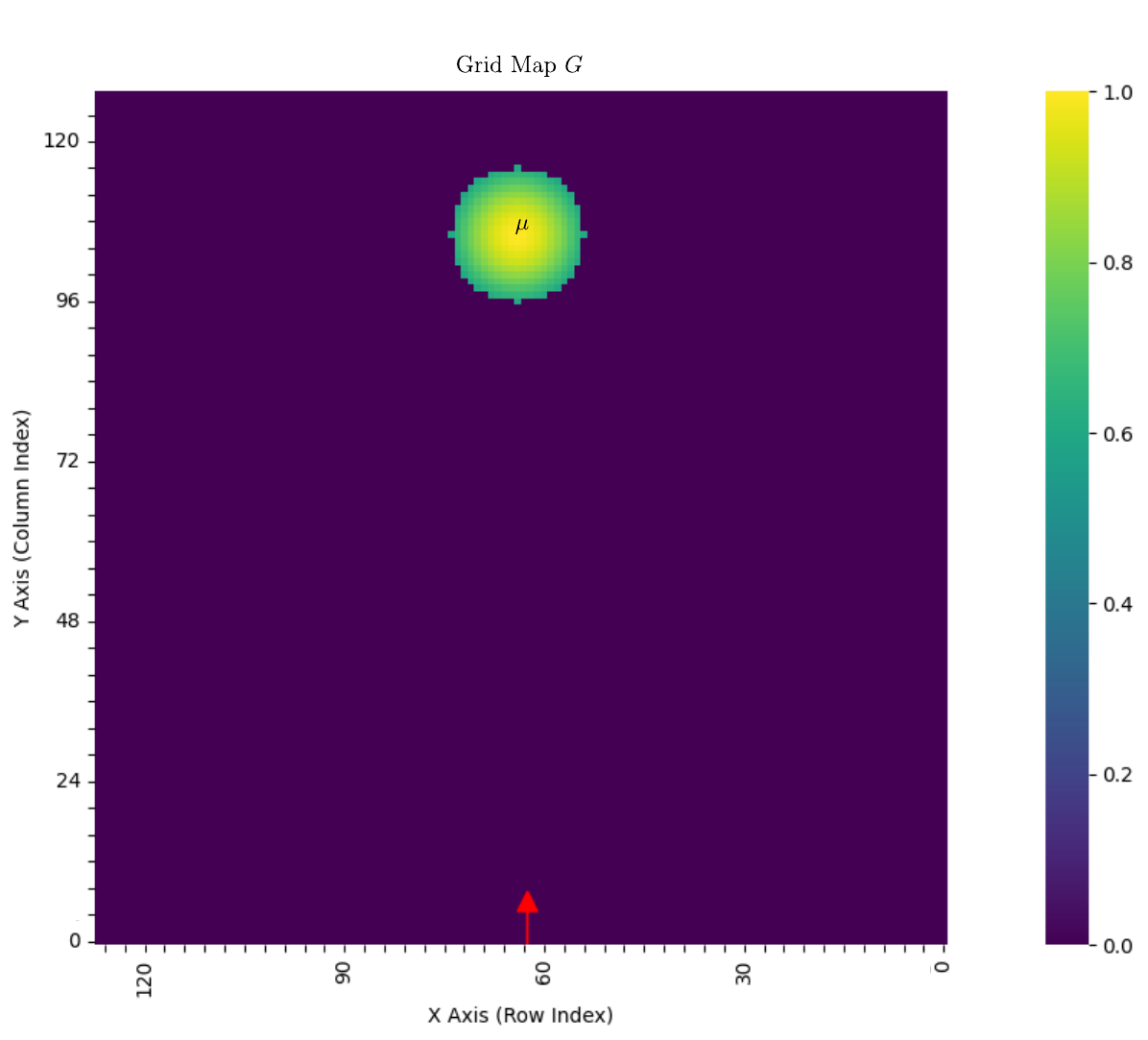}
    \caption{The red arrow indicates the robot’s direction. A normal distribution with mean \( \mu = \begin{bmatrix}
            \mu_i \\
            \mu_j
            \end{bmatrix}^{(f)}_t \) at the location of the obstacle and variance \( \sigma^2\) is formed, representing its presence.}
    \label{fig:mark-obstacle}
\end{figure}

\textcolor{black}{
The function $\mathbb{F}(\cdot,\cdot)$ is a cell-wise aggregation operator that maps the current grid state and newly available sensory information to an updated grid representation. Its first argument is the existing grid map $G_t$, and its second argument is a Gaussian distribution $\mathcal{N}\!\left(\begin{bmatrix}\mu_i \\ \mu_j\end{bmatrix}, \sigma^2\right)$ centered at the projected location of a detected obstacle or target. The operator $\mathbb{F}$ acts element-wise over the grid and can be instantiated, for example, as a maximum operator, i.e., $\mathbb{F}(a,b)=\max(a,b)$, which preserves the highest activation at each cell and reinforces obstacle regions in the map.
}

To distinguish between obstacles and the target, negative values are assigned to cells representing the target, while positive values are assigned to cells representing obstacles. If we consider \( G \) as a landscape, obstacles appear as 'hills' and the target as a 'valley,' guiding the robot to avoid the hills and move toward the valley.

\subsection*{Updating the Grid Map}
To account for the robot's positional changes, the values in the grid map \( G \) must be updated accordingly. The detailed steps are outlined in Algorithm~\ref{algorithm:shift_grid_map}.

\begin{algorithm}
\small
\caption{Grid Map Updating}
\label{algorithm:shift_grid_map}
\begin{algorithmic}[1]
\State \textbf{Inputs:} Grid map \( G_t \in \mathbb{R}^{n \times n} \) at time step $t$.
\State Create a temporary grid map \( \mathbb{G} \in \mathbb{R}^{n \times n} \), initializing its all cell values to \( 0 \).
\For{each cell \( G^{(i,j)}_{t}\) in \( G_t \)}
    \State Compute \( G^{(i,j)}_{t+1} \) using \( G^{(i,j)}_{t} \) in Equation~\eqref{equ:shifting_c}.
    \State
    $
    \mathbb{G}^{(i,j)}_t \gets  \rho G^{(i,j)}_{t+1}
    $ \Comment{Update the corresponding cell value of \( \mathbb{G} \), denoted by \( \mathbb{G}^{(i,j)}_t \), with forgetting factor \( \rho \)}
\EndFor
\State \( G \gets \mathbb{G} \)
\State \Return \( G \)
\end{algorithmic}
\end{algorithm}

Each cell in \( G \), represented by \( G^{(i,j)}_t \) at time step \( t \), is first transformed into a standard perspective at time step \( t \) using the rotation matrix  
$
R(\psi_{t})~=~\begin{bmatrix} 
\cos{(\psi_{t})} & -\sin{(\psi_{t})} \\ 
\sin{(\psi_{t})} & \cos{(\psi_{t})} 
\end{bmatrix}
$ where $\psi_t$ is the robot’s heading angle at time step $t$ and \( i \) and \( j \) are the corresponding column and row numbers of \( G \). The cells after transformation are denoted as \( \bar{G}^{(i,j)}_{t} \) that is given by
\[
\bar{G}^{(i,j)}_t = R(\psi_t) G^{(i,j)}_t, \quad \forall \{i, j\} \in \mathbb{Z}, \quad 0 \leq i, j \leq n.
\]

Let \( \Delta x_t \) and \( \Delta y_t \) represent the changes in the robot's \( x \)- and \( y \)-positions at time step \( t \), respectively:
\[
\Delta x_t = x_t - x_{t-1}, \quad \Delta y_t = y_t - y_{t-1},
\]
where \( x_t, y_t \) are the robot's current positions at time step \( t \), and \( x_{t-1}, y_{t-1} \) are its positions at the previous time step.  

By multiplying the scaling factor \( f \) by these changes, the changes in the grid map \( G \)'s column and row at time step \( t \), denoted as \( \begin{bmatrix} \Delta i \\ \Delta j \end{bmatrix}_t \), are obtained:
\[
\begin{bmatrix} \Delta i \\ \Delta j \end{bmatrix}_t = \left\lfloor f \begin{bmatrix} \Delta x_t \\ \Delta y_t \end{bmatrix} \right\rfloor.
\]

The transformed cell values of \( G \) in the standard perspective at time step \( t+1 \), denoted as \( \bar{G}^{(i,j)}_{t+1} \), is:
\[
\bar{G}^{(i,j)}_{t+1} = \bar{G}^{(i,j)}_{t} - \begin{bmatrix} \Delta i \\ \Delta j \end{bmatrix}_t, \quad \forall \{i, j\} \in \mathbb{Z} \quad 0 \leq i, j \leq n,
\]
where $\bar{G}^{(i,j)}_{t}$ is each cell of the $\bar{G}_t$ at time step $t$.
Next, \( \bar{G}_{t+1} \) must be transferred back to the driving car perspective. To achieve this, the coordinates are rotated using the rotation matrix  
$
R(-\psi_t) = \begin{bmatrix} 
\cos{(-\psi_t)} & -\sin{(-\psi_t)} \\ 
\sin{(-\psi_t)} & \cos{(-\psi_t)} 
\end{bmatrix}
$.

Each cell in \( G \) is associated with a forgetting factor \( \rho \in (0, 1) \), which determines how the probability in each cell decreases over time as \( G \) updates. $G^{(i,j)}_{t + 1}$ is represented the cell values in the grid map $G$ for the next time step $t + 1$. 


\textcolor{black}{
\begin{proposition}
\label{prop:grid_update}
Let \( G_t \in \mathbb{R}^{n \times n} \) denote the robot-centered grid map at time \( t \), and let \( \Delta x_t, \Delta y_t \) and \( \psi_t \) represent the translational and rotational motion of the UAV between time steps \( t \) and \( t+1 \). Then, the grid update rule
\begin{equation}
\label{equ:shifting_c}
G_{t+1}^{(i,j)} = \rho \left[ \, R(-\psi_t)\, \bar{G}_{t+1} \,\right]^{(i,j)}, 
\quad \forall (i,j) \in \{0,\dots,n\}^2,
\end{equation}
correctly represents the evolution of the environment in the robot-centered frame, where \( \bar{G}_{t+1} \) is obtained by expressing \( G_t \) in the world-aligned frame and applying the translational shift induced by \( (\Delta x_t, \Delta y_t) \), and \( \rho \in (0,1) \) is an element-wise forgetting factor.
\end{proposition}
}

\begin{proof}
We show that the update rule preserves a consistent robot-centered representation of the environment under the UAV motion.

Consider a grid map \( G_t \) expressed in the robot-centered frame at time \( t \). By construction, the robot is located at the center of the grid, and all obstacle information is encoded relative to this reference.

The update process consists of three transformations. First, the grid map \( G_t \) is rotated by \( R(\psi_t) \) to align it with a world-aligned frame. This removes the dependence on the robot’s heading and allows subsequent positional changes to be applied in a common reference frame.

Second, the translational motion of the UAV between time steps \( t-1 \) and \( t \), given by \( (\Delta x_t, \Delta y_t) \), is mapped to discrete grid shifts \( \begin{bmatrix} \Delta i \\ \Delta j \end{bmatrix}_t \). Applying this shift to the rotated grid translates all cells in the opposite direction of the robot’s motion. This operation is equivalent to keeping the robot fixed at the grid origin while updating the relative positions of environmental features.

Third, the shifted grid is rotated back to the robot-centered frame using \( R(-\psi_t) \), restoring the representation relative to the current UAV heading. Finally, the forgetting factor \( \rho \in (0,1) \) is applied element-wise to model temporal decay of outdated information.

These transformations collectively ensure that the relative positions of obstacles with respect to the UAV are updated consistently across time steps. Therefore, the update rule in Equation~\eqref{equ:shifting_c} correctly computes \( G_{t+1} \) from \( G_t \), preserving a coherent robot-centered representation of the environment.
\end{proof}


\begin{proposition}
The grid map \( G \) is resilient to noise in the sense that if any cell of \( G \) deviates from zero due to noise, its value will converge to zero over time, provided that the robot detects no obstacles or targets in that cell and the forgetting factor \( \rho \) lies within the interval \((0, 1)\).
\end{proposition}
\begin{proof}
    If the robot detects an obstacle or a target, some cells in $G$ will have values different from zero. When no obstacles or targets are detected, at each time step, the value of $G$ is multiplied by the forgetting factor $\rho \in (0,1)$, ensuring that each cell value gradually converges to zero (see Equation~\ref{equ:shifting_c}). Convergence to zero implies that any detection noise is progressively eliminated over time.
\end{proof}

Another advantage of the forgetting factor \( \rho \in (0, 1) \) is that even when two obstacles are closely spaced and their Gaussian responses partially overlap, this does not lead to a permanently impassable region. Since the grid map encodes probabilistic occupancy rather than binary obstacles, overlapping Gaussian distributions represent increased uncertainty instead of absolute blockage. Moreover, in the absence of persistent detections, the forgetting factor causes these elevated cell values to decay over time, preventing false gap closure and enabling the policy to explore and traverse feasible passages.

\subsection*{Extracting States}
The information from $G$ will be used as states in the DRL. Following the methodologies in~\cite{van2014scikit}, to reduce memory usage, $G$ is down-scaled to a smaller map $g$ using block reduction with the mean function and then flattened. An example of grid map $G$ and down-scaled grid map $g$ is depicted in Figure~\ref{fig:extracted_states}.

\begin{figure}
    \centering
    \includegraphics[width=1\linewidth]{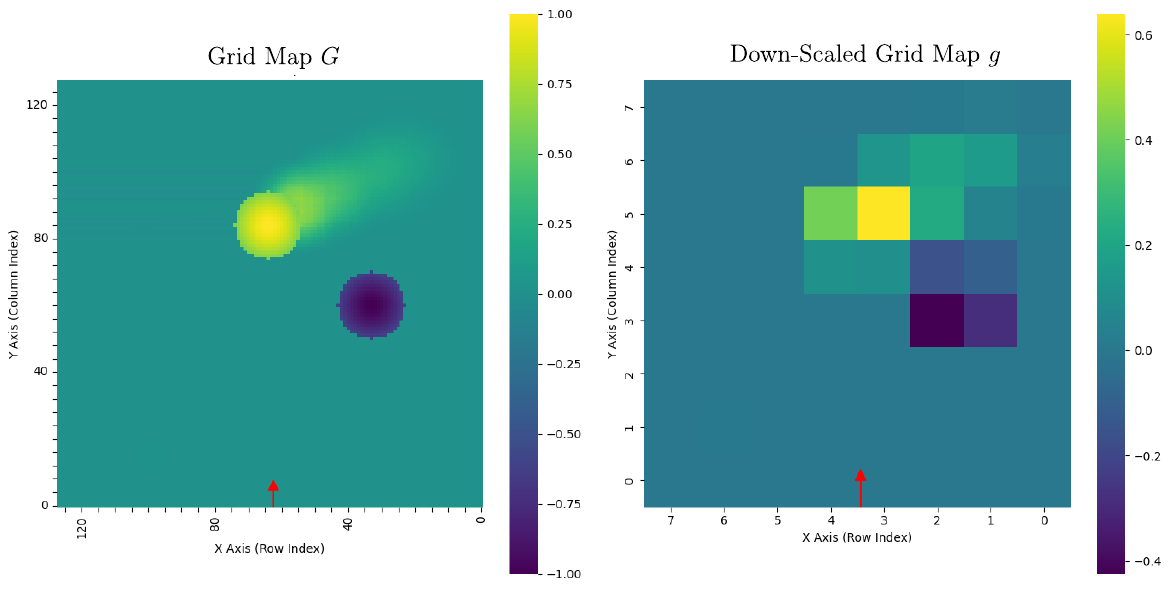}
    \caption{An example of grid map $G$ (left) and down-scaled grid map $g$ (right).}
    \label{fig:extracted_states}
\end{figure}

\textcolor{black}{Figure~\ref{fig:path_steps} illustrates several steps along an example path taken by the drone, with their corresponding behavior grid map. Each subfigure shows the environment on the left; the behavior grid map, which represents the drone's internal state, in the middle; and the down-scaled behavior grid map on the right. The steps shown are selected from the entire trajectory but are not necessarily consecutive time steps; for example, the first figure corresponds to time step zero, while the next may correspond to later steps. In each environment image (left side), the red dot represents the start point, gray circles indicate dynamic obstacles, gray lines indicate static obstacles, the blue line shows the path traversed from the start to the current position, and the drone is marked by a black cross with a colored arrow indicating its orientation. The red arrow denotes the drone's direction in the behavior grid map and the down-scaled behavior grid map. The circular obstacles are moving, directly influencing the shape of the generated path.}

\begin{figure}
  \centering
  \begin{subfigure}{0.6\linewidth}
    \centering
    \includegraphics[width=1\linewidth]{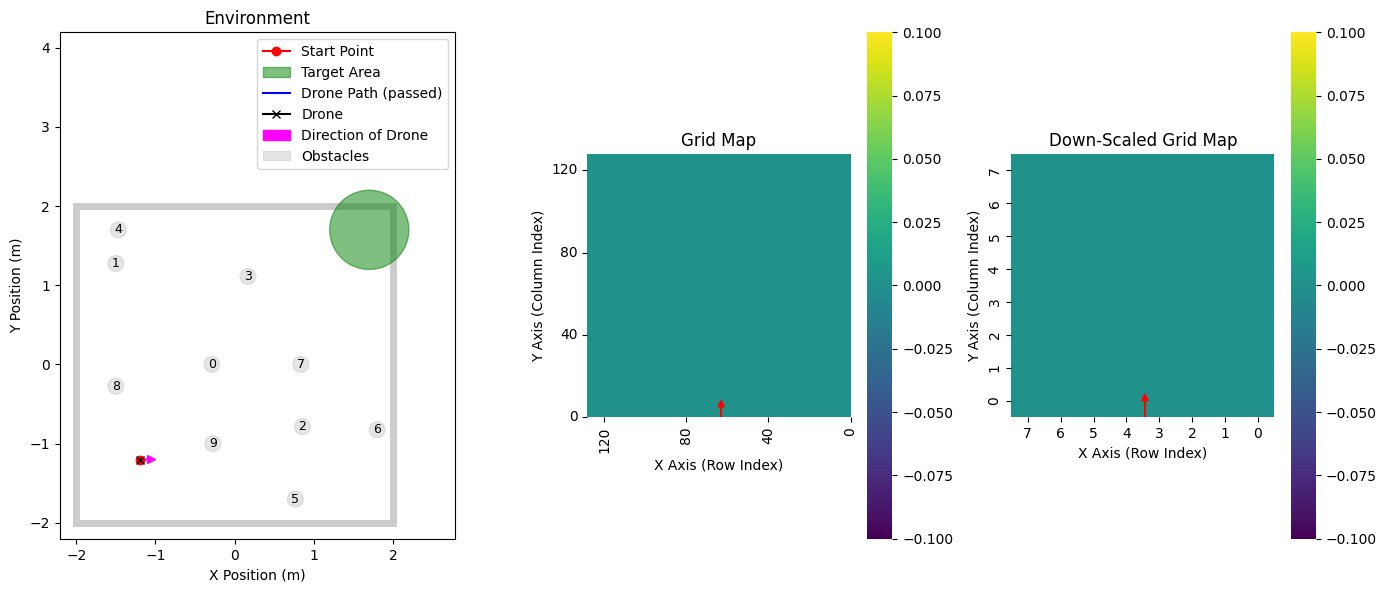} 
  \end{subfigure}
  \begin{subfigure}{0.6\linewidth}
    \centering
    \includegraphics[width=1\linewidth]{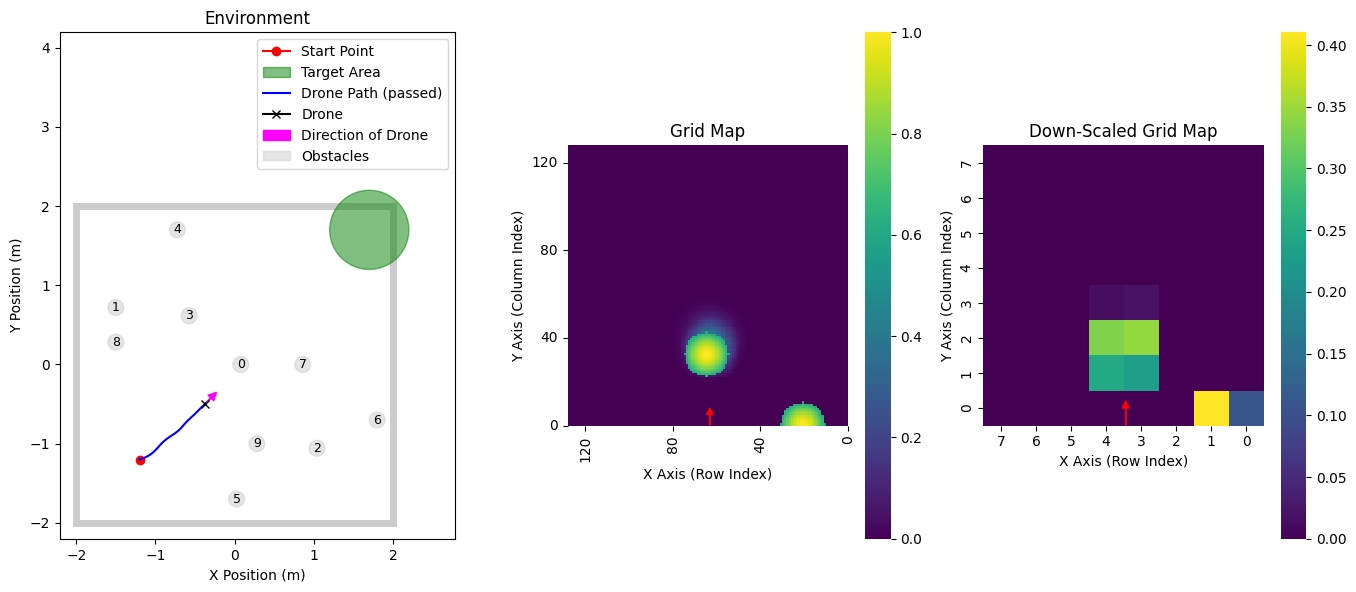} 
  \end{subfigure} 
  \begin{subfigure}{0.6\linewidth}
    \centering
    \includegraphics[width=1\linewidth]{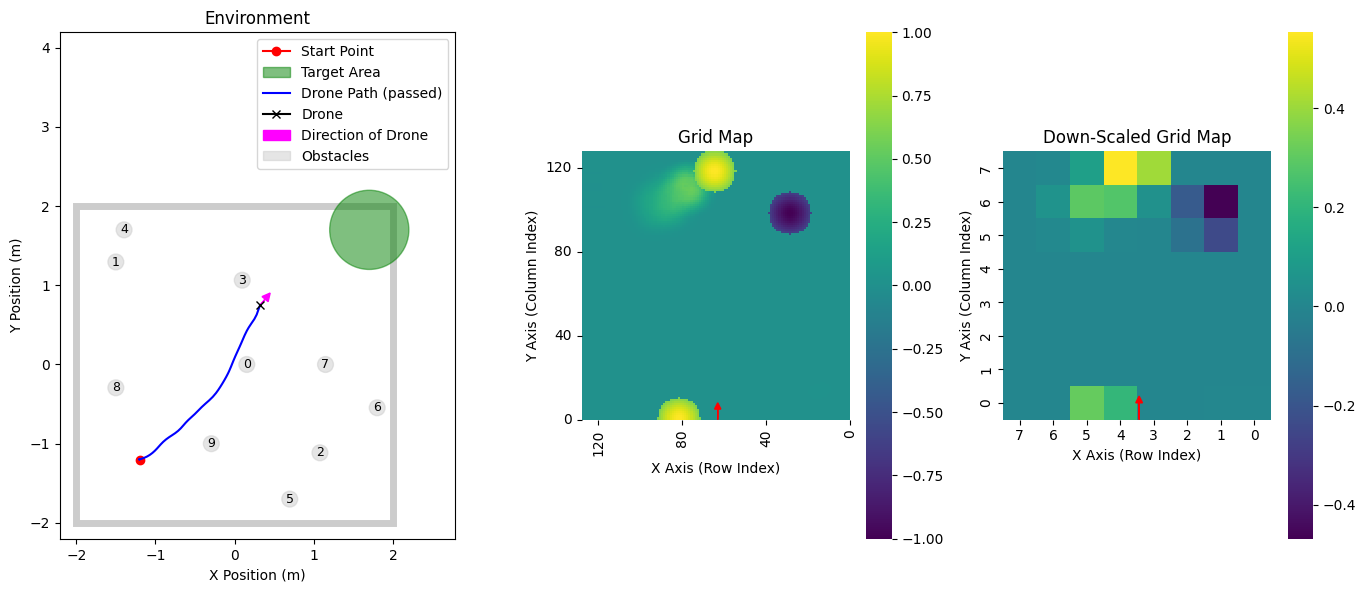} 
  \end{subfigure}
  \begin{subfigure}{0.6\linewidth}
    \centering
    \includegraphics[width=1\linewidth]{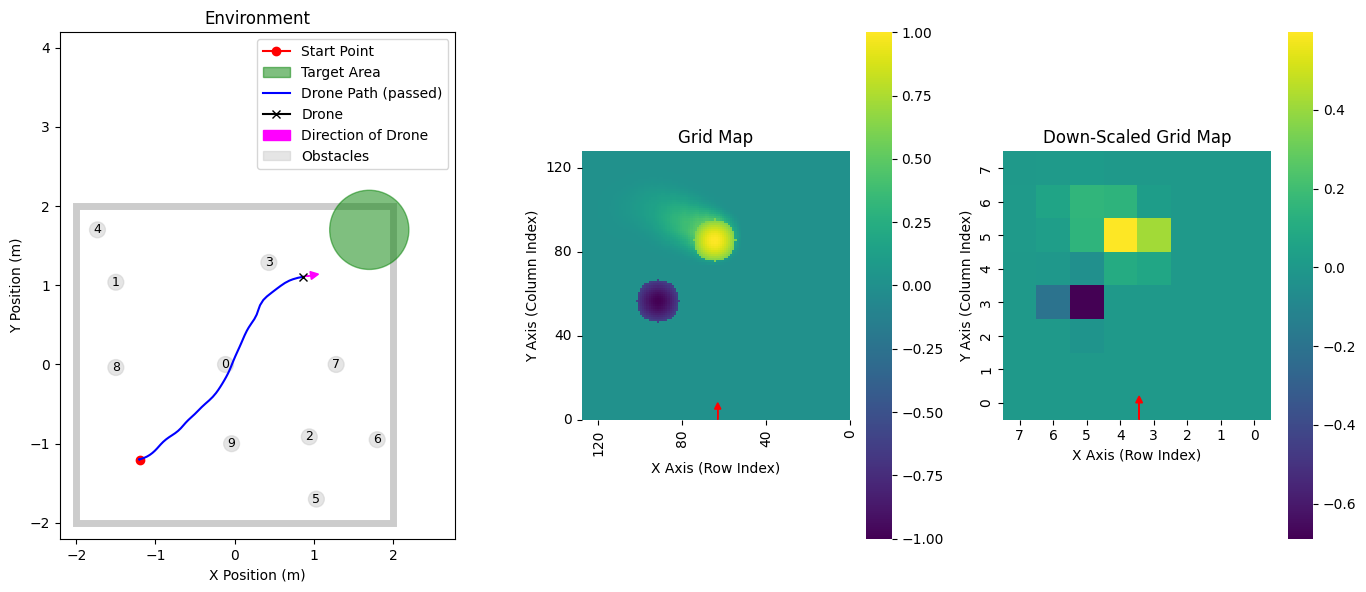} 
  \end{subfigure}
  \begin{subfigure}{0.6\linewidth}
    \centering
    \includegraphics[width=1\linewidth]{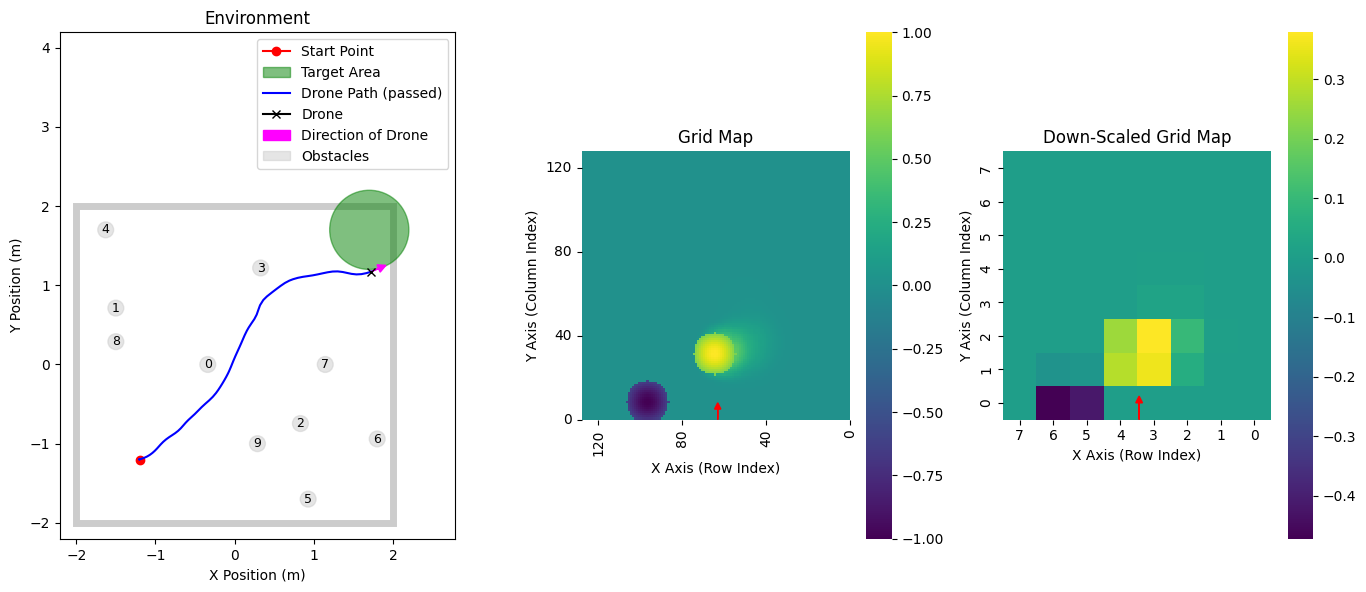} 
  \end{subfigure}  
  \caption{\textcolor{black}{Selected path steps showing an example of a UAV’s environment (left) and behavior grid map (right) during navigation. The circular obstacles are moving, directly influencing the shape of the generated path. }}
  \label{fig:path_steps}
\end{figure}

\subsection{Proximal Policy Optimization (PPO)}
The PPO algorithm with fixed-length trajectory segments is presented in Algorithm~\ref{algorithm:ppo}. In each iteration, each of the $N$ parallel actors collects $T$ timesteps of data. Using $N$ parallel actors serves two purposes: first, it accelerates experience collection, reducing overall training time by a factor proportional to $N$. Second, it gathers diverse trajectories from different environmental states, breaking data correlation and stabilizing policy training. The collected transitions $<s_t, a_t, r_t, s_{t+1}>$, where $t$  is the timestep index, $s_t$ is the state, $a_t$ is the action taken, $r_t$ is the received reward, and $s_{t+1}$ is the next state, are stored in trajectory memory and used to compute advantage estimates $\hat{A}_t$ using Generalized Advantage Estimation (GAE) with discount factor $\Gamma$ and GAE parameter $\lambda$, which controls the trade-off between bias and variance in the advantage estimates.

The surrogate loss is then formulated based on these $NT$ timesteps and optimized over $K$ epochs using the ADAM optimizer. In each epoch, the data is randomly shuffled and divided into minibatches of size $M$, where each minibatch is used to perform a policy update step by minimizing the clipped surrogate loss function. The value function is updated simultaneously by minimizing the squared error between predicted and target values.

In this algorithm, the surrogate loss $L_{t}(\theta)$ at timestep $t$ as a function of the parameter $\theta$ is defined as
\[
\begin{aligned}
  L_{t}(\theta) &= \hat{\mathbb{E}}_t \big[ 
    L_{t}^{\text{CLIP}}(\theta) - \zeta_1 L_{t}^{\text{VF}}(\theta) + \zeta_2 S[\pi_\theta](s_t) 
  \big],
\end{aligned}
\]
where $\zeta_1$ and $\zeta_2$ are coefficients, $\hat{\mathbb{E}}_t \left[ \dots \right]$ denotes the empirical average over a finite batch of samples, $S$ represents an entropy bonus, $\pi_{\theta}$ is a stochastic policy with parameter $\theta$, and $L_{t}^{\text{VF}}(\theta)$ is a squared-error loss calculated according to
\[
    L_{t}^{\text{VF}}(\theta) = \left(V_{\theta}(s_{t}) - V_{t}^{\text{target}}\right)^2
\]
where $V_{\theta}(s_{t})$ is predicted value of state $s_t$ and $V_{t}^{\text{target}}$ is a target value for the state-value function. Clipped surrogate loss, $L_{t}^{\text{CLIP}}$, is given by
\[
\begin{aligned}
L_{t}^{\text{CLIP}}(\theta) &= \hat{\mathbb{E}}_t \Big[ 
  \min\big(
    r_t(\theta) \hat{A}_t, \text{clip}(r_t(\theta), 1 - \epsilon, 1 + \epsilon) \hat{A}_t
  \big)
\Big],
\end{aligned}
\]
where $\epsilon = 0.2$ is a hyperparameter, $r_t(\theta) = \frac{\pi_{\theta}(a_t| s_t)}{\pi_{\theta_{\text{old}}}(a_t | s_t)}$ is the probability ratio,  $\pi_{\theta}(a_t| s_t)$ is the probability of taking action $a_t$ in state $s_t$ according to the current policy with parameters $\theta$, $\pi_{\theta_{\text{old}}}(a_t | s_t)$ is the probability of taking action $a_t$ in state $s_t$ according to the old policy with parameters $\theta_{\text{old}}$, $a_t$ is the action at timestep $t$, $\text{clip}(\cdot)$\footnote{The function $\text{clip}(x,a,b)$ limits $x$ to the interval $[a,b]$, i.e., $\text{clip}(x,a,b)=\min(\max(x,a),b)$.} denotes the clipping operator, and $\hat{A}_t$ is an estimator of the advantage function at timestep $t$, computed as
\[
\hat{A}_t = \delta_t + (\Gamma \lambda) \delta_{t+1} + \cdots + (\Gamma \lambda)^{T - t + 1} \delta_{T-1},
\]
where $t$ specifies the time index in $[0, T]$ within a given length-$T$ trajectory segment, $\Gamma$ is the discount factor, $\lambda$ is the Generalized Advantage Estimation (GAE) parameter, and $\delta_t$ is computed according to
\[
\delta_t = r_t + \Gamma V(s_{t+1}) - V(s_t),
\]
where $r_t$ is the reward received at timestep $t$, $V(s_{t+1})$ is the learned state-value function's prediction for the value of the next state $s_{t+1}$, and $V(s_t)$ is the learned state-value function's prediction for the value of the current state $s_t$~\cite{schulman2017proximal}.

\begin{algorithm}
\small
\caption{Proximal Policy Optimization with Clipped Surrogate Objective and GAE~\cite{schulman2017proximal}}
\label{algorithm:ppo}
\begin{algorithmic}
\For{iteration = 1, 2, ...}
    \For{actor = 1, 2, ..., N}
        \State Run policy $\pi_{\theta_{\text{old}}}$ in environment for $T$ timesteps
        \State Compute advantage estimates $\hat{A}_1, ..., \hat{A}_T$
    \EndFor
    \State Optimize surrogate $L$ w.r.t. $\theta$, with $K$ epochs and minibatch size $M \leq NT$
    \State $\theta_{\text{old}} \leftarrow \theta$
\EndFor
\end{algorithmic}
\end{algorithm}

\begin{proposition}
\label{prop:computational_complexity}
The per-iteration computational complexity of the proposed framework, consisting of behavior grid map construction, grid map updating, and policy optimization using PPO, is given by
\[
\mathcal{O}\!\left(
NT \left( C_{\text{env}} + C_{\text{grid}} \right)
+ K \frac{NT}{M} C_{\text{net}}
\right),
\]
where \( N \) is the number of parallel actors, \( T \) is the rollout horizon, \( K \) is the number of PPO optimization epochs, and \( M \) is the minibatch size. Here, \( C_{\text{env}} \) denotes the cost of one environment step, \( C_{\text{grid}} \) denotes the cost of grid map construction and updating per time step, and \( C_{\text{net}} \) denotes the cost of forward and backward passes through the policy and value networks.
\end{proposition}

\begin{proof}
During each PPO iteration, data collection is performed by running \( N \) actors for \( T \) time steps, resulting in \( NT \) environment interactions. For each time step, the behavior grid map is constructed using Algorithm~\ref{alg:grid_construction}, which involves a constant number of obstacle and target updates (\( k \in \{l,f,r,\iota\} \)). Each update consists of fixed-size matrix operations and Gaussian evaluations over the grid, yielding a per-step cost denoted by \( C_{\text{grid}} \).

Subsequently, the grid map is updated using Algorithm~\ref{algorithm:shift_grid_map}, which processes each cell of the \( n \times n \) grid once. Since the grid size is fixed throughout training, this operation contributes a constant cost per time step and is absorbed into \( C_{\text{grid}} \).

Therefore, the total cost of data collection and grid processing per iteration scales as
\(
\mathcal{O}\big(NT (C_{\text{env}} + C_{\text{grid}})\big).
\)

Following data collection, policy optimization is carried out using Algorithm~\ref{algorithm:ppo}. The surrogate objective is optimized for \( K \) epochs over minibatches of size \( M \), resulting in \( K \cdot \frac{NT}{M} \) gradient updates. Each update involves forward and backward passes through the policy and value networks, incurring a cost of \( C_{\text{net}} \) per update. This yields an optimization cost of
\(
\mathcal{O}\!\left(K \cdot \frac{NT}{M} \cdot C_{\text{net}}\right).
\)

Combining the costs of data collection, grid map processing, and policy optimization yields the stated per-iteration computational complexity.
\end{proof}

\subsection{Application PPO for Drone Online Motion Planning}
This section details the use of the PPO for drone online motion planning. The three key components of the PPO algorithm are actions, states, and rewards, which will be discussed in the following sections.

\subsection*{Actions}
Based on Section~\ref{sec:problem-definition}, the robot follows the defined UAV dynamics with a constant velocity in the \( x \)-direction. Consequently, the control input is defined in terms of the angular velocity \( r \). \textcolor{black}{The action space is defined as a continuous interval of admissible angular velocities, $\mathcal{A} = [r_{\min}, r_{\max}],$ where $r_{\min}$ and $r_{\max}$ represent the lower and upper bounds of the angular velocity commands (in radians per second).}

\subsection*{States}
At each time step \( t \), the grid \( G \) shifts according to Algorithm~\ref{algorithm:shift_grid_map}. Then, the positions of the obstacles and the target are marked following Algorithm~\ref{alg:grid_construction}. Finally, the states are extracted from the behavior grid map. \textcolor{black}{Additionally, the state vector incorporates four additional parameters to know the agent's global awareness. These include the Euclidean distance between the robot and the target, denoted as $\ell(t)$, and the absolute angular difference between the robot's heading and the target, denoted as $\Psi(t)$, both of which are scaled to the range $[0, 1]$. To accelerate the training process and provide the model with temporal context, the values from the previous time step, $\ell(t-1)$ and $\Psi(t-1)$, are also included.} The PPO states are thus composed of the extracted states that define the grid map along with these \textcolor{black}{four} parameters.

\textcolor{black}{Two PPO versions are implemented in this work. The primary framework is a PPO-based navigation system that incorporates the proposed behavior grid map representation as part of the state space. To evaluate the contribution of the grid map, a second PPO version is also implemented for comparison, where the grid map is excluded from the state representation. In this comparison version, the state vector consists of the four global parameters together with six parameters derived from range sensor measurements. These measurements provide obstacle distances in the left, right, and front directions at both the current time step ($t$) and the previous time step ($t-1$), thereby preserving temporal consistency in the state representation.}

\subsection*{Rewards}

\textcolor{black}{
The reward function at time step \( t \) is defined as
\[
R(t) =
\begin{cases}
+\bar{R}, & \text{if the robot reaches the target,} \\
-\bar{R}, & \text{if the robot exits the workspace or collides,} \\
c_1 \nu_{\text{dist}} + c_2 \nu_{\text{angle}} + c_3 F(t), & \text{otherwise,}
\end{cases}
\]
where \( R(t) \in [-\bar{R}, \bar{R}] \), and the coefficients satisfy \( c_1, c_2, c_3 \ge 0 \) with \( c_1 + c_2 + c_3 = 1 \). 
The reward function is designed to promote progress toward the target while maintaining safe navigation. The distance component penalizes lack of progress:
\[
\nu_{\text{dist}} =
\begin{cases}
0, & \text{if } \ell(t-1) - \ell(t) > 0, \\
-1, & \text{otherwise,}
\end{cases}
\]
where \( \ell(t) \) denotes the Euclidean distance to the target. 
The heading component penalizes angular misalignment:
\[
\nu_{\text{angle}} = -|\Psi(t)|,
\]
where \( \Psi(t) \) is the angular deviation between the UAV heading and the target direction. 
The grid-based component \( F(t) \) encodes proximity to the target and obstacles through the behavior grid map \( G \). If the target lies within the local grid representation, i.e., \( \min(G) < 0 \), then
\[
F(t) = 1 - \ell(t),
\]
which encourages convergence to the target. Otherwise,
\[
F(t) = -\max(G)\bigl(1 - d_{\text{obs}}\bigr),
\]
which penalizes proximity to obstacles. Here, \( d_{\text{obs}} \in [0,1] \) denotes the normalized distance to the nearest obstacle inferred from \( G \).
}

\textcolor{black}{
In the PPO variant without the grid representation, the reward function is simplified by removing the grid-based term \( c_3 F(t) \). In this case, the agent relies solely on \( \nu_{\text{dist}} \) and \( \nu_{\text{angle}} \) for navigation, while the terminal rewards for goal reaching and collisions remain unchanged.
}

\subsection{Proposed Architecture for UAV Online Motion Planning}
\label{sec:proposed_arch}
The proposed architecture for UAV online motion planning combines the behavior grid map algorithm, PPO with actor-critic neural networks, and a PID controller. The input to the PPO model consists of the extracted states, which include outputs from the behavior grid map algorithm, as well as $\ell$ and $\Psi$. The actor neural network then outputs the desired angular velocity $r$. The desired velocities are $u = \text{constant}$, $v = 0$, $w = 0$, and $r$. These desired velocities are subsequently passed as inputs to the PID controller. The complete architecture is illustrated in Figure~\ref{fig:actor_critic}.

\begin{figure}
    \centering
    \includegraphics[width=1\linewidth]{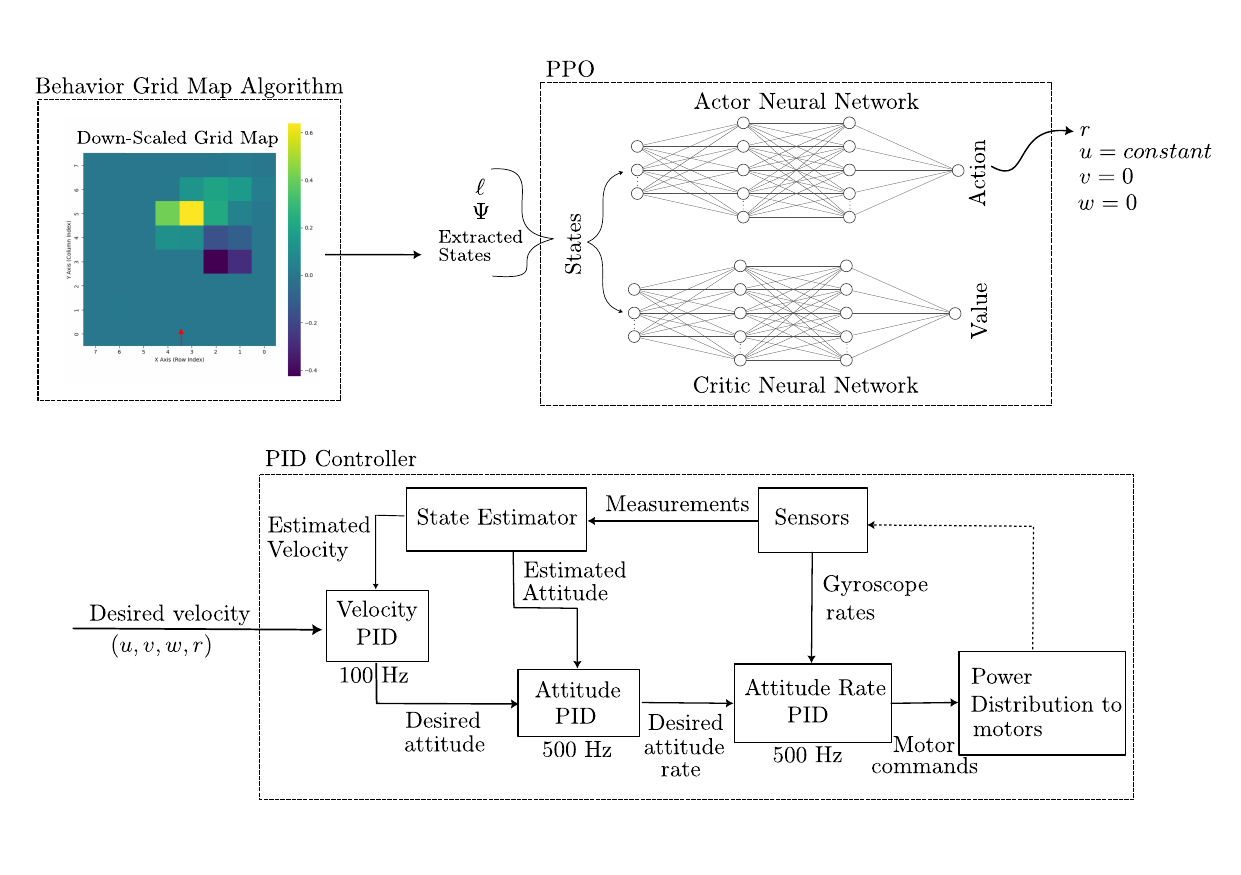}
        \caption{The proposed UAV online motion planning architecture integrates a behavior grid map algorithm, PPO with actor-critic neural networks, and a PID controller. Extracted states, including outputs from the behavior grid map and parameters $\ell$ and $\Psi$, serve as inputs to the PPO model. The actor-network outputs the desired angular velocity $r$, and the desired linear velocities are set as $u = \text{constant}$, $v = 0$, and $w = 0$. These desired velocities are then fed into the PID controller.}
    \label{fig:actor_critic}
\end{figure}

\textcolor{black}{
\textbf{Remark.} The proposed framework naturally extends to configurations with additional sensors without requiring architectural modifications. Each additional sensor contributes supplementary obstacle information to the behavior grid map, improving spatial coverage, reducing perception blind spots, and enhancing obstacle avoidance capability under sparse sensing conditions.
}

\subsection*{Enhanced Perception Method} \label{subsec:extended_method}
This section explains the enhanced perception method. The robot is equipped with three range finder sensors positioned in different directions, which limit its perception of the environment. \textcolor{black}{To compensate for this limitation, the robot is encouraged to perform an S-shaped motion that is generated by periodically changing the robot's heading, allowing the sensors to scan a wider portion of the surrounding environment than would be possible during straight-line motion. When the action value produced by the policy is close to zero, a sinusoidal yaw motion is imposed on the robot. The yaw angle \( \psi \) is set to $\psi = A\sin(\boldsymbol{\omega} t)$, where \( A \) represents the amplitude of the yaw oscillation, \( \boldsymbol{\omega} \) is the angular frequency, and \( t \) denotes the current time. The angular velocity $r = \dot{\psi} = A\boldsymbol{\omega}\cos(\boldsymbol{\omega} t)$ and the maximum magnitude of the angular velocity is bounded by $A\boldsymbol{\omega} \leq r_{\max}$.}

With this strategy, energy consumption increases, but the resulting path is safer. The user can choose between the proposed standard method, which conserves energy, and the enhanced method, which prioritizes safety. \textcolor{black}{The parameters \(A\) and \(\boldsymbol{\omega}\) control the intensity of the S-shape movement and therefore affect the trade-off between safety and energy efficiency. Small values produce limited exploration benefits, while excessively large values increase path oscillation.} 

\begin{color}{black} 
\captionsetup{labelfont={color=black},textfont={color=black}} 

\makeatletter \g@addto@macro\@floatboxreset{\color{black}} \makeatletter

\section{Results and Discussion}\label{sec:result}
This section evaluates the proposed DRL-based online motion-planning framework through extensive simulation experiments and preliminary real-world validation with a Crazyflie~2.1 UAV. This section includes: (1) detailed ablation studies on the proposed grid-map representation and reward formulation, (2) evaluation under multiple environment sizes and congestion levels, (3) comparison against multiple baselines, including PPO-only, PPO with S-shape movement, MPC, and two variants of the proposed framework, (4) battery consumption comparison between standard and enhanced methods, (5) computational efficiency and deployment feasibility, and (6) real-world deployment results.

\subsection{Simulation Setup}
The training is performed on a PC running Ubuntu 24.04 LTS with 16 GB RAM and an Intel Core i7 CPU. We conducted the simulations on a CPU, while a GPU was used for the real-world experiments. The main simulation and PPO training parameters are summarized in Table~\ref{tab:sim_params}.

\begin{table}
\small
    \centering
    \caption{Simulation and PPO Training Parameters}
    \begin{tabular}{ll}
        \toprule
        \textbf{Parameter} & \textbf{Value} \\
        \midrule
        \multicolumn{2}{c}{\textit{Simulation Setup}} \\
        \midrule
        Simulation platform & Webots 2024a \\
        Simulation area & $20 \times 20$ meters \\
        Number of obstacles & 16 cylinders \\
        Radius of cylinder & 0.2 meters \\
        Radius of target area & 0.1 meters \\
        UAV model & Crazyflie 2.1 \\
        UAV flight altitude & Constant (0.3 meters) \\
        Max episode length & 512 steps \\
        Drone's Velocity in $x$ direction (Robot Frame) & 0.5 m/s \\
        PC specifications & 16 GB RAM, Intel Core i7 CPU \\
        \midrule
        \multicolumn{2}{c}{\textit{PPO Training Parameters}} \\
        \midrule
        Total training steps & 55,296 \\
        Evaluation frequency & every 10,000 steps \\
        Batch size & 2048 \\
        Mini-batch size & 64 \\
        Actor learning rate ($lr_a$) & 0.0003 \\
        Critic learning rate ($lr_c$) & 0.0003 \\
        Discount factor ($\Gamma$) & 0.99 \\
        GAE lambda ($\lambda$) & 0.95 \\
        Clipping parameter ($\epsilon$) & 0.2 \\
        PPO epochs per update ($K_{epochs}$) & 10 \\
        Hidden layer width & 64 \\
        Number of Hidden layers & 2 \\
        Activation Function & Hyperbolic tangent ($\tanh$) \\
        Entropy coefficient & 0.01 \\
        State dimension & 68 \\
        Advantage normalization & Enabled \\
        State normalization & Enabled \\
        Reward normalization & Enabled \\
        Learning rate decay & Enabled \\
        Gradient clipping & Enabled \\
        Orthogonal initialization & Enabled \\
        $r_{\min}$ & -1 rad/s \\
        $r_{\max}$ & 1~~rad/s \\
        \bottomrule
    \end{tabular}
    \label{tab:sim_params}
\end{table}
To improve training stability and convergence, several optimization strategies were employed, including reward normalization, state normalization, orthogonal initialization, learning-rate decay, and gradient clipping.

The UAV used in the simulation is a Crazyflie 2.1 drone, which implements the dynamics in Equation~\ref{equ:dynamics}, and the simulation was conducted using Webots 2024a\footnote{\url{https://cyberbotics.com}}, a professional open-source robot simulator. The Webots simulation setup, featuring static and dynamic obstacles, the Crazyflie drone, and a rubber duck as the target area, is shown in Figure~\ref{fig:webots}.

\begin{figure}
  \centering
  \begin{subfigure}{0.3\linewidth}
    \centering
    \includegraphics[width=1\linewidth]{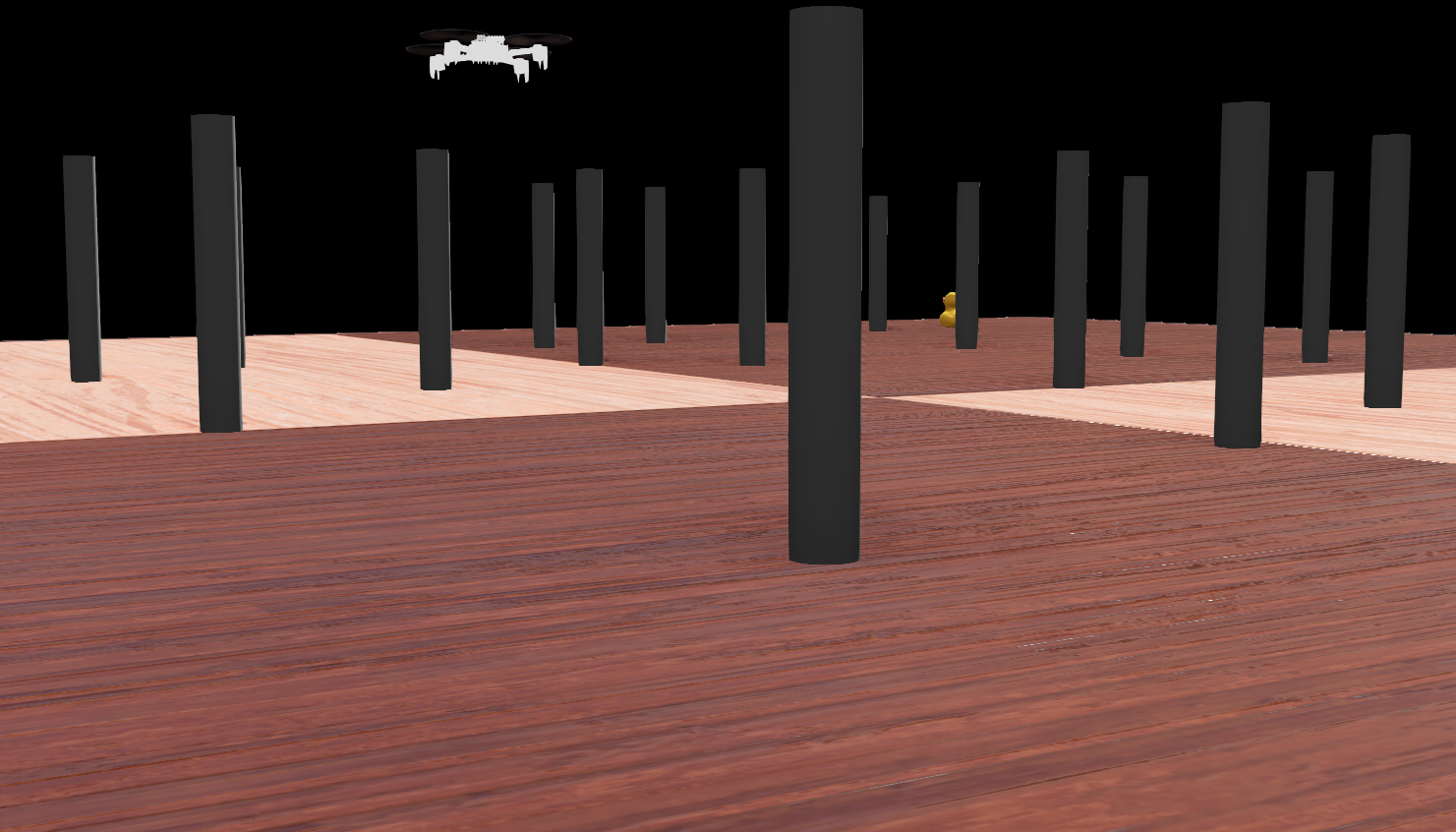} 
  \end{subfigure} 
  \begin{subfigure}{0.3\linewidth}
    \centering
    \includegraphics[width=1\linewidth]{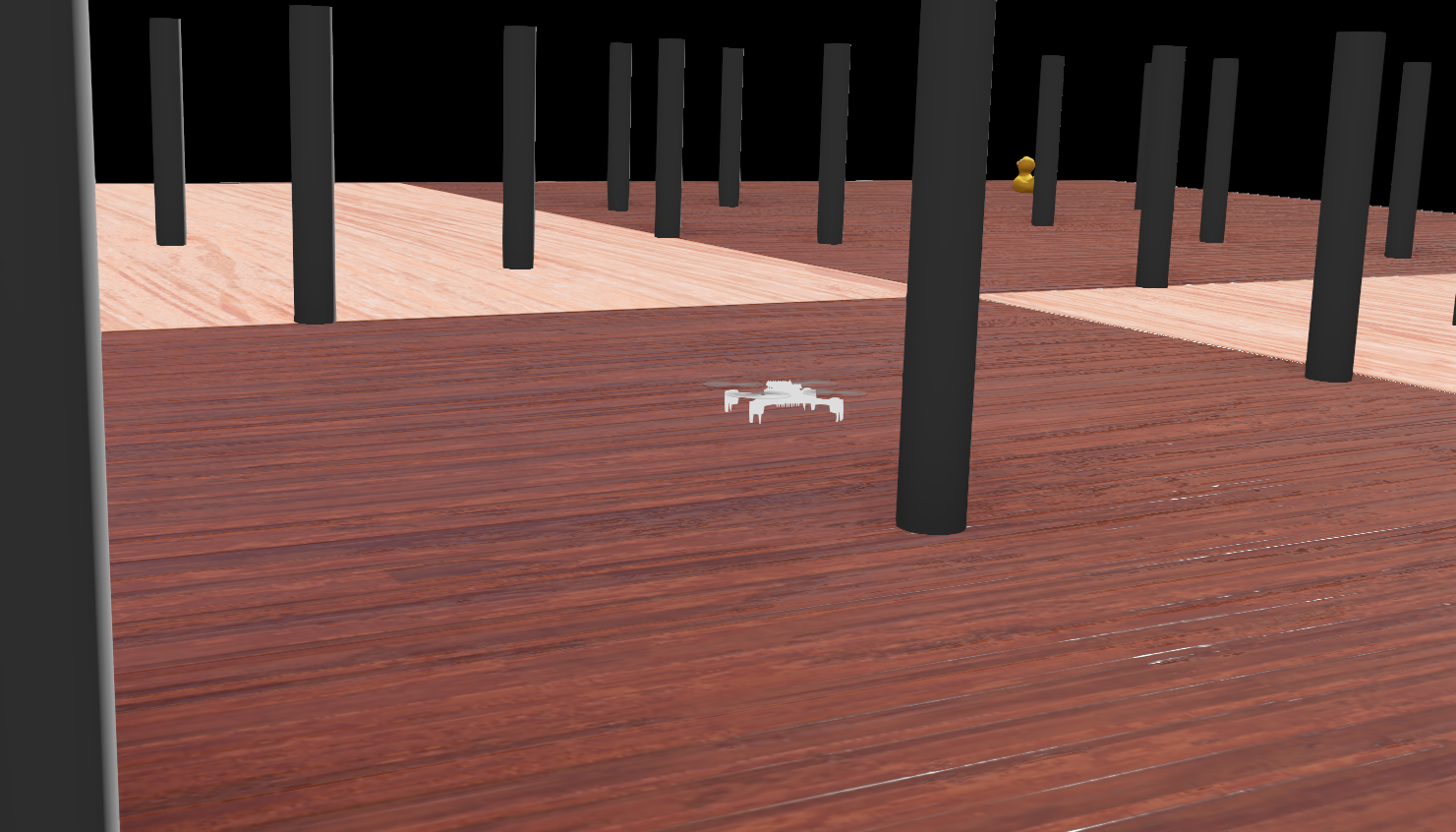} 
  \end{subfigure}
  \begin{subfigure}{0.3\linewidth}
    \centering
    \includegraphics[width=1\linewidth]{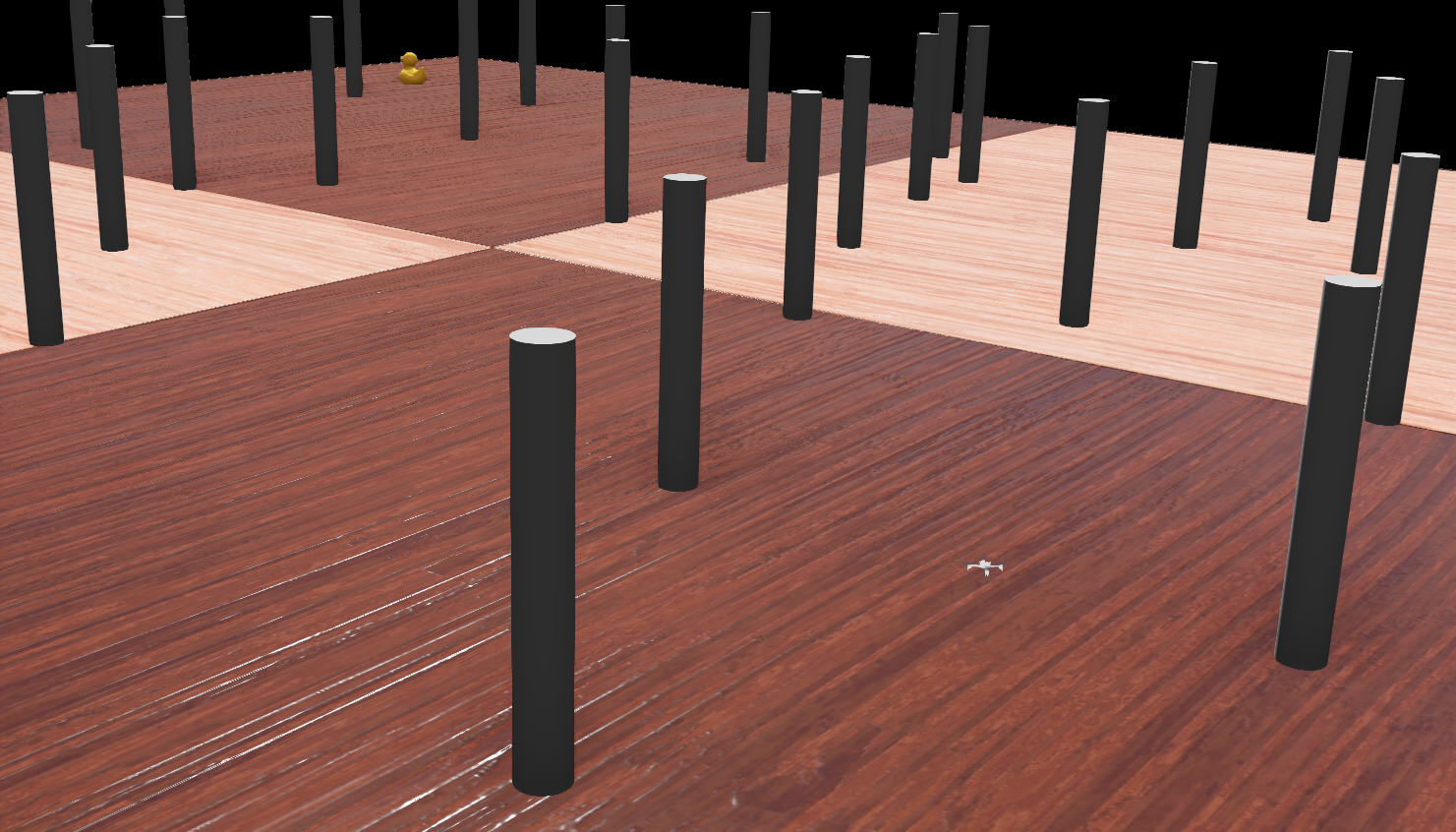} 
  \end{subfigure}
  \begin{subfigure}{0.9\linewidth}
    \centering
    \includegraphics[width=1\linewidth]{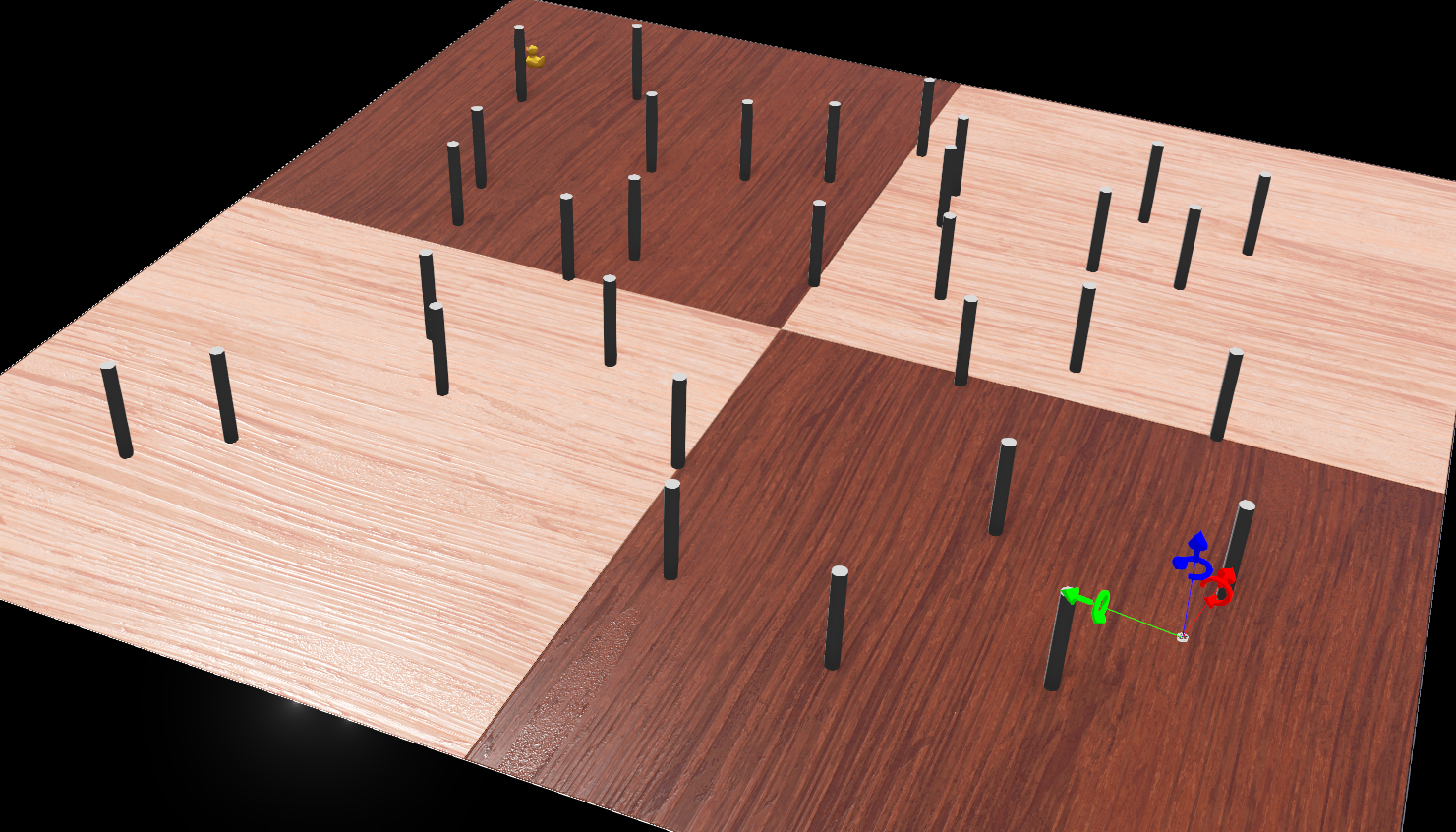} 
  \end{subfigure}  
  \caption{A different perspective of one scene of the Webots simulation setup, featuring static and dynamic obstacles, the Crazyflie drone, and a rubber duck as the target area.}
  \label{fig:webots}
\end{figure}

Figure~\ref{fig:angular-velocity} illustrates the comparison between the actual and desired angular velocities ($r$) for an example. The results demonstrate that the actual angular velocity closely follows the desired one, indicating effective control of the drone.
\begin{figure}
    \centering
    \includegraphics[width=0.6\linewidth]{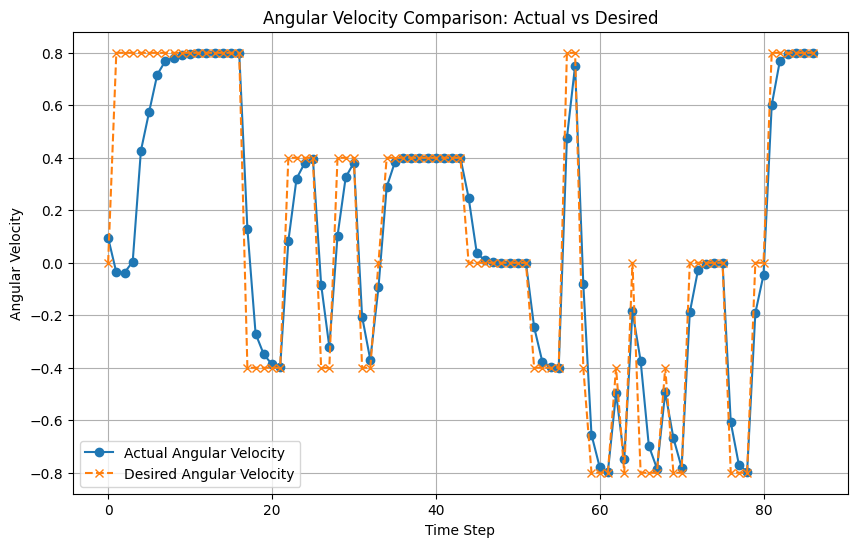}
    \caption{Comparison between the actual and desired angular velocities ($r$).}
    \label{fig:angular-velocity}
\end{figure}

\subsection{PPO Training Analysis} \label{ppo_training_analysis}
Figure \ref{fig:ppo_graphs} tracks the evolution of the PPO agent over 100,000 steps. The training dynamics are defined by four key metrics:

Value Loss: Measures the error in the critic's reward predictions. The graph shows a sharp initial drop followed by a steady decline, indicating the agent is learning to accurately value its environment.

Policy Gradient Loss: Represents the direction and magnitude of policy updates. It trends downward into negative territory, signifying the reinforcement of advantageous actions.

Entropy: Quantifies the randomness of actions. The upward-sloping curve indicates a transition from high exploration (randomness) to exploitation (certainty).

Total Loss: The weighted summation of all objectives. Its V-shaped trajectory highlights the overall optimization progress.

The training reaches its near peak efficiency at the red dashed line (step 55,296). At this intersection, the Value Loss and Total Loss reach their near-lowest local points, while the Policy Gradient Loss achieves its near maximum stability.

The shaded red region marks post-selection degradation. In this phase, the value loss begins to climb, and the policy gradient loss becomes volatile, suggesting the model is either overfitting to recent data or experiencing catastrophic forgetting. By selecting the checkpoint at the point of lowest loss, we ensure the best balance between learned behavior and generalizability.

\begin{figure}
  \centering
  \begin{subfigure}{0.45\linewidth}
    \centering
    \includegraphics[width=\linewidth]{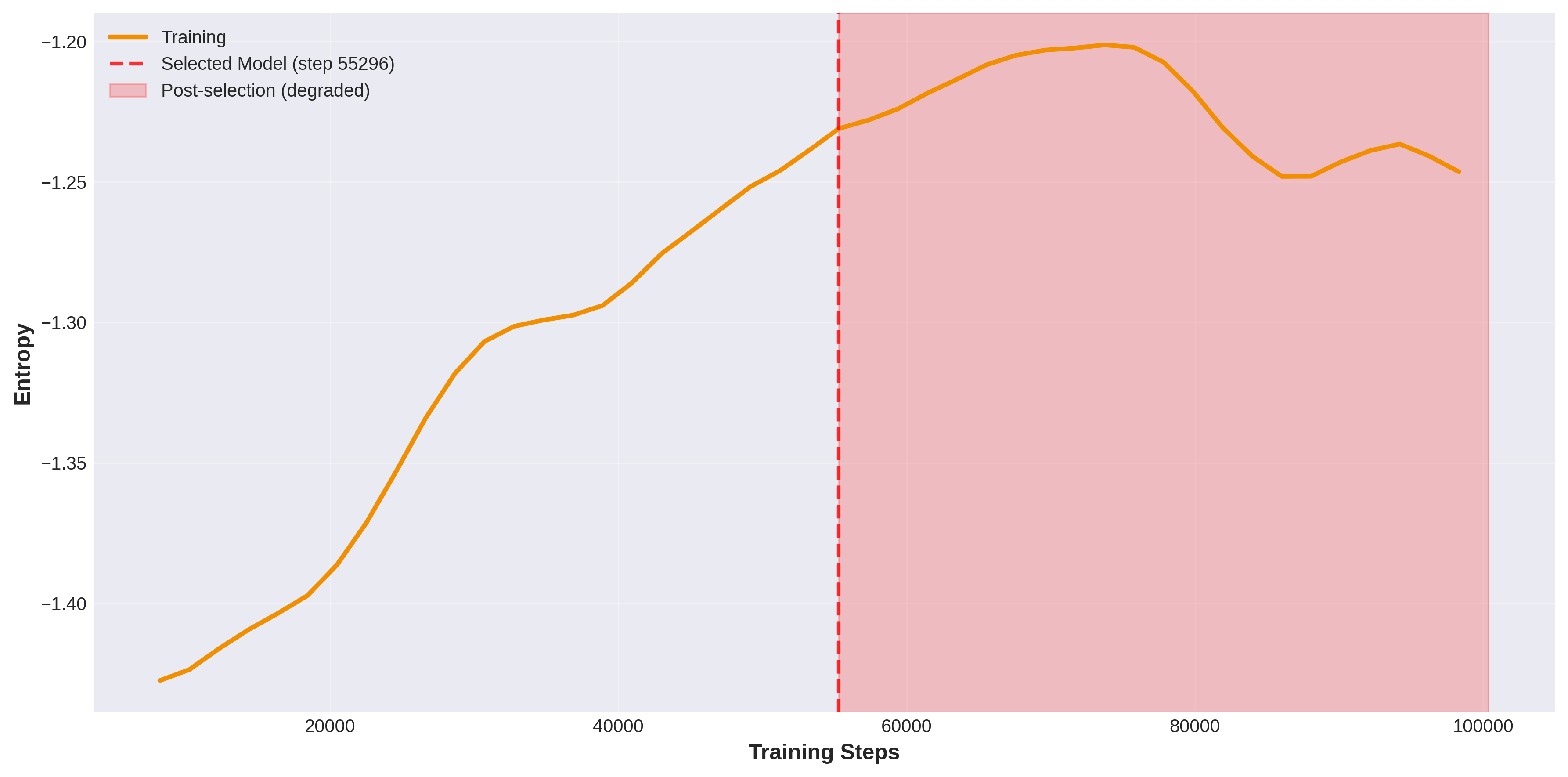} 
    \caption{Entropy loss}
    \label{fig:entropy_loss}
  \end{subfigure}
  \hfill
  \begin{subfigure}{0.45\linewidth}
    \centering
    \includegraphics[width=\linewidth]{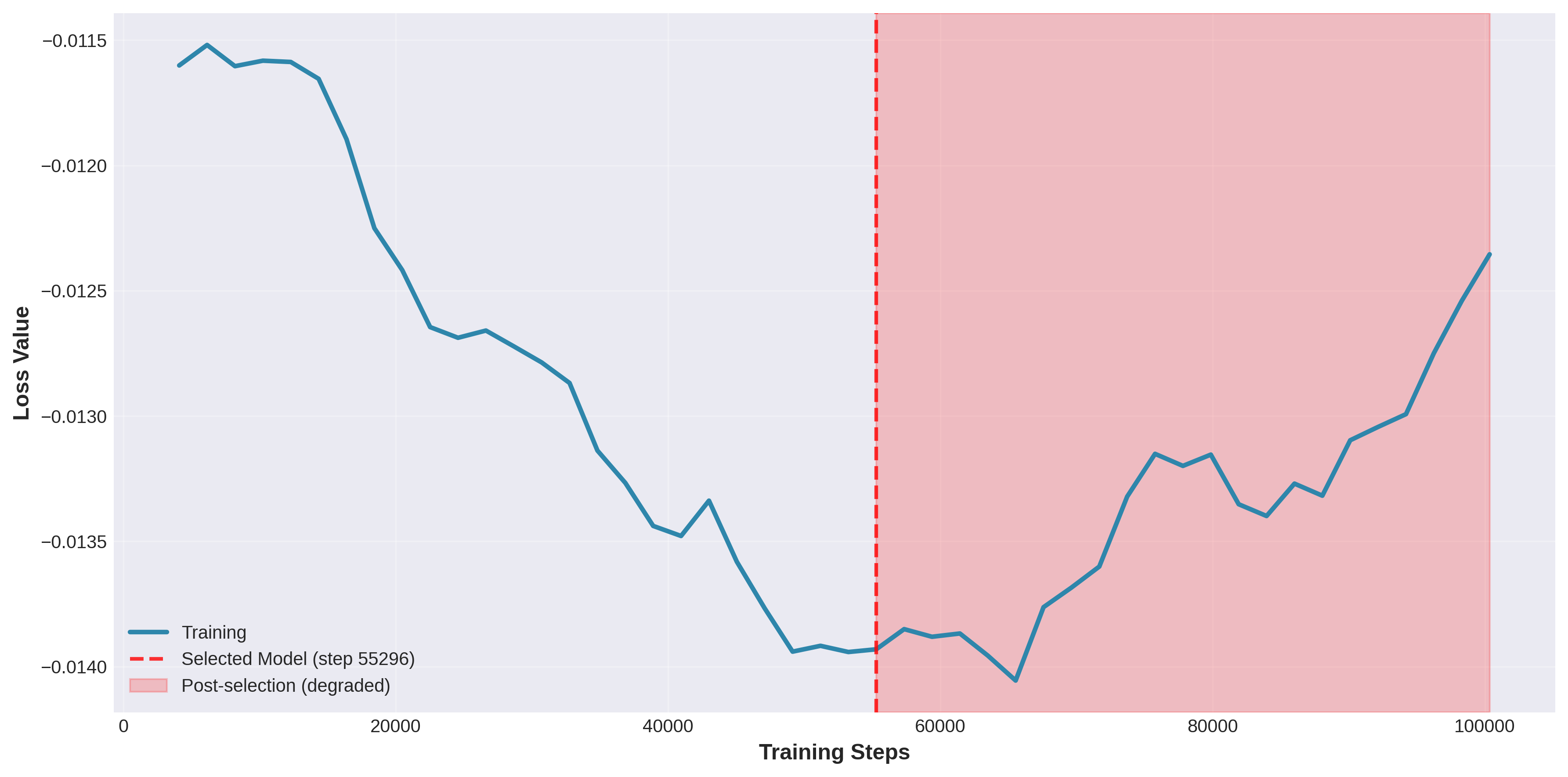} 
    \caption{Policy gradient loss}
    \label{fig:policy_loss}
  \end{subfigure}
  
  \vspace{0.3cm}
  
  \begin{subfigure}{0.45\linewidth}
    \centering
    \includegraphics[width=\linewidth]{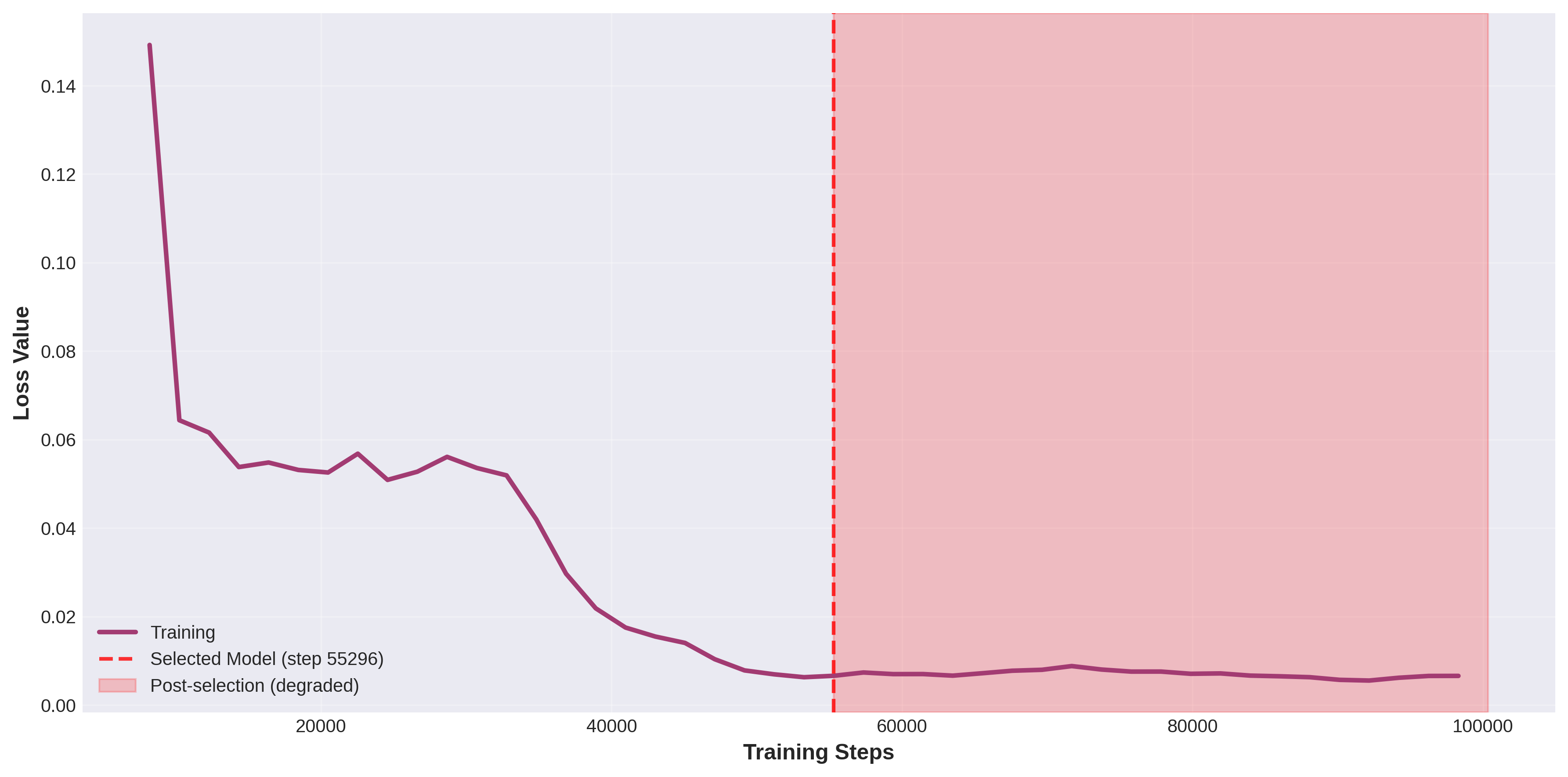} 
    \caption{Value function loss}
    \label{fig:value_loss}
  \end{subfigure}
  \hfill
  \begin{subfigure}{0.45\linewidth}
    \centering
    \includegraphics[width=\linewidth]{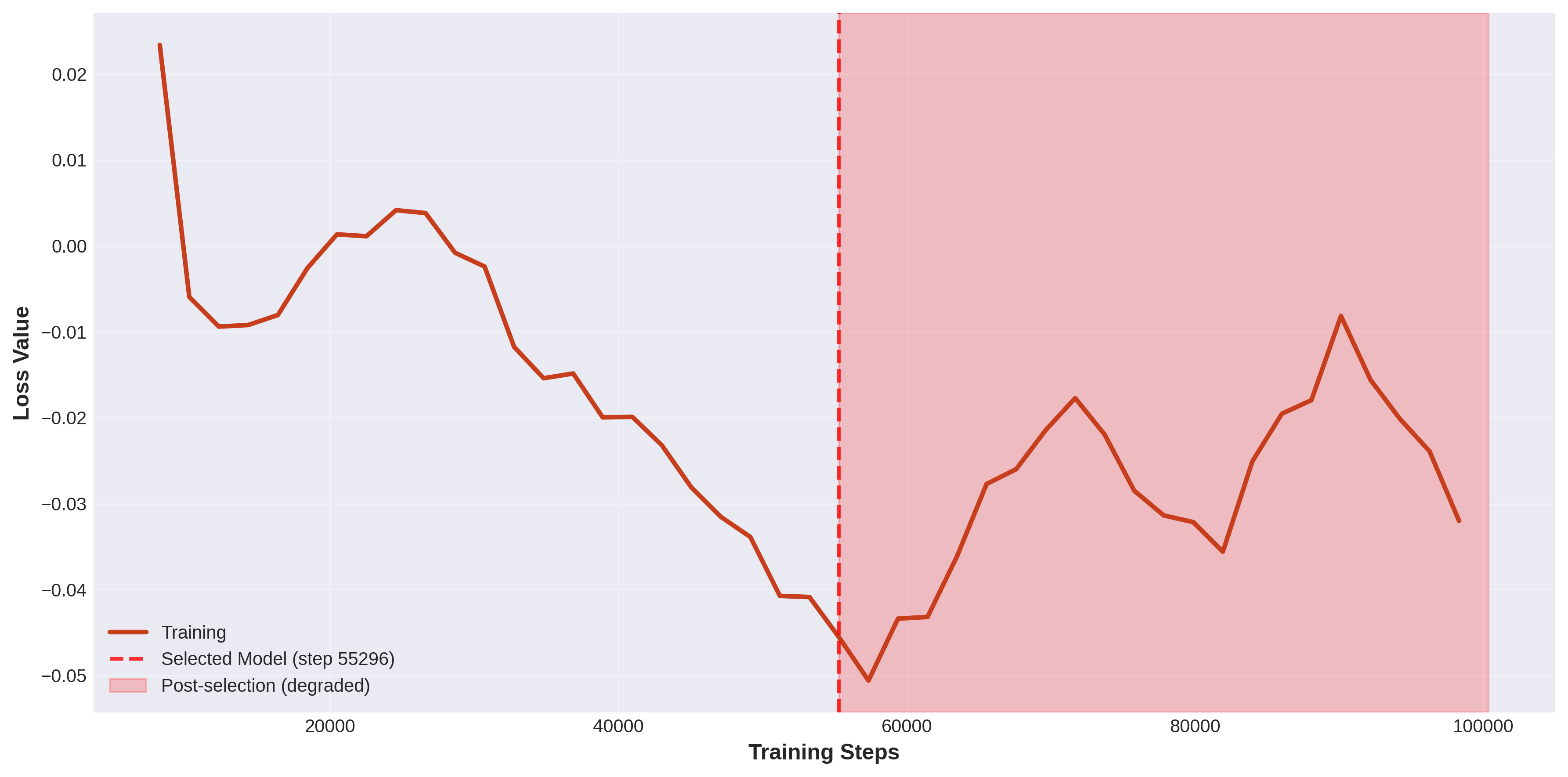} 
    \caption{Total combined loss}
    \label{fig:total_loss}
  \end{subfigure}
  \caption{PPO training metrics over 100,000 steps. The red dashed line at step 55,296 indicates the selected model checkpoint, before which all metrics demonstrate stable convergence. The shaded region highlights post-selection degradation in training dynamics.}
  \label{fig:ppo_graphs} 
\end{figure}

\subsection{Evaluation Metrics}

To provide a comprehensive evaluation of navigation quality, computational efficiency, and safety, the following metrics are used.

\textbf{Success Rate} represents the percentage of episodes in which the UAV reaches the target region without collision.

\textbf{Path Length} the total traveled distance from the initial position to the target. It is computed by accumulating Euclidean distances between consecutive trajectory points.

\textbf{Trajectory Smoothness} evaluates directional consistency along the generated path. It is defined as the cumulative heading-angle variation between consecutive trajectory segments:
\[
\text{Smoothness} =
\sum_{i=2}^{n-1}
\left|
\arccos
\left(
\frac{
\vec{v}_{i,1}\cdot\vec{v}_{i,2}
}{
\|\vec{v}_{i,1}\|
\|\vec{v}_{i,2}\|
}
\right)
\right|,
\]
where $\vec{v}_{i,1}$ and $\vec{v}_{i,2}$ denote consecutive motion vectors.

Lower smoothness values correspond to smoother and more dynamically feasible trajectories.

\textbf{Average Clearance} measures the mean minimum distance between the UAV and nearby obstacles throughout the trajectory:
\begin{align*}
\text{AC} = \frac{\sum_{i=1}^{n} \min \left( \|\vec{p}_i - \vec{o}_{j}\| \; \text{for all} \; \vec{o}_{j} \in \mathcal{O} \right)}{n},
\end{align*}
where $\vec{p}_i$ is the position of the robot at the $i$-th step, $\vec{o}_{j}$ is the position of the $j$-th obstacle, $\mathcal{O}$ denotes the set of all obstacles, $n$ is the total number of trajectory points, and $\text{AC}$ is the Average Clearance.

Higher clearance values indicate safer navigation behavior.

\textbf{Average Planning Time Per Step} measures the average computational time required to generate one planning action during online navigation. It reflects the real-time capability of the proposed framework.

\subsection{Ablation Study for Grid Map} \label{subsec:Ablation_Study_for_Grid_Map}
One of the main contributions of the proposed framework is the dynamic grid-map representation used to encode obstacle information for the DRL policy. To analyze the influence of individual design choices, multiple ablation studies were conducted on the Gaussian spread parameter, forgetting factor, coverage area, grid resolution, reduction size, and reduction operator. The tests were conducted in a \(20 \times 20~m^2\) environment under moderate obstacle density conditions. Details of the environment setup and testing methodology are provided in a subsequent subsection.

\textbf{Effect of Gaussian Spread Parameter:}
Table~\ref{tab:sigma_ablation} evaluates different Gaussian spread values ($\sigma$). Small values produce highly localized obstacle representations, while large values overly diffuse obstacle information across the map.

The results show that $\sigma=16$ achieves the best overall performance with a success rate of $90\%$. When $\sigma=2$, obstacle influence becomes excessively sparse, reducing environmental awareness and leading to lower success rates. Conversely, $\sigma=64$ produces less informative occupancy distributions, degrading navigation performance values.

These results indicate that moderate obstacle diffusion provides the best balance between obstacle awareness and motion flexibility.

\begin{table}[h]
\centering
\caption{Ablation Study on Gaussian Spread ($\sigma$) in Grid Map}
\label{tab:sigma_ablation}
\begin{tabularx}{\linewidth}{
>{\hsize=0.03\hsize}X
|>{\hsize=0.14\hsize}X
|>{\hsize=0.18\hsize}X
|>{\hsize=0.2\hsize}X
|>{\hsize=0.2\hsize}X
|>{\hsize=0.25\hsize}X
}
\hline
\textbf{$\sigma$} & \textbf{Success Rate (\%)} & \textbf{Avg. Plan Time per Step (ms)} & \textbf{Avg. Path Length (m)} & \textbf{Avg. Clearance (m)} & \textbf{Avg. Smoothness} \\
\hline
2 & 77 & 80.77 & 9.30 ± 1.06 & 1.78 ± 0.36 & 3.00 ± 1.06\\
16 & 90  & 77.77 & 9.37 ± 1.20 & 1.70 ± 0.36 & 4.51 ± 2.44 \\
64 & 81 & 84.24 & 9.70 ± 1.38 & 1.79 ± 0.35 & 4.91 ± 2.48\\
\hline
\end{tabularx}
\end{table}

\textbf{Effect of Forgetting Factor:}
Table~\ref{tab:lambda_ablation} analyzes the temporal forgetting factor $\rho$, which determines how quickly previous obstacle observations decay.

A small forgetting factor ($\rho=0.1$) rapidly removes obstacle history, reducing temporal consistency and lowering success rates. Increasing $\rho$ improves navigation robustness by preserving obstacle motion information across time steps. The highest success rate is obtained using $\rho=0.95$, demonstrating the importance of temporal memory in dynamic environments.

\begin{table}[h]
\centering
\caption{Ablation Study on Forgetting Factor ($\rho$)}
\label{tab:lambda_ablation}
\begin{tabularx}{\linewidth}{
>{\hsize=0.04\hsize}X
|>{\hsize=0.14\hsize}X
|>{\hsize=0.18\hsize}X
|>{\hsize=0.2\hsize}X
|>{\hsize=0.2\hsize}X
|>{\hsize=0.25\hsize}X
}
\hline
\textbf{$\rho$} & \textbf{Success Rate (\%)} & \textbf{Avg. Plan Time per Step (ms)} & \textbf{Avg. Path Length (m)} & \textbf{Avg. Clearance (m)}& \textbf{Avg. Smoothness} \\
\hline
0.1 & 77 & 80.92 & 9.39 ± 1.07 & 1.78 ± 0.35 & 3.65 ± 1.59\\
0.5 & 82 & 81.04 & 9.51 ± 1.16 & 1.76 ± 0.36 & 4.07 ± 2.17\\
0.95 & 90  & 77.77 & 9.37 ± 1.20 & 1.70 ± 0.36 & 4.51 ± 2.44 \\
\hline
\end{tabularx}
\end{table}

\textbf{Effect of Grid Coverage Area:}
Table~\ref{tab:grid_coverage_area_ablation} evaluates different spatial coverage areas centered around the UAV.

A small coverage area ($1\times1$~m$^2$) limits environmental awareness and reduces obstacle anticipation capability. In contrast, a large area ($4\times4$~m$^2$) introduces excessive irrelevant information, increasing trajectory oscillations and reducing success rate. The intermediate coverage size ($2\times2$~m$^2$) provides the best trade-off between local detail and global awareness.

\begin{table}[h]
\centering
\caption{Ablation Study on Grid Coverage Area Around Robot}
\label{tab:grid_coverage_area_ablation}
\begin{tabularx}{\linewidth}{
>{\hsize=0.13\hsize}X
|>{\hsize=0.1\hsize}X
|>{\hsize=0.11\hsize}X
|>{\hsize=0.2\hsize}X
|>{\hsize=0.2\hsize}X
|>{\hsize=0.25\hsize}X
}
\hline
\textbf{Grid Coverage Area ($m^2$)} & \textbf{Success Rate (\%)} &\textbf{ Avg. Plan Time per Step (ms)} & \textbf{Avg. Path Length (m)} & \textbf{Avg. Clearance (m)} & \textbf{Avg. Smoothness} \\
\hline
1 $\times$ 1 & 82 & 80.45 & 9.37 ± 1.09 & 1.75 ± 0.37 & 3.42 ± 1.44\\
2 $\times$ 2 & 90  & 77.77 & 9.37 ± 1.20 & 1.70 ± 0.36 & 4.51 ± 2.44 \\
4 $\times$ 4 & 76 & 81.16 & 9.47 ± 1.19 & 1.81 ± 0.31 & 5.70 ± 2.22\\
\hline
\end{tabularx}
\end{table}

\textbf{Effect of Grid Resolution:}
The influence of grid resolution is analyzed in Table~\ref{tab:grid_ablation}. Increasing the grid size from $32\times32$ to $128\times128$ significantly improves success rate due to richer spatial encoding. However, further increasing the resolution to $256\times256$ provides only marginal performance improvement while dramatically increasing computational cost from $77.77$~ms to $193.03$~ms per planning step.

Therefore, the $128\times128$ resolution was selected as the best compromise between accuracy and computational efficiency.

\begin{table}[h]
\centering
\caption{Ablation Study on Grid Resolution}
\label{tab:grid_ablation}
\begin{tabularx}{\linewidth}{
>{\hsize=0.15\hsize}X
|>{\hsize=0.02\hsize}X
|>{\hsize=0.1\hsize}X
|>{\hsize=0.11\hsize}X
|>{\hsize=0.2\hsize}X
|>{\hsize=0.2\hsize}X
|>{\hsize=0.25\hsize}X
}
\hline
\textbf{Grid Size} & \textbf{$\sigma$} & \textbf{Success Rate (\%)} & \textbf{Avg. Plan Time per Step (ms)} & \textbf{Avg. Path Length (m)} & \textbf{Avg. Clearance (m)} & \textbf{Avg. Smoothness} \\
\hline
32 $\times$ 32 & 4 & 81 & 42.53 & 9.47 ± 1.19 & 1.75 ± 0.37 & 3.83 ± 2.19\\
128 $\times$ 128 & 16 & 90  & 77.77 & 9.37 ± 1.20 & 1.70 ± 0.36 & 4.51 ± 2.44 \\
256 $\times$ 256 & 32 & 91 & 193.03 & 9.58 ± 1.11 & 1.75 ± 0.35 & 5.87 ± 2.25\\
\hline
\end{tabularx}
\end{table}

\textbf{Effect of Down-Scaled Grid Map Size:}
Table~\ref{tab:grid_reduction_ablation} evaluates different down-scaled grid map sizes resulting from feature compression.

A small down-scaled grid map size ($4\times4$) oversimplifies obstacle information and results in unstable trajectories with high smoothness values. Conversely, a large down-scaled grid map size ($16\times16$) preserves excessive local variations and significantly decreases the success rate to 67\%. This performance drop occurs because the neural network structure is kept constant (consisting of 2 hidden layers with 64 nodes each) to strictly isolate the effects of the input size. For the $16\times16$ configuration, the input dimension is set to 260, which overwhelms the fixed capacity of the hidden layers and leads to severe underfitting or representation bottlenecks. The intermediate down-scaled grid map size ($8\times8$) achieves the best overall navigation performance.

\begin{table}[h]
\centering
\caption{Ablation Study on Down-Scaled Grid Map Size}
\begin{tabularx}{\linewidth}{
>{\hsize=0.13\hsize}X
|>{\hsize=0.1\hsize}X
|>{\hsize=0.11\hsize}X
|>{\hsize=0.2\hsize}X
|>{\hsize=0.2\hsize}X
|>{\hsize=0.25\hsize}X
}
\hline
\textbf{Reduction Size} & \textbf{Success Rate (\%)} & \textbf{Avg. Plan Time per Step (ms)} & \textbf{Avg. Path Length (m)} & \textbf{Avg. Clearance (m)} & \textbf{Avg. Smoothness} \\
\hline
4 $\times$ 4 & 84 & 81.46 & 9.60 ± 1.08 & 1.75 ± 0.35 & 10.87 ± 2.12\\
8 $\times$ 8 & 90  & 77.77 & 9.37 ± 1.20 & 1.70 ± 0.36 & 4.51 ± 2.44 \\
16 $\times$ 16 & 67 & 81.76 & 9.16 ± 0.96 & 1.81 ± 0.34 & 3.21 ± 1.58\\
\hline
\label{tab:grid_reduction_ablation}
\end{tabularx}
\end{table}

\textbf{Effect of Reduction Operator:}
Table~\ref{tab:reduction_method_ablation} compares mean-pooling and max-pooling reduction operators. Mean-pooling aggregates the feature representations by calculating the average value across a specified region, capturing the smooth, global context of the data. In contrast, max-pooling extracts only the maximum value within that region, effectively emphasizing the most prominent or dominant feature while discarding less intense signals.

Mean pooling achieves a significantly higher success rate than max pooling because it preserves spatial occupancy distributions more effectively and produces smoother obstacle representations for the policy network.

Overall, the ablation studies confirm that the proposed grid-map formulation is highly sensitive to representation design choices. The selected configuration achieves the best balance among success rate, safety, trajectory quality, and computational efficiency.

\begin{table}[h]
\centering
\caption{Ablation Study on Grid Reduction Operator}
\label{tab:reduction_method_ablation}
\begin{tabularx}{\linewidth}{
>{\hsize=0.13\hsize}X
|>{\hsize=0.1\hsize}X
|>{\hsize=0.11\hsize}X
|>{\hsize=0.2\hsize}X
|>{\hsize=0.2\hsize}X
|>{\hsize=0.25\hsize}X
}
\hline
\textbf{Operator ($\mathbb{F}$)} & \textbf{Success Rate (\%)} & \textbf{Avg. Plan Time per Step (ms)} & \textbf{Avg. Path Length (m)} & \textbf{Avg. Clearance (m)} & \textbf{Avg. Smoothness} \\
\hline
Max & 81 & 80.49 & 9.57 ± 1.17 & 1.77 ± 0.36 & 4.63 ± 2.65\\
Mean & 90  & 77.77 & 9.37 ± 1.20 & 1.70 ± 0.36 & 4.51 ± 2.44 \\
\hline
\end{tabularx}
\end{table}

\subsection{Ablation Study for Reward Function Coefficients} \label{subsec:AblationStudyforRewardFunctionCoefficients}
The reward function combines three objectives: target reaching, heading alignment, and obstacle-awareness through the grid-map representation. Table~\ref{tab:reward_coeff} evaluates the impact of different coefficient combinations.

When the grid-map coefficient dominates ($c_3=0.8$), the UAV behaves conservatively, maintaining larger obstacle distances but producing longer trajectories. When the heading term dominates ($c_2=0.8$), the UAV prioritizes directional consistency but becomes less reactive to obstacles, reducing success rate. Similarly, excessive emphasis on goal reaching ($c_1=0.8$) encourages aggressive motion toward the target and increases collision probability.

The balanced configuration $(c_1,c_2,c_3)=(0.33,0.33,0.33)$ achieves the highest success rate, demonstrating that stable navigation requires simultaneous consideration of target progress, trajectory consistency, and obstacle avoidance.

\begin{table}[h]
\centering
\caption{Ablation Study on Reward Function Coefficients (Goal Weight ($c_1$), Heading Weight ($c_2$), and Grid Map Weight ($c_3$))}
\label{tab:reward_coeff}
\begin{tabularx}{\linewidth}{
>{\hsize=0.05\hsize}X
|>{\hsize=0.04\hsize}X
|>{\hsize=0.04\hsize}X
|>{\hsize=0.1\hsize}X
|>{\hsize=0.11\hsize}X
|>{\hsize=0.2\hsize}X
|>{\hsize=0.2\hsize}X
|>{\hsize=0.25\hsize}X
}
\hline
\textbf{$c_1$} & \textbf{$c_2$} & \textbf{$c_3$} & \textbf{Success Rate (\%)} & \textbf{Avg. Plan Time per Step (ms)} & \textbf{Avg. Path Length (m)} & \textbf{Avg. Clearance (m)} & \textbf{Avg. Smoothness} \\
\hline
0.1 & 0.1 & 0.8 & 90 & 76.39 & 10.02 ± 1.03 & 1.82 ± 0.43 & 5.68 ± 2.39\\
0.1 & 0.8 & 0.1 & 66 & 79.30 & 9.27 ± 0.99 & 1.79 ± 0.35 & 6.56 ± 2.41\\
0.8 & 0.1 & 0.1 & 84 & 78.40 & 9.50 ± 1.11 & 1.76 ± 0.36 & 3.49 ± 1.69\\
0.33 & 0.33 & 0.33 & 90  & 77.77 & 9.37 ± 1.20 & 1.70 ± 0.36 & 4.51 ± 2.44 \\
\hline
\end{tabularx}
\end{table}

\subsection{Comparison Across Different Environments}\label{subsec:Comparison_Across_Different_Environments}
To evaluate scalability and robustness, the proposed framework was tested in environments with different dimensions and obstacle densities. Figure~\ref{fig:env_different_tests} illustrates the evaluation scenarios, while Table~\ref{tab:env_types} summarizes their configurations. 

In these evaluation scenarios, the robot's initial position is fixed at $(0, 0)$ with a heading angle (yaw) of $0^\circ$. The target positions are distributed along a square perimeter that completely encloses the obstacle field. By placing the targets outside the obstacle region, the testing framework forces the robot to navigate through a diverse set of obstacle configurations. The obstacles are randomly dispersed. This setup ensures comprehensive generality, allowing us to evaluate the robot's decision-making behavior across a wide range of realistic scenarios.

For each individual test run, the drone takes off from the center start point and attempts to reach a single designated target location selected from the outer perimeter shown in Figure~\ref{fig:env_different_tests}. Although the figure displays the entire set of available target locations simultaneously for visualization purposes, only one unique target point is active during a given trial. The drone evaluates the environment by sequentially navigating toward each target location one by one. The final performance metrics are subsequently derived from the average results across all of these individual runs. 

Furthermore, while the initial placement and motion paths of both static and dynamic obstacles are generated randomly, their specific environmental layouts are kept strictly identical across all compared algorithms for a given test instance. This consistency ensures a fair evaluation, subjecting every method to the exact same spatial challenges and localized constraints. Additionally, to ensure a realistic kinematic challenge, the average velocity of the dynamic obstacles is set to approximately half of the drone's maximum velocity.

Two environment sizes were considered: small-scale environments ($5\times5$~m$^2$) and large-scale environments ($20\times20$~m$^2$).

Each environment includes four congestion levels: no obstacle, low congestion, moderate congestion, and high congestion. The following methods were compared: 
\begin{itemize}
    \item \textbf{PPO-NoMap-NoSshape:} PPO planner without the grid map and without S-shape movement.
    \item \textbf{PPO-NoMap-WithSshape:} PPO planner without the grid map but with S-shape movement.
    \item \textbf{PPO-WithMap-NoSshape:} proposed framework that is PPO planner with the grid map but without S-shape movement (standard).
    \item \textbf{PPO-WithMap-WithSshape:} full proposed method that is PPO planner with the grid map and with S-shape movement (enhanced).
    \item \textbf{MPC:} model predictive control baseline.
\end{itemize}

\begin{table}
  \centering
  \caption{Test environments with increasing levels of congestion used for evaluation.}
  \label{tab:env_types}
  \begin{tabular}{c|c|c|c}
      \hline
      \textbf{Environment Size ($m^2$)} &  \textbf{Scenario} & \textbf{\# Static Obstacles} & \textbf{\# Dynamic Obstacles} \\
      \hline
      \multirow{4}{*}{5 $\times$ 5} & No obstacle & 0 & 0 \\
      & Low & 1 & 1   \\
      & Moderate & 2 & 2   \\
      & High & 4 & 4  \\
      \hline
       \multirow{4}{*}{20 $\times$ 20} & No obstacle & 0 & 0 \\
      & Low & 4 & 4   \\
      & Moderate & 8 & 8  \\
      & High & 16 & 16  \\
      \hline
  \end{tabular}
\end{table}

\begin{figure}
  \centering
  \begin{subfigure}{0.24\linewidth}
    \centering
    \includegraphics[width=0.6\linewidth]{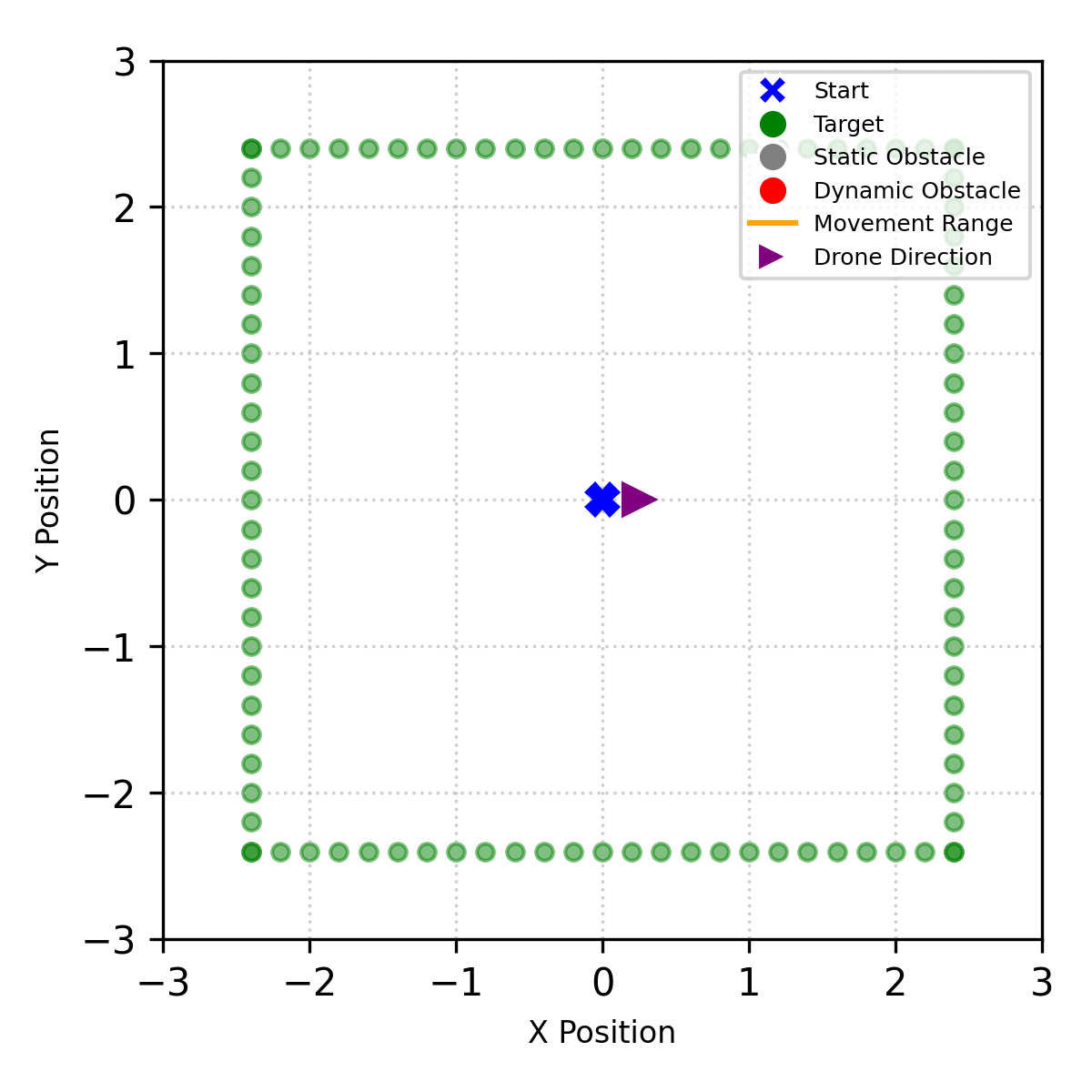} 
    \caption{5x5 No-Obstacle}
  \end{subfigure}
  \begin{subfigure}{0.24\linewidth}
    \centering
    \includegraphics[width=0.6\linewidth]{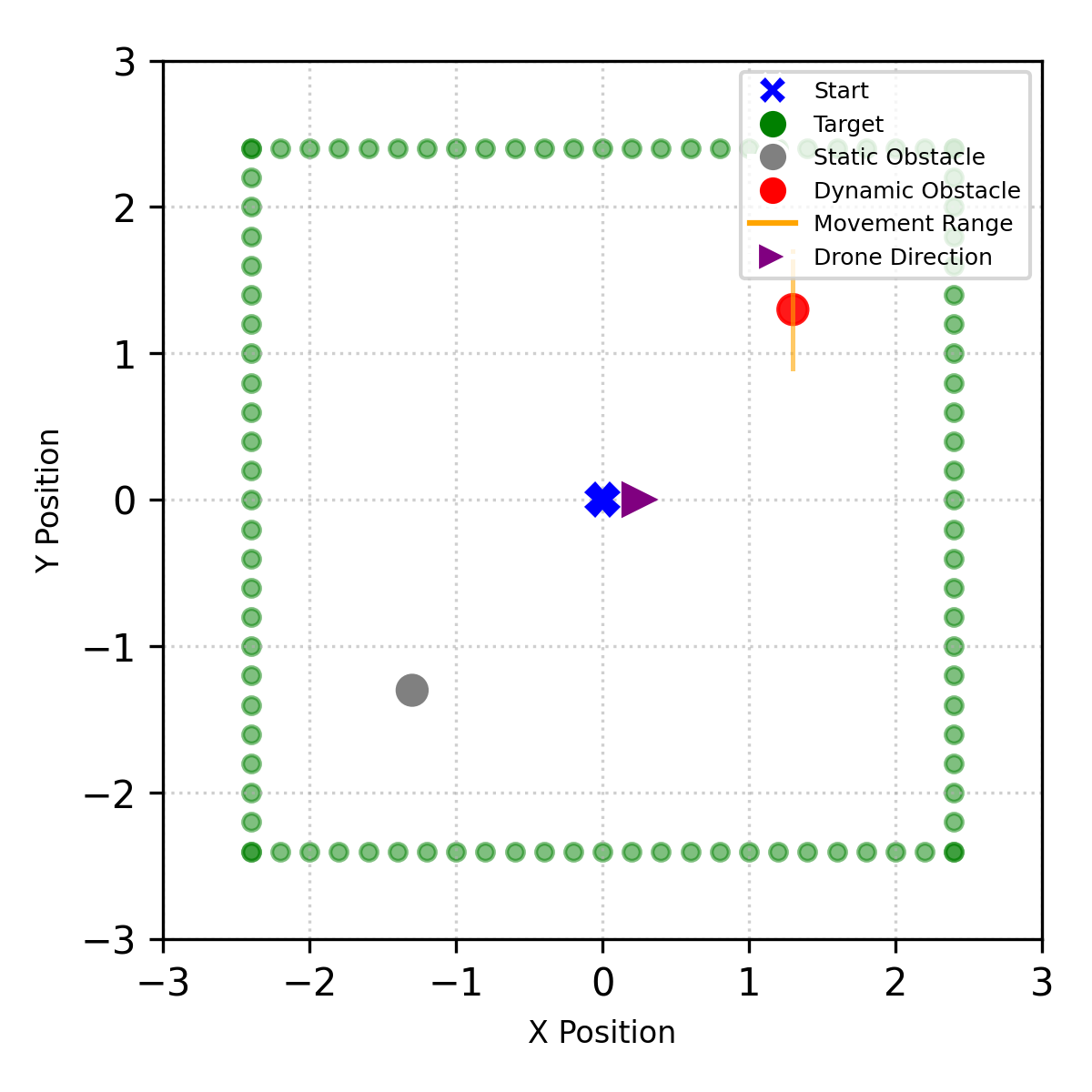} 
    \caption{5x5 Low}
  \end{subfigure}
  \begin{subfigure}{0.24\linewidth}
    \centering
    \includegraphics[width=0.6\linewidth]{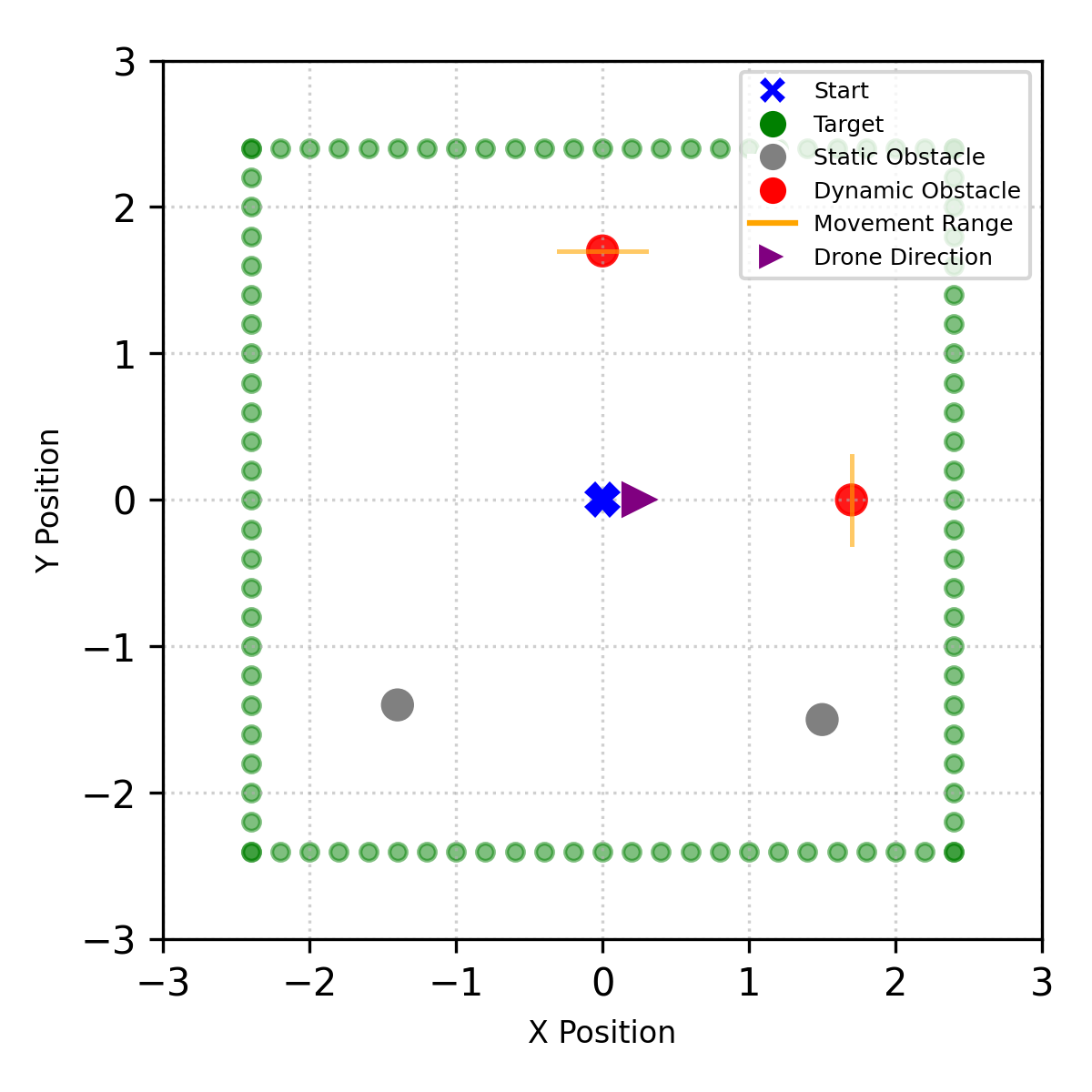} 
    \caption{5x5 Moderate}
  \end{subfigure}
  \begin{subfigure}{0.24\linewidth}
    \centering
    \includegraphics[width=0.6\linewidth]{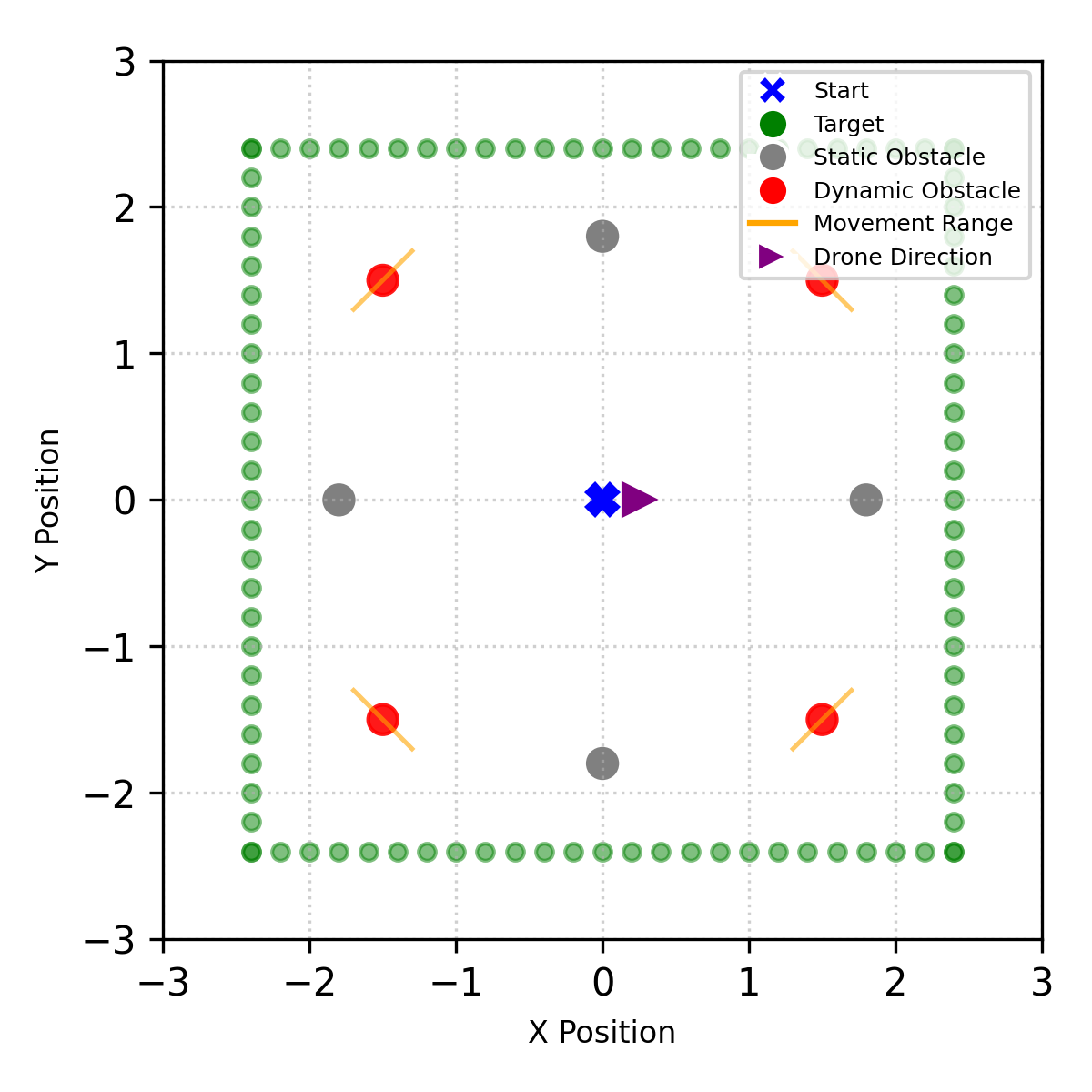} 
    \caption{5x5 High}
  \end{subfigure}

  \vspace{1em} 

  \begin{subfigure}{0.24\linewidth}
    \centering
    \includegraphics[width=1.0\linewidth]{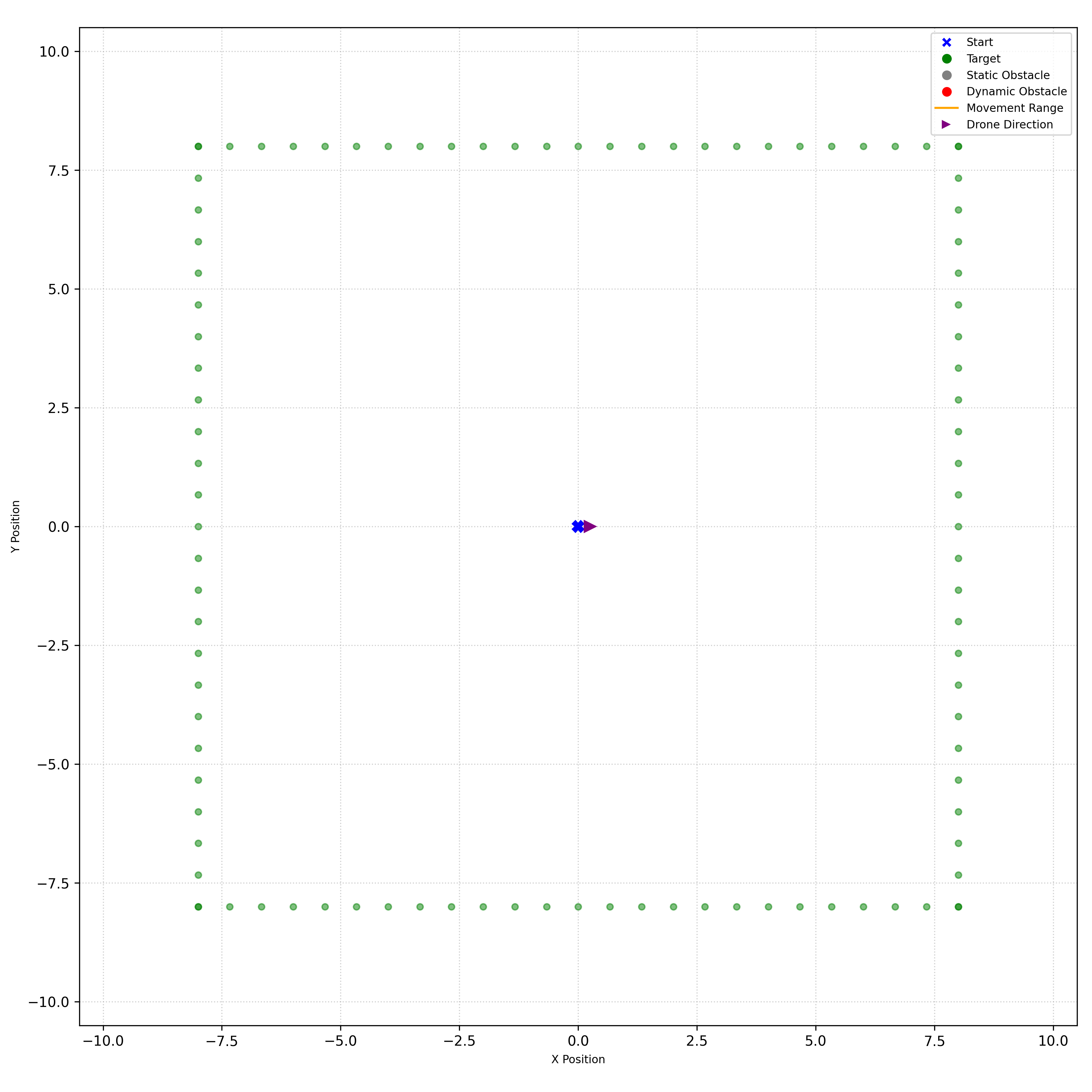} 
    \caption{20x20 No-Obstacle}
  \end{subfigure}
  \begin{subfigure}{0.24\linewidth}
    \centering
    \includegraphics[width=1.0\linewidth]{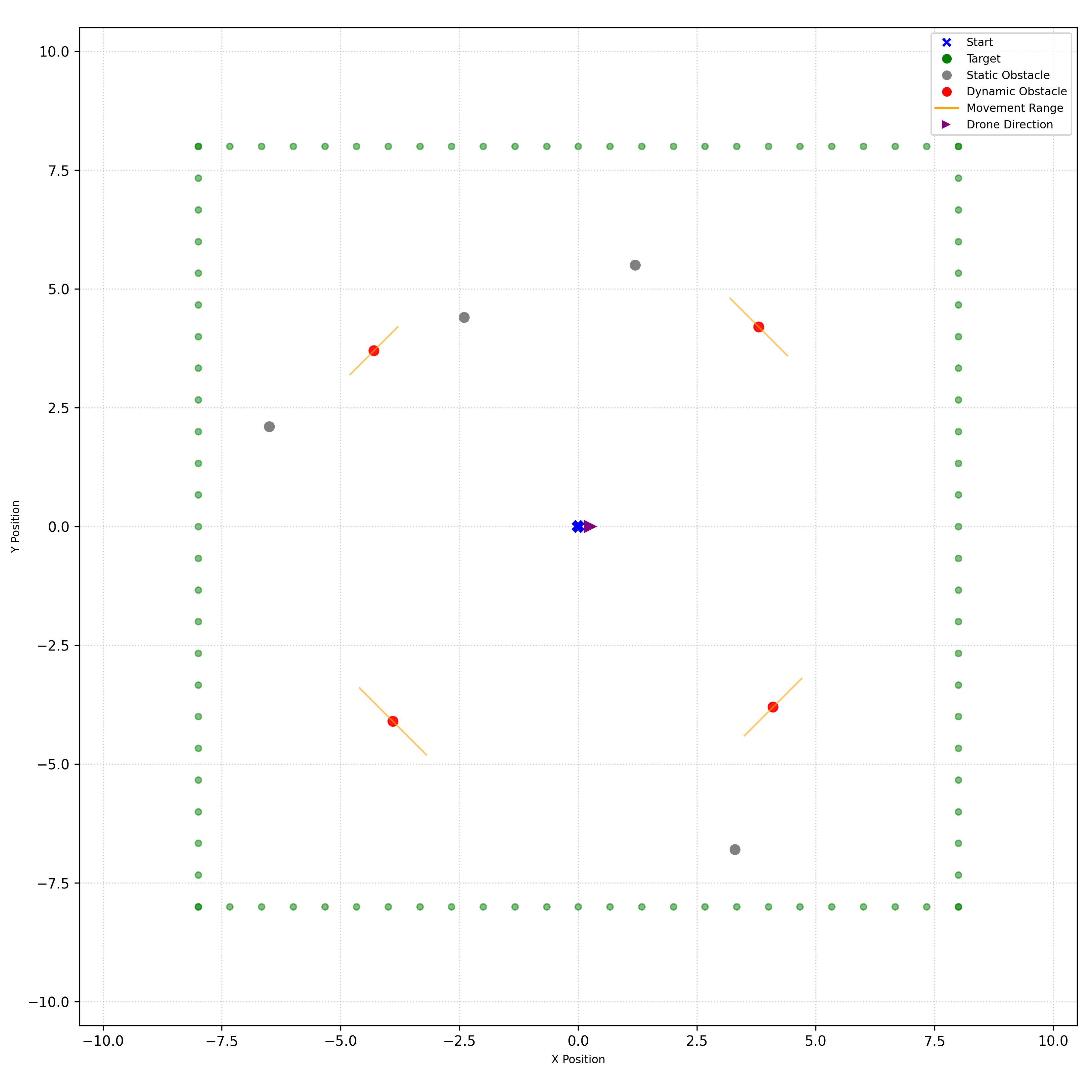} 
    \caption{20x20 Low}
  \end{subfigure}
  \begin{subfigure}{0.24\linewidth}
    \centering
    \includegraphics[width=1.0\linewidth]{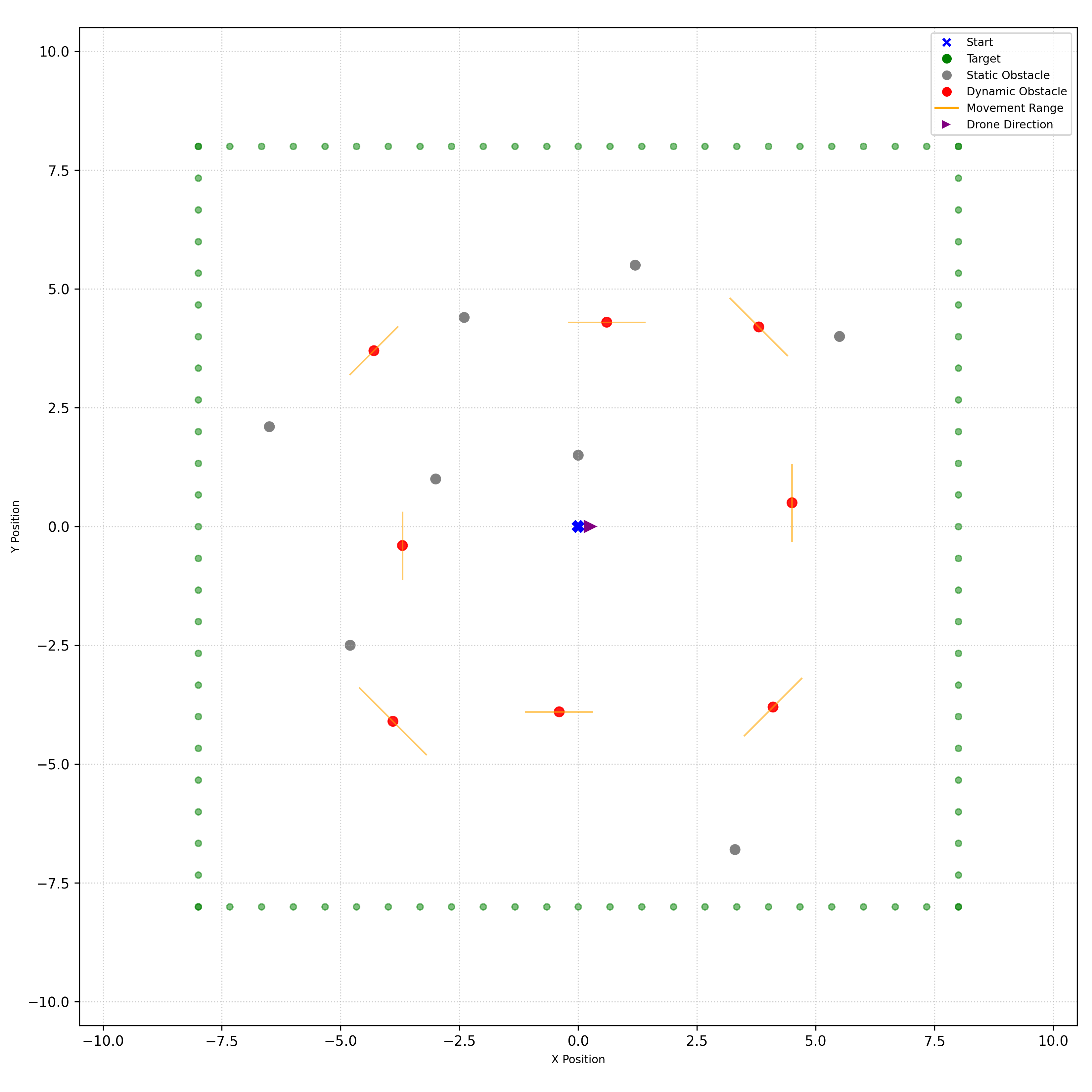} 
    \caption{20x20 Moderate}
  \end{subfigure}
  \begin{subfigure}{0.24\linewidth}
    \centering
    \includegraphics[width=1.0\linewidth]{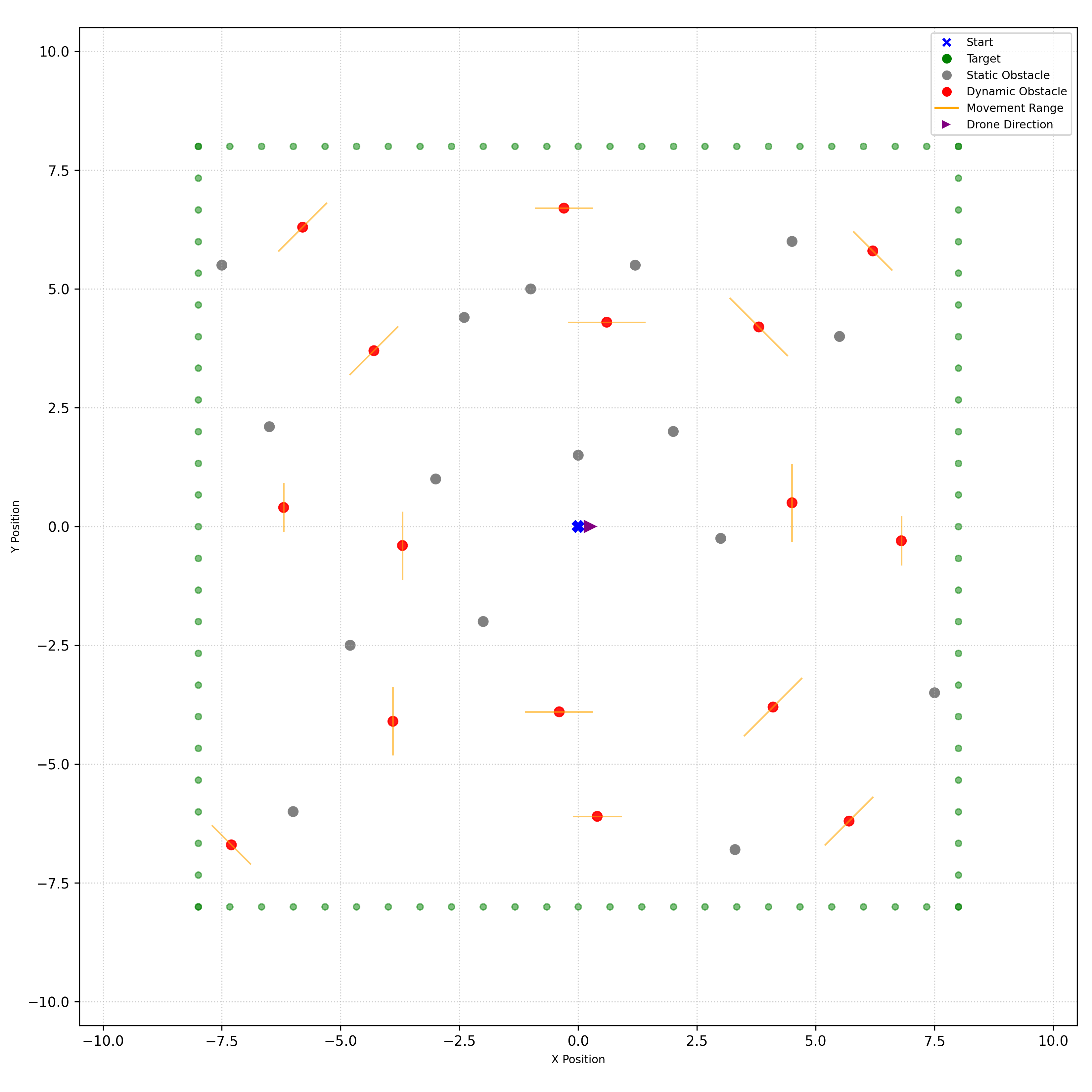} 
    \caption{20x20 High}
  \end{subfigure}
  \caption{Test environments with increasing size and congestion levels. Subfigures (a)-(d) correspond to $5\times5$ m$^2$ environments, while subfigures (e)-(h) correspond to $20\times20$ m$^2$ environments. Within each row, obstacle density increases from left to right. The blue cross denotes the drone's start position, and the purple triangle indicates its initial heading direction. The filled green circle represents the target position. Gray filled circles denote static obstacles, whereas red filled circles denote dynamic obstacles. The orange line associated with each dynamic obstacle indicates its movement range.}
  \label{fig:env_different_tests}
\end{figure}

\textbf{Performance in Small-Scale Environments:}
Table~\ref{tab:env_comparison_5} presents the results for the $5\times5$~m$^2$ environments.

In obstacle-free scenarios, all methods achieve $100\%$ success rate. MPC produces the shortest and smoothest trajectories due to its deterministic optimization strategy. However, as obstacle density increases, the advantages of the proposed DRL-based framework become more evident.

In highly congested environments, the PPO-WithMap-WithSshape method achieves the highest success rate (94\%) due to the combined benefits of the grid map and S-shape movement, outperforming PPO-NoMap-NoSshape (79\%), which lacks both enhancements; PPO-NoMap-WithSshape (81\%), which benefits only from S-shape movement; PPO-WithMap-NoSshape (88\%), which benefits only from the grid map; and MPC (90\%), whose performance is limited by its reactive planning capability in dense environments. This demonstrates the benefit of integrating dynamic grid-map information into the policy representation.

Although MPC produces shorter trajectories and lower planning times in several scenarios, its performance degrades in dense dynamic environments due to limited long-term adaptability and sensitivity to obstacle motion uncertainty.

\setlength{\tabcolsep}{1pt}
\begin{table}[h]
\centering
\caption{Performance Comparison Across Different Scenarios for 5 $\times$ 5 Environment}
\label{tab:env_comparison_5}
\begin{tabularx}{\linewidth}{
>{\hsize=0.2\hsize}X
|>{\hsize=0.48\hsize}X
|>{\hsize=0.14\hsize}X
|>{\hsize=0.11\hsize}X
|>{\hsize=0.22\hsize}X
|>{\hsize=0.22\hsize}X
|>{\hsize=0.22\hsize}X
}
\hline
\textbf{Env} & \textbf{Method} & \textbf{Success Rate (\%)} & \textbf{Avg. Plan Time per Step (ms)} & \textbf{Avg. Path Length (m)} & \textbf{Avg. Clearance (m)} & \textbf{Avg. Smoothness} \\
\hline
\multirow{5}{*}{No obstacle} 
& PPO-NoMap-NoSshape & \textbf{100}  & 101.84 & 2.59 ± 0.48 & N/A & 4.01 ± 1.14 \\
& PPO-NoMap-WithSshape & \textbf{100}  & 101.85 & 2.59 ± 0.48 & N/A & 4.02 ± 1.11 \\
& PPO-WithMap-NoSshape & \textbf{100}  & 166.69 & 2.57 ± 0.48 & N/A & \textbf{2.34 ± 1.06} \\
& PPO-WithMap-WithSshape & \textbf{100}  & 166.30 & 2.57 ± 0.48 & N/A & 2.61 ± 1.02 \\
& MPC & \textbf{100}  & \textbf{33.20} & \textbf{2.44 ± 0.32} & N/A & 2.35 ± 1.07 \\
\hline
\multirow{5}{*}{Low} 
& PPO-NoMap-NoSshape & 94  & 102.05 & 2.57 ± 0.48 & \textbf{1.66 ± 0.31} & 3.97 ± 1.15 \\
& PPO-NoGrid-WithSshape & 93  & 102.30 & 2.56 ± 0.48 & \textbf{1.66 ± 0.31} & 3.98 ± 1.08 \\
& PPO-WithMap-NoSshape & \textbf{99}  & 165.14 & 2.61 ± 0.54 & 1.64 ± 0.31 & 2.74 ± 1.60 \\
& PPO-WithMap-WithSshape & 98  & 164.36 & 2.60 ± 0.52 & 1.65 ± 0.31 & 2.90 ± 1.44 \\
& MPC & 98  & \textbf{35.57} & \textbf{2.43 ± 0.33} & \textbf{1.66 ± 0.31} & \textbf{2.45 ± 1.23} \\
\hline
\multirow{5}{*}{Moderate} 
& PPO-NoMap-NoSshape & 87  & 100.35 & 2.62 ± 0.47 & 1.20 ± 0.15 & 4.01 ± 1.14 \\
& PPO-NoGrid-WithSshape & 89  & 100.66 & 2.61 ± 0.48 & 1.19 ± 0.16 & 4.02 ± 1.10 \\
& PPO-WithMap-NoSshape & 94  & 165.75 & 2.63 ± 0.50 & 1.20 ± 0.14 & 3.00 ± 1.43 \\
& PPO-WithMap-WithSshape & \textbf{96}  & 165.04 & 2.69 ± 0.57 & 1.20 ± 0.14 & 3.33 ± 1.48 \\
& MPC & 91  & \textbf{31.78} & \textbf{2.45 ± 0.31} & \textbf{1.26 ± 0.16} & \textbf{2.69 ± 1.24} \\
\hline
\multirow{5}{*}{High} 
& PPO-NoMap-NoSshape & 79  & 99.98 & 2.64 ± 0.46 & 1.10 ± 0.05 & 4.39 ± 1.58 \\
& PPO-NoGrid-WithSshape & 81  & 100.77 & 2.62 ± 0.46 & 1.10 ± 0.05 & 4.25 ± 1.40 \\
& PPO-WithMap-NoSshape & 88  & 163.20 & 2.80 ± 0.78 & \textbf{1.11 ± 0.07} & \textbf{3.79 ± 2.34} \\
& PPO-WithMap-WithSshape & \textbf{94}  & 158.62 & 2.97 ± 0.98 & 1.10 ± 0.08 & 4.24 ± 2.55 \\
& MPC & 90  & \textbf{74.49} & \textbf{2.50 ± 0.36} & 1.19 ± 0.06 & 4.18 ± 2.59 \\
\hline
\end{tabularx}
\end{table}

\textbf{Performance in Large-Scale Environments:}
Table~\ref{tab:env_comparison_20} presents the results for the larger $20\times20$~m$^2$ environments.

The performance gap between methods becomes more pronounced as the environment scale and congestion increase. In highly congested large-scale environments, the PPO-WithMap-WithSshape method achieves an $83\%$ success rate, outperforming PPO-NoMap-NoSshape ($62\%$), PPO-NoGrid-WithSshape ($62\%$), PPO-WithMap-NoSshape ($79\%$), and MPC ($71\%$).

The results indicate that the proposed framework generalizes effectively to larger environments while maintaining robust obstacle avoidance capability. The inclusion of behavior grid map information significantly improves environmental awareness compared to PPO-only approaches.

\setlength{\tabcolsep}{1pt}
\begin{table}[h]
\centering
\caption{Performance Comparison Across Different Scenarios for 20 $\times$ 20 Environment}
\label{tab:env_comparison_20}
\begin{tabularx}{\linewidth}{
>{\hsize=0.2\hsize}X
|>{\hsize=0.48\hsize}X
|>{\hsize=0.14\hsize}X
|>{\hsize=0.11\hsize}X
|>{\hsize=0.22\hsize}X
|>{\hsize=0.22\hsize}X
|>{\hsize=0.22\hsize}X
}
\hline
\textbf{Env} & \textbf{Method} & \textbf{Success Rate (\%)} & \textbf{Avg. Plan Time per Step (ms)} & \textbf{Avg. Path Length (m)} & \textbf{Avg. Clearance (m)} & \textbf{Avg. Smoothness} \\
\hline
\multirow{5}{*}{No obstacle} 
& PPO-NoMap-NoSshape & \textbf{100}  & \textbf{33.07} & 9.23 ± 1.04 & N/A & 12.17 ± 2.22 \\
& PPO-NoMap-WithSshape & \textbf{100}  & 33.21 & 9.22 ± 1.05 & N/A & 12.24 ± 1.94 \\
& PPO-WithMap-NoSshape & \textbf{100}  & 80.58 & 9.15 ± 1.07 & N/A & 2.74 ± 0.92 \\
& PPO-WithMap-WithSshape & \textbf{100}  & 80.20 & 9.15 ± 1.06 & N/A & 5.67 ± 1.14 \\
& MPC & \textbf{100}  & 35.79 & \textbf{8.97 ± 1.00} & N/A & \textbf{2.66 ± 1.18} \\
\hline
\multirow{5}{*}{Low} 
& PPO-NoMap-NoSshape & 87  & \textbf{33.19} & 9.11 ± 1.02 & 3.07 ± 0.48 & 12.21 ± 2.15 \\
& PPO-NoMap-WithSshape & 88  & 33.21 & 9.13 ± 1.02 & \textbf{3.09 ± 0.52} & 12.13 ± 1.91 \\
& PPO-WithMap-NoSshape & \textbf{97}  & 79.27 & 9.22 ± 1.25 & 3.05 ± 0.49 & 3.54 ± 1.80 \\
& PPO-WithMap-WithSshape & \textbf{97}  & 79.77 & 9.25 ± 1.23 & 3.04 ± 0.48 & 6.17 ± 1.76 \\
& MPC & 92  & 34.55 & \textbf{8.91 ± 0.97} & 3.06 ± 0.49 & \textbf{2.97 ± 1.47} \\
\hline
\multirow{5}{*}{Moderate} 
& PPO-NoMap-NoSshape & 75  & 33.12 & 9.29 ± 1.05 & \textbf{1.76 ± 0.34} & 12.14 ± 2.02 \\
& PPO-NoMap-WithSshape & 75  & 32.68 & 9.28 ± 1.06 & \textbf{1.76 ± 0.34} & 12.25 ± 1.86 \\
& PPO-WithMap-NoSshape & 90  & 77.77 & 9.37 ± 1.20 & 1.70 ± 0.36 & 4.51 ± 2.44 \\
& PPO-WithMap-WithSshape & \textbf{92}  & 79.03 & 9.38 ± 1.12 & 1.71 ± 0.34 & 6.90 ± 2.08 \\
& MPC & 80  & \textbf{29.42} & \textbf{8.92 ± 0.95} & 1.71 ± 0.33 & \textbf{3.04 ± 1.40} \\
\hline
\multirow{5}{*}{High} 
& PPO-NoMap-NoSshape & 62& 33.42& 9.18 ± 0.98& \textbf{1.34 ± 0.27}& 12.07 ± 1.95 \\
& PPO-NoMap-WithSshape & 62  & \textbf{33.33} & 9.19 ± 0.99 & \textbf{1.34 ± 0.27} & 12.24 ± 1.69 \\
& PPO-WithMap-NoSshape & 79  & 79.06 & 9.41 ± 1.14 & 1.29 ± 0.26 & 5.78 ± 3.16 \\
& PPO-WithMap-WithSshape & \textbf{83}  & 79.69 & 9.55 ± 1.29 & 1.28 ± 0.26 & 8.06 ± 2.83 \\
& MPC & 71  & 35.91 & \textbf{8.95 ± 0.95} & 1.33 ± 0.25 & \textbf{3.67 ± 1.90} \\
\hline
\end{tabularx}
\end{table}

Figure~\ref{fig:all_paths_plots} qualitatively compares generated trajectories. In challenging scenarios, PPO-based baselines frequently collide with obstacles. In contrast, our proposed methods, PPO-WithMap-NoSshape (standard) and PPO-WithMap-WithSshape (enhanced), successfully navigate around moving obstacles and reach the target.

Overall, the experiments demonstrate that the proposed framework provides superior robustness and scalability in complex dynamic environments.

\begin{figure}
  \centering
  \begin{subfigure}{0.8\linewidth}
    \centering
    \includegraphics[width=1\linewidth]{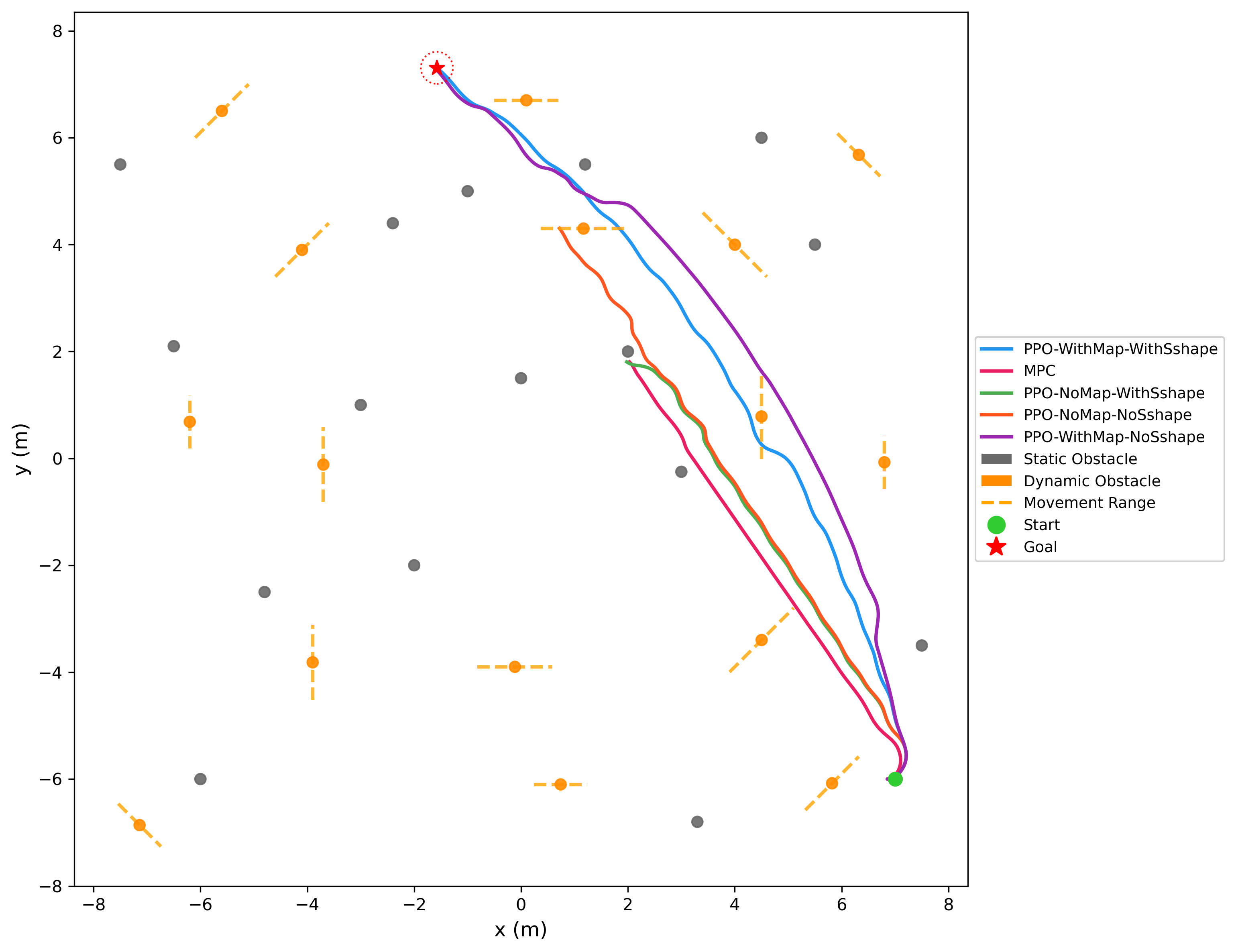} 
    \caption{}
  \end{subfigure}
  \begin{subfigure}{0.8\linewidth}
    \centering
    \includegraphics[width=1\linewidth]{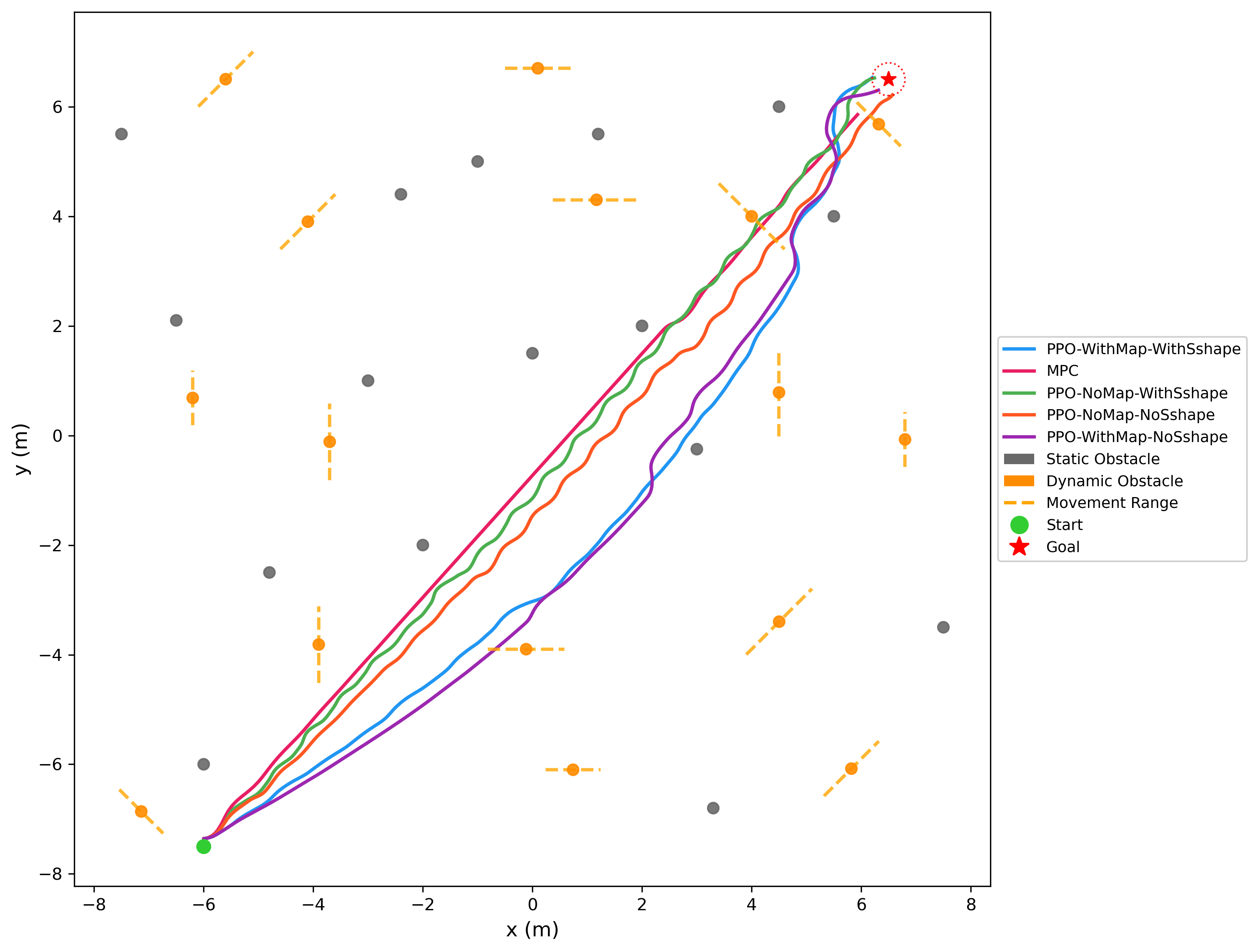} 
    \caption{}
  \end{subfigure} 
  \caption{Plot of the shape of the paths of different methods. a) PPO-WithMap-WithSshape (enhanced) and PPO-WithMap-NoSshape (standard) methods are achieved to the target, and other methods intersect with the obstacle. b) All methods could achieve the target except MPC.}
  \label{fig:all_paths_plots}
\end{figure}

\subsection{Effect of S-shape Movement on Battery Consumption and Safety}

To analyze the effect of the proposed S-shape exploration behavior on navigation safety and energy consumption, additional experiments were conducted in high-congestion $20 \times 20$~m$^2$ environments using different values of the S-shape amplitude parameter ($A$) and angular frequency parameter ($\boldsymbol{\omega}$). These experiments evaluate the trade-off between navigation robustness and battery consumption introduced by the enhanced exploration strategy.

During the study of the amplitude parameter $A$, the angular frequency was fixed at $\boldsymbol{\omega} = (0.8)2\pi$. Similarly, during the study of $\boldsymbol{\omega}$, the amplitude was fixed at $A = 0.5$. These values were selected because they achieved the best overall navigation performance in terms of success rate and trajectory quality. The same parameter configuration ($A=0.5$ and $\boldsymbol{\omega} = (0.8)2\pi$) was used in all previous experiments involving the enhanced method.

Table~\ref{tab:energy_comparison_A} presents the influence of different amplitude values on navigation performance and battery consumption. To quantify the energy metrics, battery consumption is measured directly within the Webots simulation environment by enabling the robot's built-in battery sensor API. The simulation tracks energy depletion in Joules by continuously aggregating the basal electrical power of the Crazyflie 2.1 electronics alongside the dynamic power drawn by the four main propellers as their rotational speeds fluctuate. The total consumption for each navigation trial is determined by calculating the absolute difference between the initial energy capacity at takeoff and the remaining energy recorded upon reaching the target destination. When $A=0$, corresponding to the standard method without S-shape movement, the success rate reaches $80\%$ with an average battery consumption of $397.53$~J. Introducing a moderate S-shape motion with $A=0.5$ improves the success rate to $86\%$, demonstrating that the additional lateral exploration enhances environmental awareness and obstacle avoidance capability in highly congested environments. However, this improvement comes with a slight increase in battery consumption to $401.46$~J and an increase in path smoothness values due to the oscillatory motion. Increasing the amplitude further to $A=1.0$ decreases the success rate to $82\%$ and produces less stable trajectories with excessive deviations from the direct path. These results indicate that overly aggressive lateral movements reduce navigation efficiency.

\begin{table}[h]
\centering
\caption{Average battery consumption and safety comparison for different \( A \) values in high-congestion $20 \times 20$~m$^2$ environments.}
\label{tab:energy_comparison_A}
\begin{tabularx}{\linewidth}{
>{\hsize=0.05\hsize}X
|>{\hsize=0.1\hsize}X
|>{\hsize=0.13\hsize}X
|>{\hsize=0.15\hsize}X
|>{\hsize=0.15\hsize}X
|>{\hsize=0.22\hsize}X
|>{\hsize=0.22\hsize}X
}
\hline
\textbf{\( A \)} & \textbf{Success Rate (\%)} & \textbf{Avg. Plan Time per Step (ms)} & \textbf{Avg. Path Length (m)} & \textbf{Avg. Clearance (m)} & \textbf{Avg. Smoothness} & \textbf{Avg. Battery Consumption (J)} \\
\hline
0 & 80 & 78.28 & 9.74 ± 1.28 & 1.30 ± 0.26 & 6.14 ± 3.35 & 397.53 ± 39.06 \\
0.5 & \textbf{86} & 78.08 & 9.86 ± 1.41 & 1.29 ± 0.26 & 8.12 ± 3.02 & 401.46 ± 42.96 \\
1.0 & 82 & 78.31 & 9.86 ± 1.38 & 1.30 ± 0.26 & 9.71 ± 2.75 & 401.77 ± 42.14 \\
\hline
\end{tabularx}
\end{table}

Table~\ref{tab:energy_comparison_w} shows the effect of different angular frequency values. A small angular frequency value, $\boldsymbol{\omega}=(0.1)2\pi$, produces weak lateral oscillations that have limited impact on environmental exploration, resulting in a lower success rate of $79\%$. In contrast, the moderate value $\boldsymbol{\omega}=(0.8)2\pi$ achieves the highest success rate of $86\%$, indicating an effective balance between exploration and trajectory oscillation. Increasing the angular frequency to $\boldsymbol{\omega}=(1.6)2\pi$ causes excessive oscillatory behavior, which increases navigation oscillation and slightly decreases the success rate to $80\%$. Similar to the amplitude study, the battery consumption differences remain relatively small across all configurations, showing that the improved safety and success rate of the enhanced method are achieved with only a minor increase in energy consumption.

Overall, the results demonstrate that the proposed S-shape movement improves navigation robustness and obstacle avoidance performance in highly congested environments by increasing environmental perception coverage. Although the enhanced exploration behavior introduces a small increase in battery consumption and trajectory oscillation, the additional energy cost remains limited compared to the achieved improvement in navigation safety and success rate.

\begin{table}[h]
\centering
\caption{Average battery consumption and safety comparison for different \( \boldsymbol{\omega} \) values in high-congestion $20 \times 20$~m$^2$ environments.}
\label{tab:energy_comparison_w}
\begin{tabularx}{\linewidth}{
>{\hsize=0.1\hsize}X
|>{\hsize=0.1\hsize}X
|>{\hsize=0.13\hsize}X
|>{\hsize=0.17\hsize}X
|>{\hsize=0.17\hsize}X
|>{\hsize=0.2\hsize}X
|>{\hsize=0.22\hsize}X
}
\hline
\textbf{\( \boldsymbol{\omega} \)} & \textbf{Success Rate (\%)} & \textbf{Avg. Plan Time per Step (ms)} & \textbf{Avg. Path Length (m)} & \textbf{Avg. Clearance (m)} & \textbf{Avg. Smoothness} & \textbf{Avg. Battery Consumption (J)} \\
\hline
$(0.1)2\pi$ & 79 & 78.45 & 9.74 ± 1.24 & 1.29 ± 0.26 & 6.73 ± 2.98 & 397.80 ± 37.82 \\
$(0.8)2\pi$ & \textbf{86} & 78.08 & 9.86 ± 1.41 & 1.29 ± 0.26 & 8.12 ± 3.02 & 401.46 ± 42.96 \\
$(1.6)2\pi$ & 80 & 78.10 & 9.71 ± 1.14 & 1.31 ± 0.26 & 7.97 ± 2.39 & 396.88 ± 34.73 \\
\hline
\end{tabularx}
\end{table}

\subsection{Computational Efficiency and Deployment Feasibility}

One of the main advantages of the proposed framework is its suitability for deployment on resource-limited UAV platforms using inexpensive hardware and low-cost sensors. Unlike many existing UAV navigation approaches that depend on computationally expensive perception systems such as LiDARs, RGB-D cameras, or high-resolution mapping modules, the proposed method relies only on sparse rangefinder measurements together with a lightweight PPO policy network.

Table~\ref{tab:model_efficiency} summarizes the computational characteristics of the trained PPO model. The policy network consists of a compact multilayer perceptron architecture with two hidden layers of 64 neurons each, resulting in only 17,283 trainable parameters and a memory footprint of approximately 0.066~MB. This compact architecture enables deployment on embedded processors, micro-SBCs, or lightweight onboard computing units commonly used in low-cost UAV systems.

In addition to the low memory requirement, the proposed model achieves an average inference latency of only 0.81~ms, corresponding to a theoretical maximum control frequency of approximately 1238~Hz. This execution speed is significantly higher than the typical control frequencies required for real-time UAV navigation, which commonly range between 50 and 100~Hz. Therefore, the proposed framework can generate navigation commands in real time with negligible computational overhead, even in dynamic and cluttered environments.

The results demonstrate that the proposed DRL framework provides an effective balance between navigation performance and computational efficiency. The lightweight architecture, combined with the use of inexpensive rangefinder sensors, makes the proposed approach particularly suitable for low-cost UAV platforms operating under hardware and energy constraints.

\begin{table}[h]
\centering
\caption{Computational Resource Analysis of the Trained PPO Policy}
\label{tab:model_efficiency}
\begin{tabularx}{\linewidth}{
>{\hsize=0.5\hsize}X
|>{\hsize=0.5\hsize}X
}
\hline
\textbf{Metric} & \textbf{Value} \\
\hline
Network Architecture & MLP [64, 64] \\
Total Parameters & 17,283 \\
Model Size (Memory Footprint) & 0.0659 MB \\
Observation Space Dimension & 68 \\
Action Space Dimension & 1 (Continuous) \\
Average Inference Latency & 0.81 ms \\
Maximum Control Frequency & 1238.3 Hz \\
\hline
\end{tabularx}
\end{table}

\subsection{Real-World Experimental Results} \label{subsec:Real_World_Experimental_Results}
The proposed algorithm was validated in a real-world setup using a Crazyflie~2.1 quadcopter\footnote{\url{https://www.bitcraze.io/products/crazyflie-2-1-plus/}}. The Crazyflie was equipped with a Flow Deck v2\footnote{\url{https://www.bitcraze.io/products/flow-deck-v2/}} for onboard position estimation, and a Multi-ranger Deck\footnote{\url{https://www.bitcraze.io/products/multi-ranger-deck/}} to measure the distance to the closest obstacle in three directions: front, left, and right (Figure~\ref{fig:env_setup}a). Communication between the Crazyflie and a host PC was established using the Crazyradio PA\footnote{\url{https://www.bitcraze.io/products/crazyradio-pa/}}(Figure~\ref{fig:env_setup}b). Figures~\ref{fig:env_setup}c--f show the testing environments: without obstacles (c), with one static obstacle (d), with four static obstacles (e), and with one dynamic obstacle (f). The environment size is $1.63\times2.0~m^2$.
\begin{figure}
  \centering
  \begin{subfigure}{0.32\linewidth}
    \centering
    \includegraphics[width=\linewidth, height=0.75\linewidth]{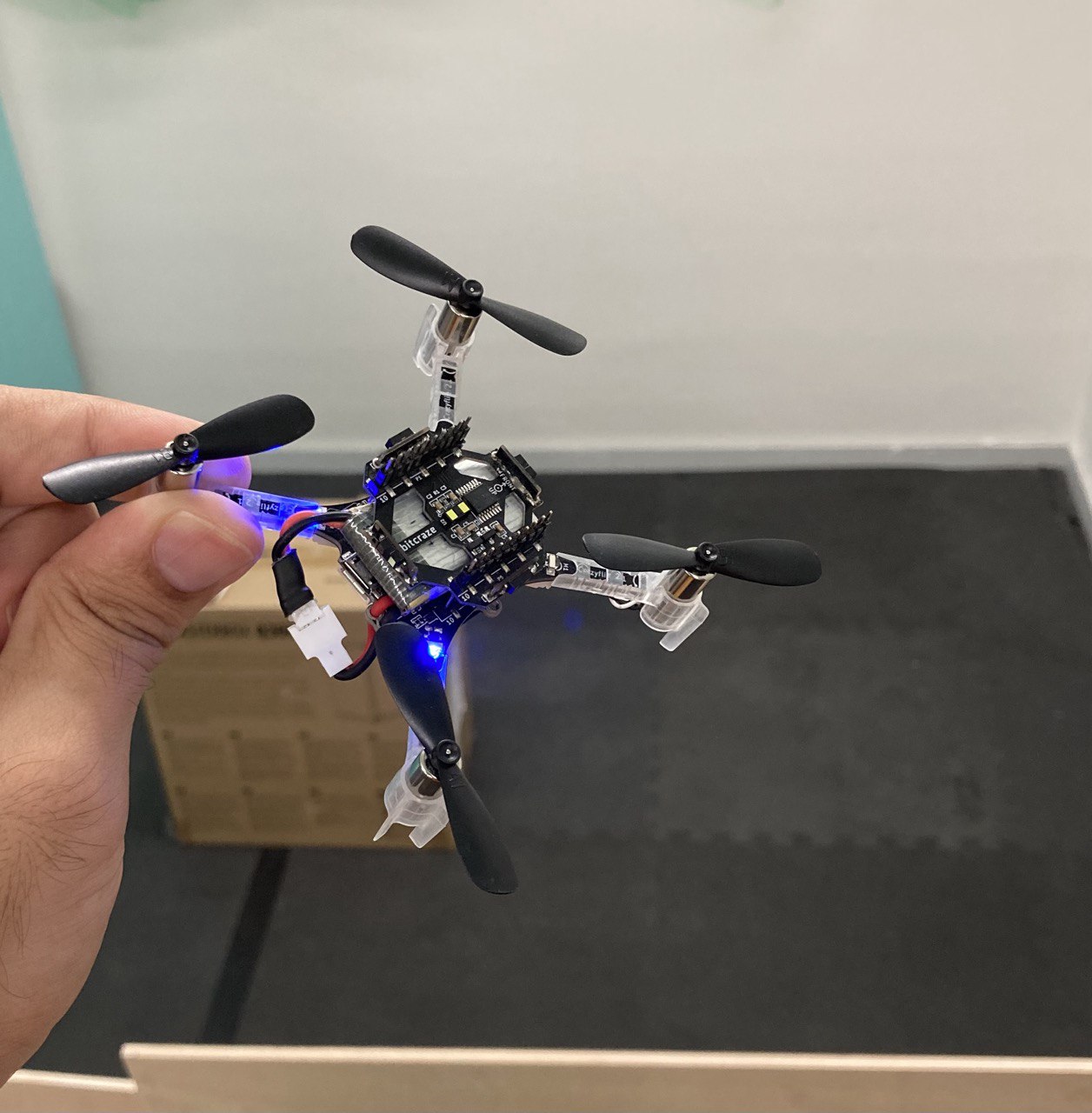} 
    \caption{}
  \end{subfigure}
  \begin{subfigure}{0.32\linewidth}
    \centering
    \includegraphics[width=\linewidth, height=0.75\linewidth]{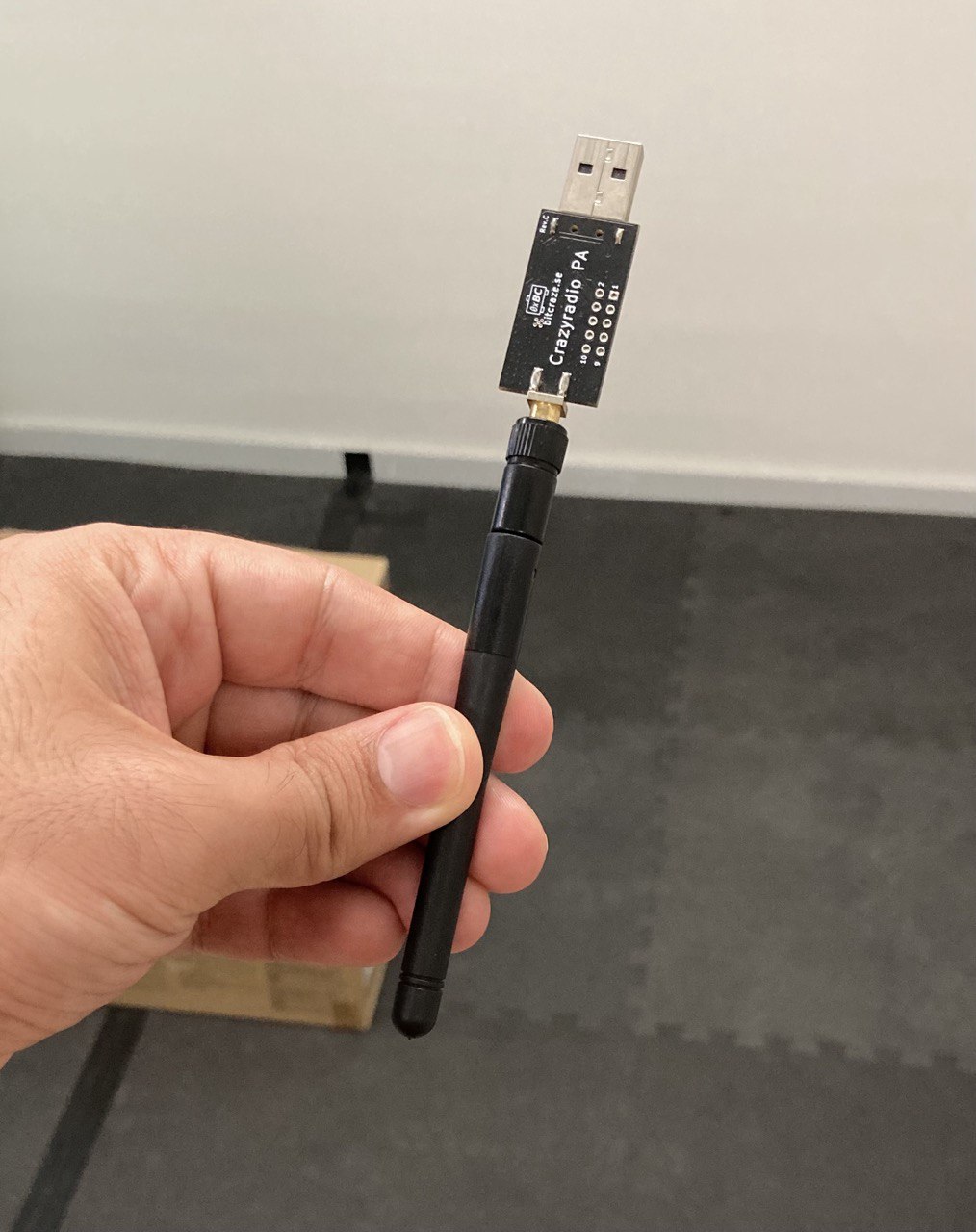} 
    \caption{}
  \end{subfigure}
  \begin{subfigure}{0.32\linewidth}
    \centering
    \includegraphics[width=\linewidth, height=0.75\linewidth]{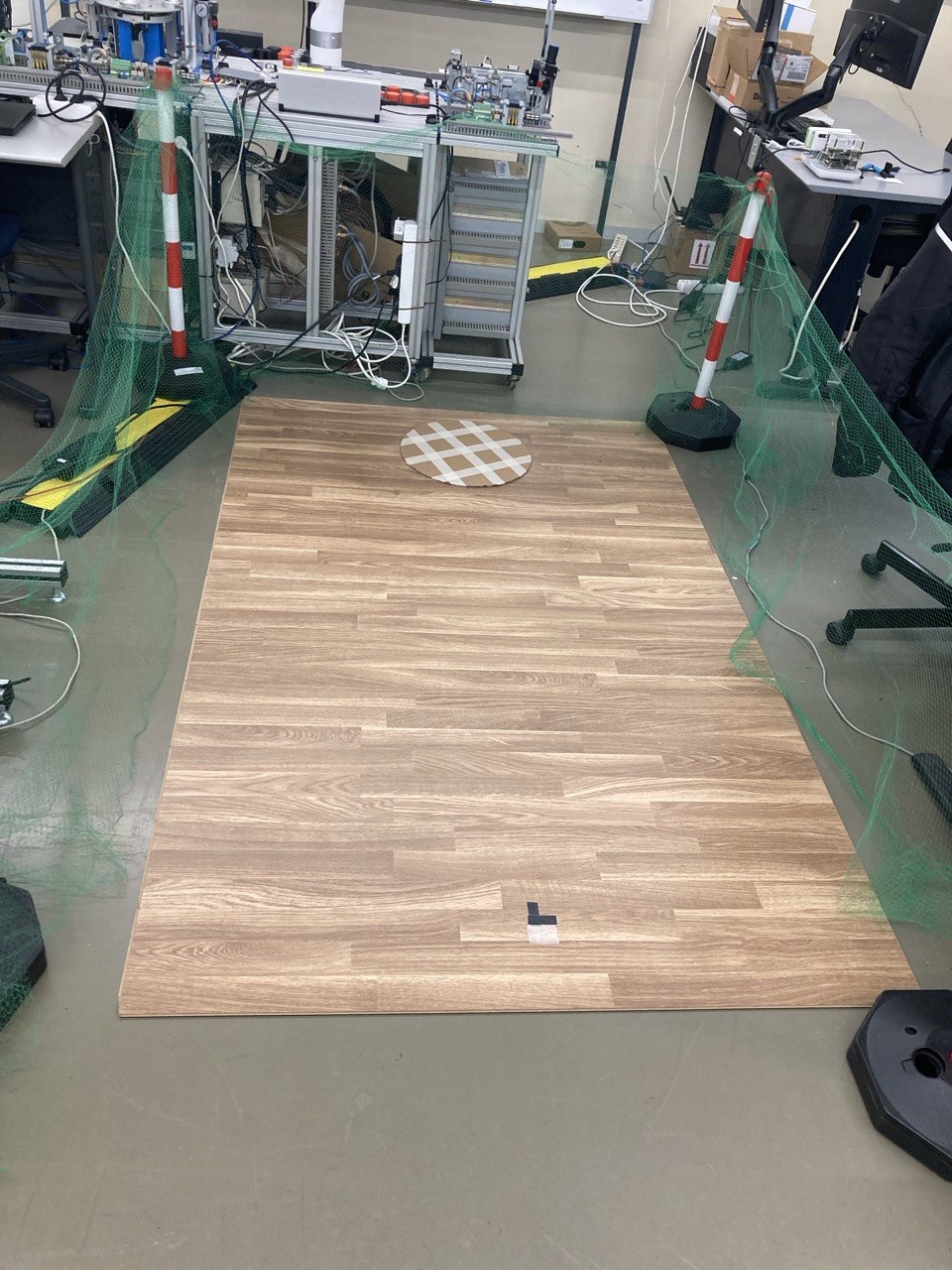} 
    \caption{}
  \end{subfigure}
  \begin{subfigure}{0.32\linewidth}
    \centering
    \includegraphics[width=\linewidth, height=0.75\linewidth]{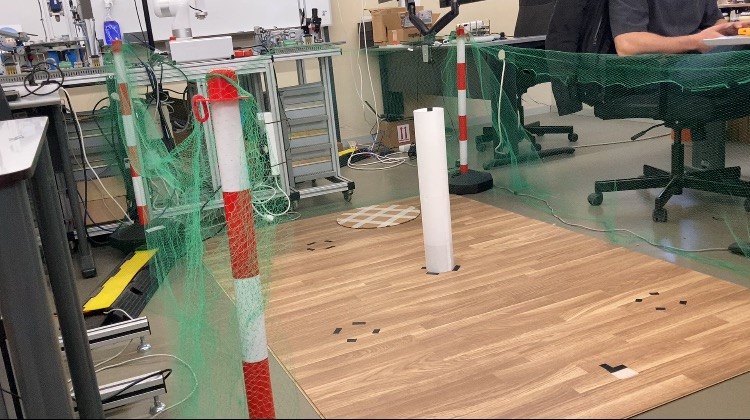} 
    \caption{}
  \end{subfigure}
  \begin{subfigure}{0.32\linewidth}
    \centering
    \includegraphics[width=\linewidth, height=0.75\linewidth]{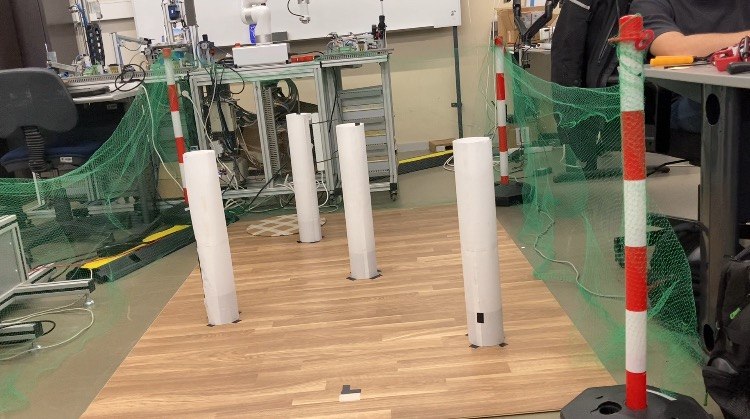} 
    \caption{}
  \end{subfigure}
  \begin{subfigure}{0.32\linewidth}
    \centering
    \includegraphics[width=\linewidth, height=0.75\linewidth]{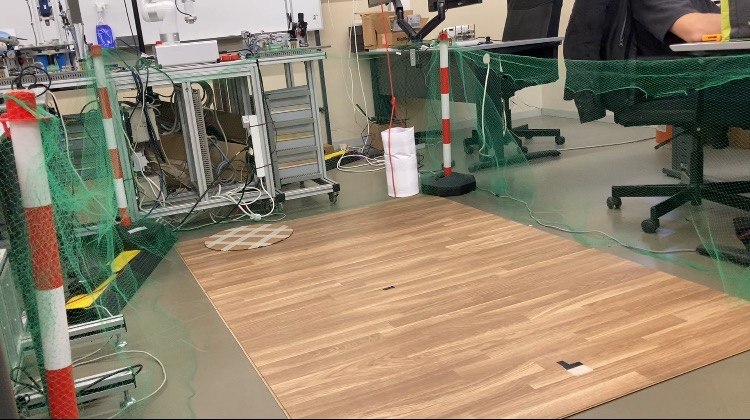} 
    \caption{}
  \end{subfigure}
\caption{Real-world experimental setup: (a) Crazyflie~2.1 nano quadcopter, (b) Crazyradio PA communication module, and experimental environments including (c) obstacle-free, (d) one static obstacle, (e) four static obstacles, and (f) one dynamic obstacle scenarios.}
  \label{fig:env_setup}
\end{figure}

The host PC ran Ubuntu~22.04.5 LTS with an Intel Core i7 CPU, 16~GB RAM, and an NVIDIA RTX 3060~Ti GPU. The onboard sensors provide distance measurements and localization feedback, while the proposed method runs externally on a host PC and sends velocity commands to the UAV in real time.

The Crazyflie receives velocity commands from the PC every 100~ms via the Crazyradio. These commands are generated by the proposed method running on the PC. The Crazyflie runs the official firmware\footnote{\url{https://github.com/bitcraze/crazyflie-firmware}}, which includes a PID controller that translates the velocity commands into motor outputs. The onboard STM32F405 microcontroller (Cortex-M4, 168 MHz, 192 KB SRAM, 1 MB flash) handles these operations. Feedback from the Flow Deck, including estimated position and yaw angle in the robot frame, is sent back to the PC for use in the proposed method. Before deployment, the scenario was simulated in Gazebo Garden using CrazySim~\cite{LlanesICRA2024}. The policy was trained from scratch in Gazebo to adapt to the real-world setting.

Four experimental scenarios were considered: (i) obstacle-free, (ii) one static obstacle, (iii) four static obstacles, and (iv) one dynamic obstacle. All obstacles had a cylindrical shape with a radius of 0.05~m, and the flight altitude was fixed at 0.4~m throughout all experiments.

\textbf{Obstacle-free:}
The first experiment evaluated the UAV performance in an obstacle-free environment with different start and goal positions. The quantitative results are summarized in Table~\ref{tab:real_test_info_no_obstacle}, while the corresponding trajectories are illustrated in Figure~\ref{fig:paths_episodes_no_obstacle}. In each subplot of this figure, the drone's starting position is designated by a small green-filled circle, and the goal destination is indicated by a red asterisk surrounded by a dotted red region. The drone's trajectory is depicted as a blue line, with arrows representing its heading at each timestep. The color of these arrows corresponds to the heading angle, mapped using the color bar on the right side of the plot, where $\pi$ is yellow, $-\pi$ is purple, and intermediate values scale continuously between them. All subsequent trajectory plots in this paper follow this identical visual convention. 

The UAV successfully reached the target region in all test episodes. The average planning time remained nearly constant at approximately 32--33~ms, demonstrating stable real-time performance. The generated trajectories remained smooth and confirmed the reliability of the proposed approach in basic navigation tasks.

\begin{table}[h]
\centering
\caption{Quantitative results of real-world UAV navigation experiments in an obstacle-free environment.}
\label{tab:real_test_info_no_obstacle}
\begin{tabularx}{\linewidth}{
>{\hsize=0.15\hsize}X
|>{\hsize=0.25\hsize}X
|>{\hsize=0.35\hsize}X
|>{\hsize=0.25\hsize}X
}
\hline
\textbf{Episode} & \textbf{Path Length (m)} & \textbf{Avg. Plan Time (ms)} & \textbf{Smoothness} \\
\hline
a & 0.686 & 33.29 & 1.337 \\
b & 1.115 & 33.00 & 2.186 \\
c & 1.297 & 32.02 & 1.716 \\
\hline
\end{tabularx}
\end{table}

\begin{figure}
  \centering
  \begin{subfigure}{0.32\linewidth}
    \centering
    \includegraphics[width=\linewidth, height=0.75\linewidth]{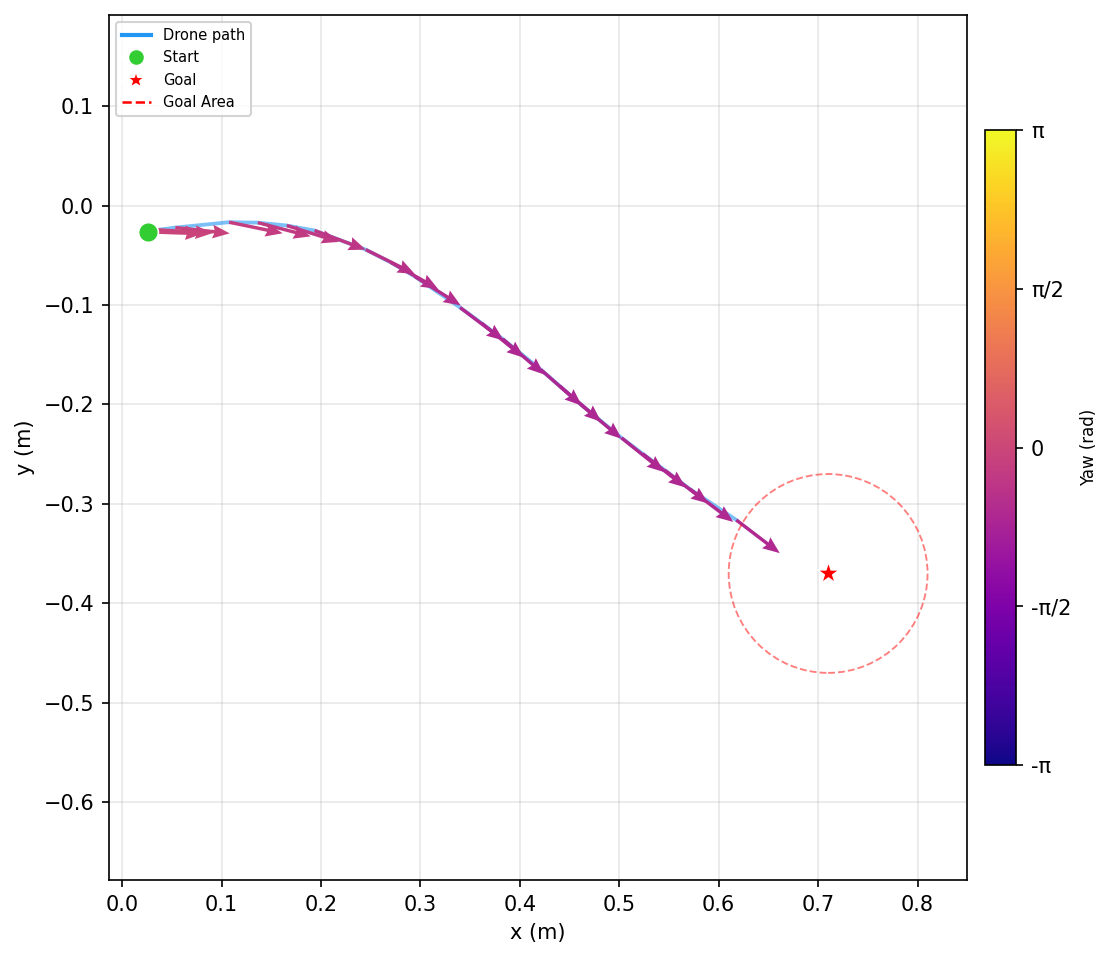} 
    \caption{}
  \end{subfigure}
  \begin{subfigure}{0.32\linewidth}
    \centering
    \includegraphics[width=\linewidth, height=0.75\linewidth]{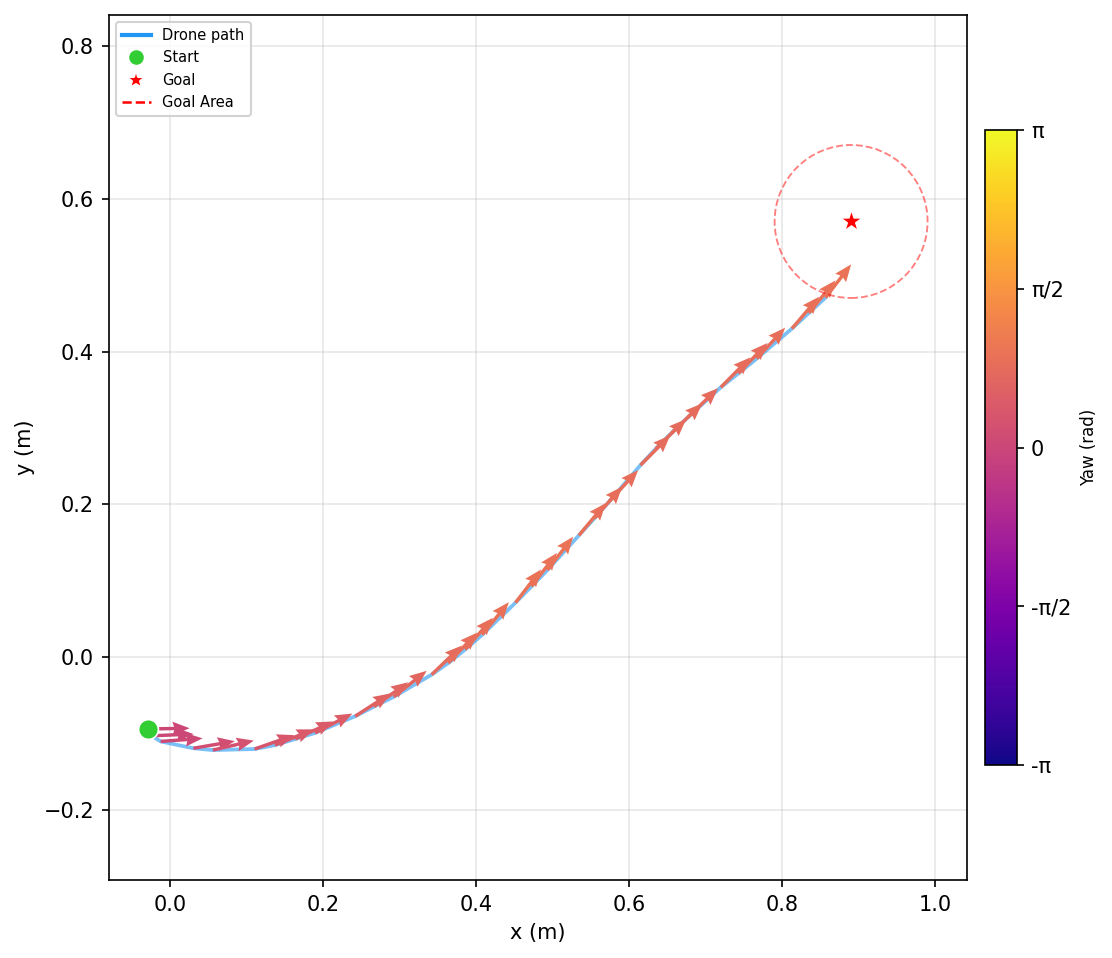} 
    \caption{}
  \end{subfigure}
  \begin{subfigure}{0.32\linewidth}
    \centering
    \includegraphics[width=\linewidth, height=0.75\linewidth]{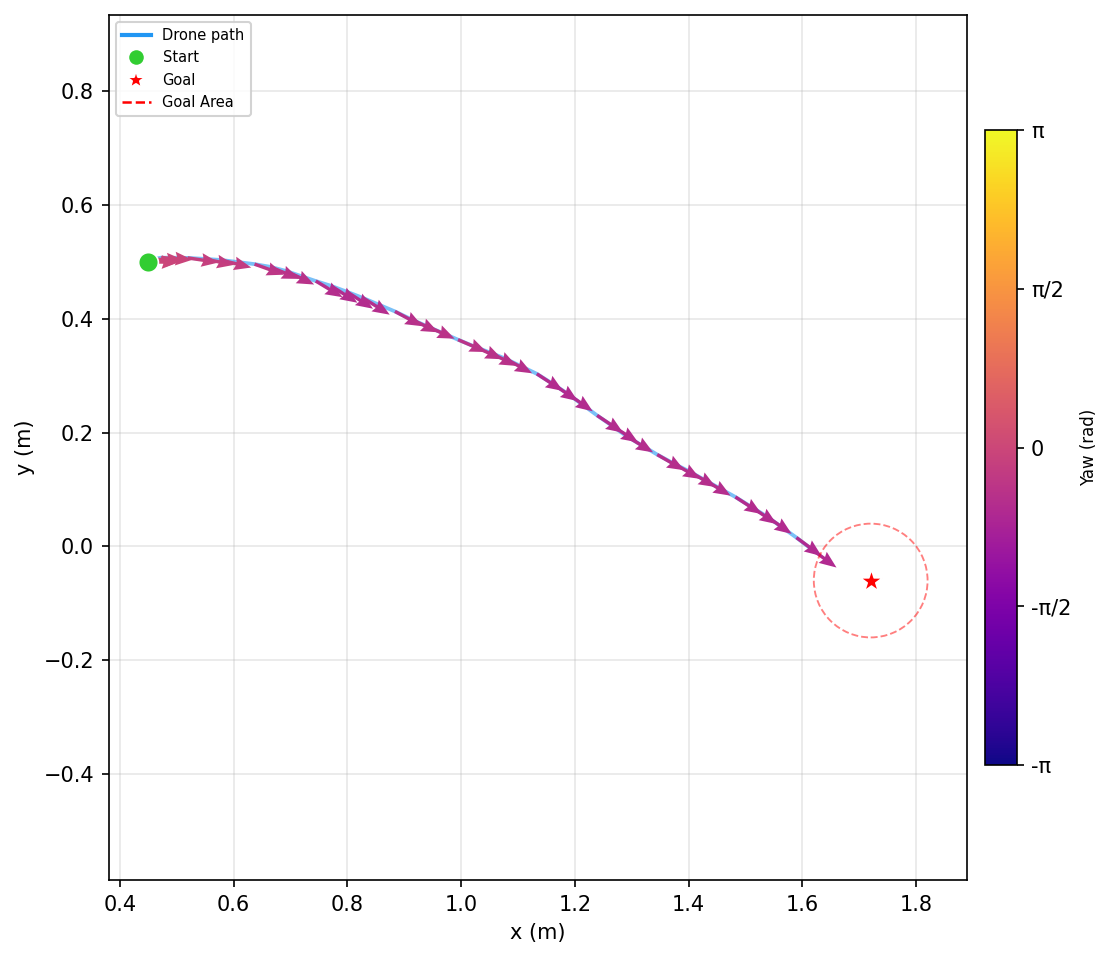} 
    \caption{}
  \end{subfigure}
  \caption{The plots of real-world UAV trajectories in the obstacle-free environment for different experimental episodes.}
  \label{fig:paths_episodes_no_obstacle}
\end{figure}

Figure~\ref{fig:video_shot_no_obstacle}a displays a composite snapshot of a successful experimental flight, constructed by overlaying sequential frames of the drone onto a single static background to visualize its complete path from takeoff to landing. The corresponding trajectory data for this entire episode is mapped out in Figure~\ref{fig:video_shot_no_obstacle}b.

\begin{figure}
  \centering
  \begin{subfigure}{0.48\linewidth}
    \centering
    \includegraphics[width=\linewidth, height=0.75\linewidth]{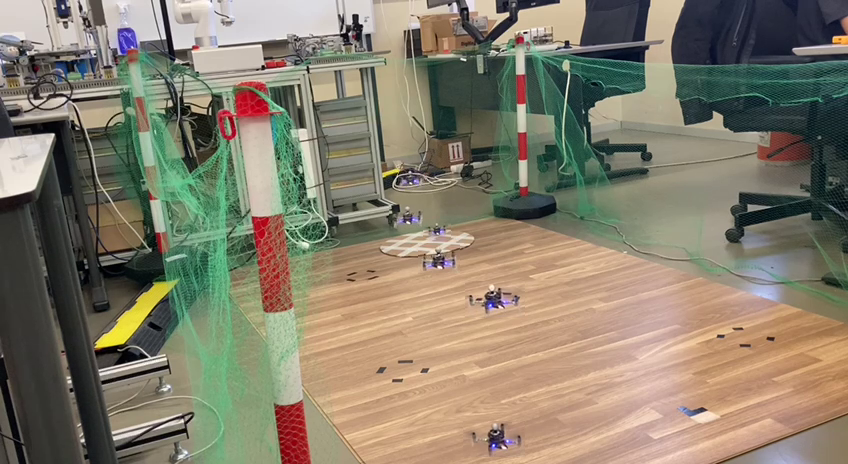} 
    \caption{}
  \end{subfigure}
  \begin{subfigure}{0.48\linewidth}
    \centering
    \includegraphics[width=\linewidth, height=0.75\linewidth]{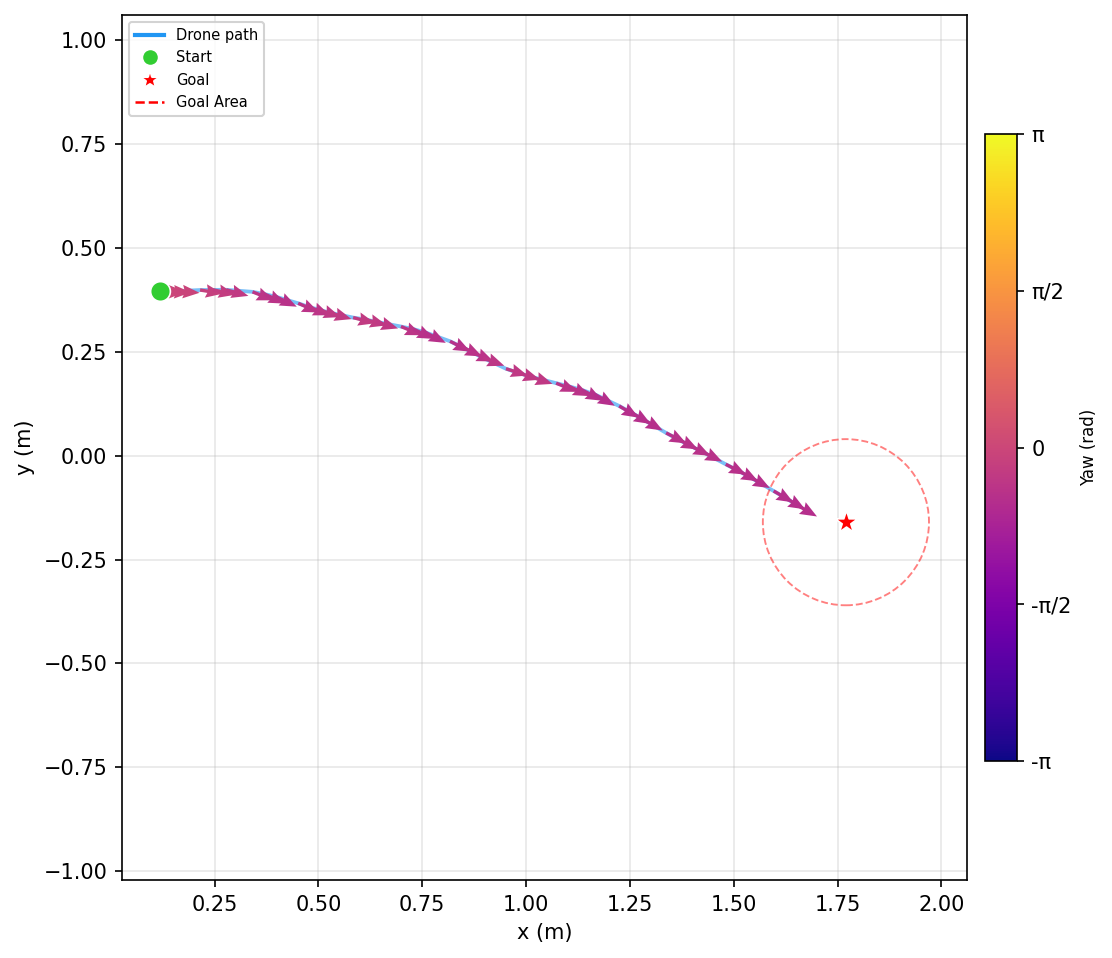} 
    \caption{}
  \end{subfigure}
 \caption{Experimental results for the obstacle-free scenario: (a) Composite snapshot tracking the drone's sequential positions from takeoff to landing against a fixed background, and (b) the corresponding 2D trajectory plot of the full episode.}
  \label{fig:video_shot_no_obstacle}
\end{figure}

\textbf{One Static Obstacle:}
The second experiment evaluated the UAV in the presence of a single static obstacle positioned between the start and target regions. Table~\ref{tab:real_test_info_one_static_obstacle} reports the experimental results, and the corresponding trajectories are shown in Figure~\ref{fig:paths_episodes_one_static}. In each subplot of this figure, the static obstacle is designated by a gray-filled circle.

In episode~a, the static obstacle is located directly between the starting position and the target. When the drone's front range sensor detects the obstacle, the UAV initiates a clockwise rotation. As the maneuver progresses, the left-side sensor detects the obstacle, which prompts the drone to turn further in the same direction. After successfully clearing the obstacle, the drone steers back toward the goal to compensate for its heading angle error, reaching the target successfully.

In episode~b, the target is located along the positive $y$-axis, and the obstacle is offset from the direct line of sight between the start and goal positions. Consequently, the obstacle exerts less influence on the drone's path, resulting in a shallower, less aggressive turn during the avoidance maneuver. A symmetric behavior is observed in episode~c, where a similar trajectory deviation occurs toward the negative $y$-axis to accommodate a target shifted in that direction.

The UAV successfully avoided the obstacle in all experiments while maintaining a stable trajectory. The average clearance values ranged from 0.413~m to 0.508~m, showing that the controller preserved a safe distance from the obstacle during navigation. The planning time remained close to 31--33~ms, which is comparable to the obstacle-free case and indicates that obstacle avoidance did not significantly increase computational cost. The results demonstrate that the proposed method can efficiently plan safe and feasible trajectories in the presence of a static obstacle.

\begin{table}[h]
\centering
\caption{Quantitative results of real-world UAV navigation experiments in the presence of one static obstacle.}
\label{tab:real_test_info_one_static_obstacle}
\begin{tabularx}{\linewidth}{
>{\hsize=0.1\hsize}X
|>{\hsize=0.2\hsize}X
|>{\hsize=0.2\hsize}X
|>{\hsize=0.2\hsize}X
|>{\hsize=0.2\hsize}X
}
\hline
\textbf{Episode} & \textbf{Path Length (m)} & \textbf{Avg. Clearance (m)} & \textbf{Avg. Plan Time (ms)} & \textbf{Smoothness} \\
\hline
a & 1.684 & 0.464 & 33.262 & 4.313 \\
b & 1.340 & 0.413 & 31.662 & 4.434 \\
c & 1.449 & 0.508 & 31.062 & 5.009 \\
\hline
\end{tabularx}
\end{table}

\begin{figure}
  \centering
  \begin{subfigure}{0.32\linewidth}
    \centering
    \includegraphics[width=\linewidth, height=0.75\linewidth]{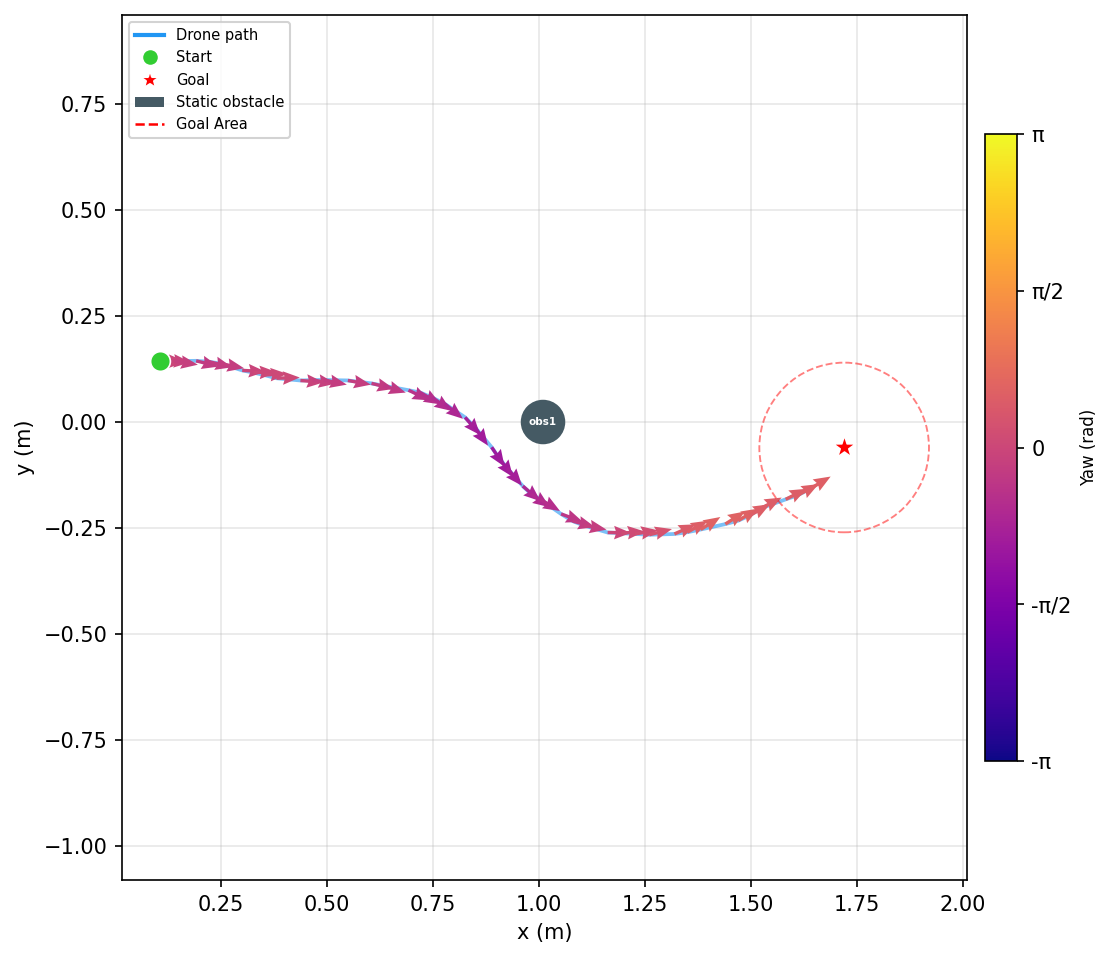} 
    \caption{}
  \end{subfigure}
  \begin{subfigure}{0.32\linewidth}
    \centering
    \includegraphics[width=\linewidth, height=0.75\linewidth]{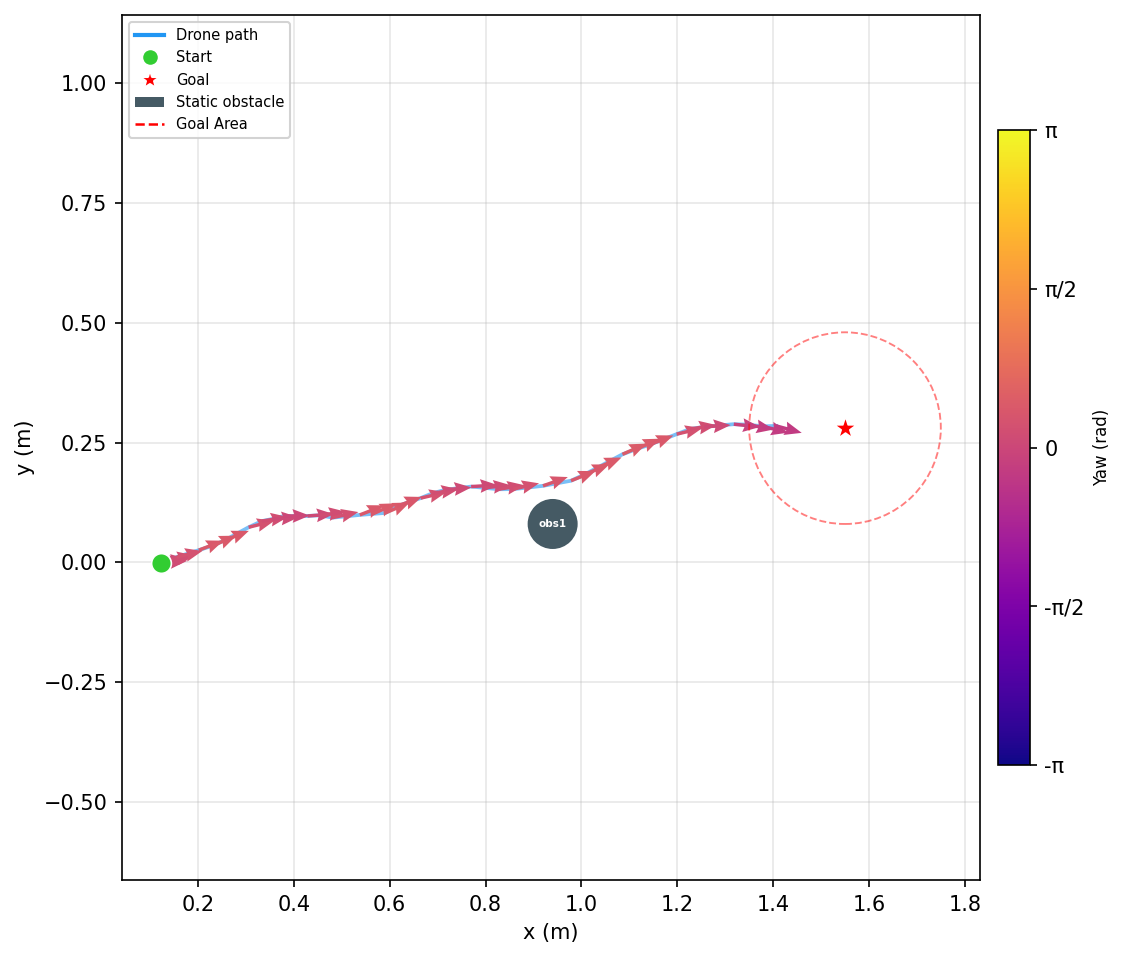} 
    \caption{}
  \end{subfigure}
  \begin{subfigure}{0.32\linewidth}
    \centering
    \includegraphics[width=\linewidth, height=0.75\linewidth]{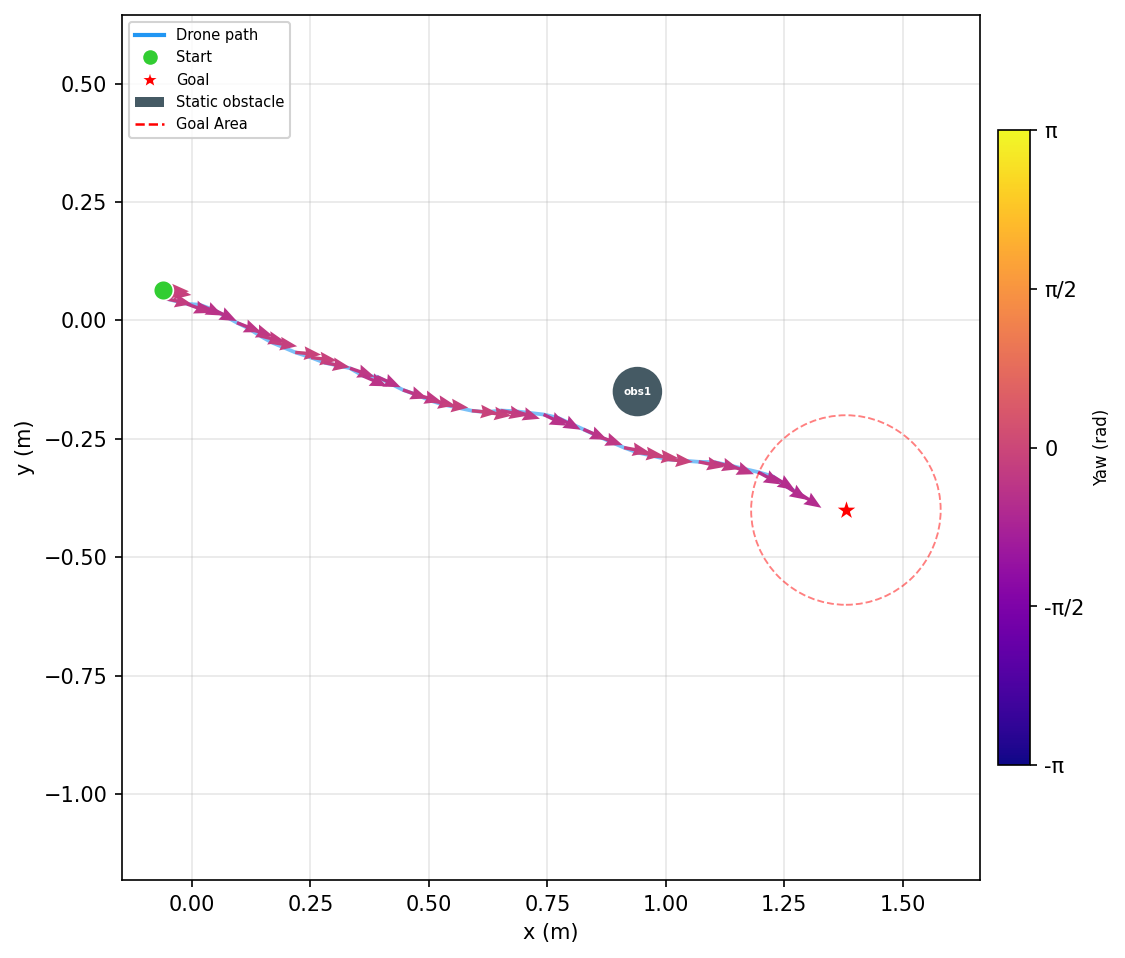} 
    \caption{}
  \end{subfigure}
 \caption{The plots of real-world UAV trajectories in the environment containing one static obstacle.}
  \label{fig:paths_episodes_one_static}
\end{figure}

Figure~\ref{fig:video_shot_one_static}a presents a composite snapshot of a successful flight navigating around a single static obstacle, created by overlaying sequential video frames onto a fixed background to capture the full trajectory from takeoff to landing. Figure~\ref{fig:video_shot_one_static}b illustrates the corresponding trajectory data recorded for this episode.

\begin{figure}
  \centering
  \begin{subfigure}{0.48\linewidth}
    \centering
    \includegraphics[width=\linewidth, height=0.75\linewidth]{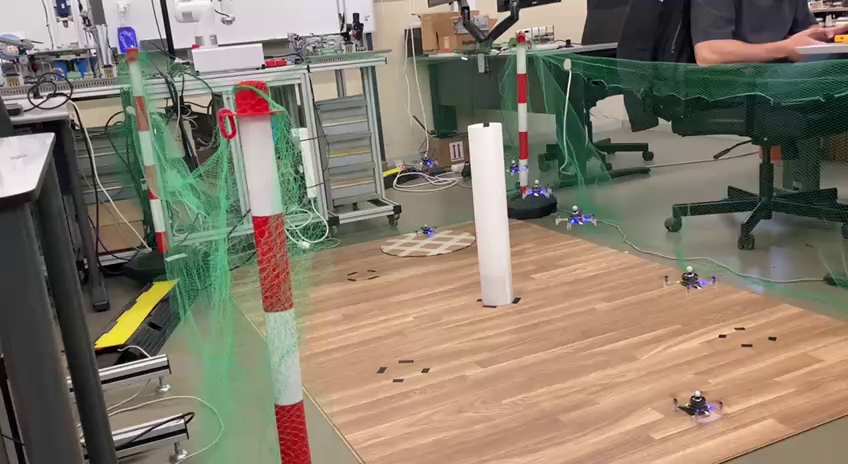} 
    \caption{}
  \end{subfigure}
  \begin{subfigure}{0.48\linewidth}
    \centering
    \includegraphics[width=\linewidth, height=0.75\linewidth]{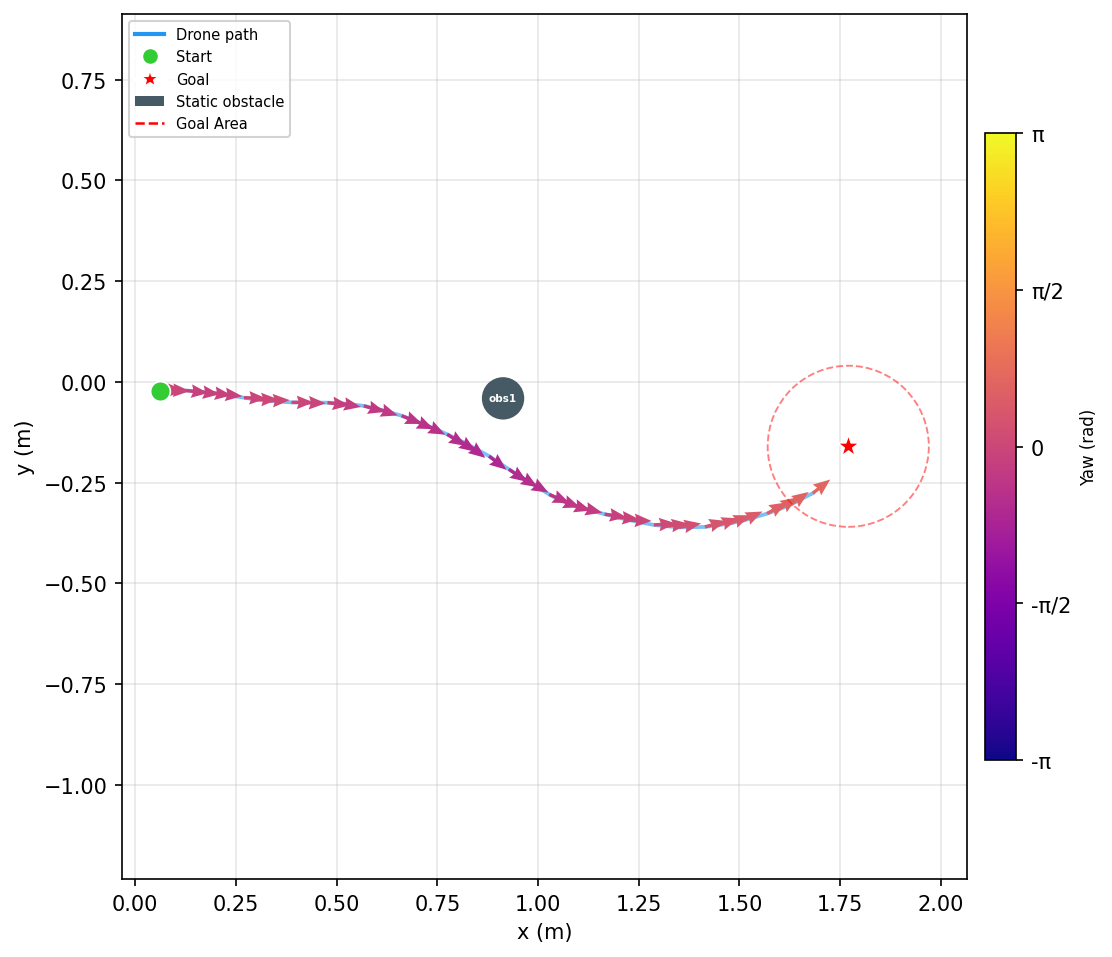} 
    \caption{}
  \end{subfigure}
 \caption{Experimental results with a single static obstacle: (a) Composite snapshot tracking the drone's sequential positions from takeoff to landing against a fixed background, and (b) the corresponding 2D trajectory plot of the full episode.}
  \label{fig:video_shot_one_static}
\end{figure}

\textbf{Four Static Obstacles:}
To further evaluate the robustness of the proposed method, experiments were conducted in a cluttered environment containing four static obstacles. To test adaptability, the target positions were varied between episodes. The quantitative results are summarized in Table~\ref{tab:real_test_info_four_static_obstacle}, and the corresponding trajectories are illustrated in Figure~\ref{fig:paths_episodes_four_static}.

In episode~a, the drone successfully reached the target. The resulting path is smooth, and the average clearance demonstrates that the generated trajectory remains safe and reliable, even within a more constrained free space.

Conversely, in episode~b, the drone collided with an obstacle and failed to reach the goal. We ran the standard method (without S-shape movement). This failure occurred because the obstacle fell within the range sensor's blind spot. Specifically, the range sensor is mounted between two propellers, leaving a structural gap between the sensor location and the outer edge of the propeller radius. Because the drone's sensing system could not detect the obstacle from this approach angle, it was unable to initiate an avoidance maneuver.

In episode~c, a similar scenario was tested, but the goal position was shifted slightly. This modification forced the drone to execute a sharper turn after passing the first obstacle to compensate for its yaw angle error. This rotational adjustment allowed the right-side sensor to partially detect the obstacle that caused the collision in episode~b, enabling the drone to successfully avoid it and reach the destination.

Compared to the single-obstacle scenario, navigation became significantly more challenging due to the reduced free space and increased maneuvering requirements. Consequently, the average planning time increased slightly to approximately 36--39~ms to accommodate the increased environmental complexity.

\begin{table}[h]
\centering
\caption{Quantitative results of real-world UAV navigation experiments in the presence of four static obstacles.}
\label{tab:real_test_info_four_static_obstacle}
\begin{tabularx}{\linewidth}{
>{\hsize=0.1\hsize}X
|>{\hsize=0.1\hsize}X
|>{\hsize=0.2\hsize}X
|>{\hsize=0.2\hsize}X
|>{\hsize=0.2\hsize}X
|>{\hsize=0.2\hsize}X
}
\hline
\textbf{Episode} & \textbf{Goal} & \textbf{Path Length (m)} & \textbf{Avg. Clearance (m)} & \textbf{Avg. Plan Time (ms)} & \textbf{Smoothness} \\
\hline
a & Yes & 1.302 & 0.270 & 36.00 & 5.959   \\
b & No & 1.410 & 0.235 & 40.841 & 7.034  \\
c & Yes & 1.817 & 0.253 & 38.98 & 7.237  \\
\hline
\end{tabularx}
\end{table}

\begin{figure}
  \centering
  \begin{subfigure}{0.32\linewidth}
    \centering
    \includegraphics[width=\linewidth, height=0.75\linewidth]{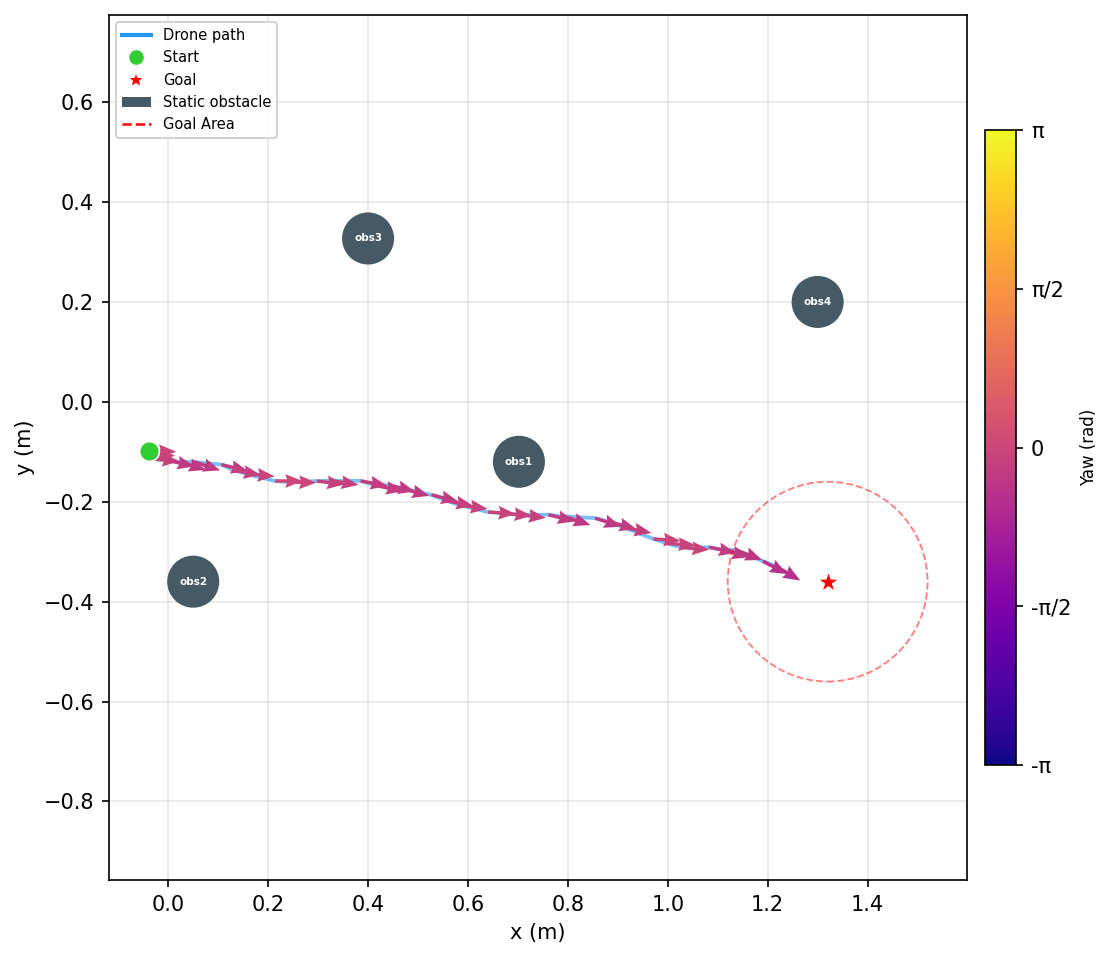} 
    \caption{}
  \end{subfigure}
  \begin{subfigure}{0.32\linewidth}
    \centering
    \includegraphics[width=\linewidth, height=0.75\linewidth]{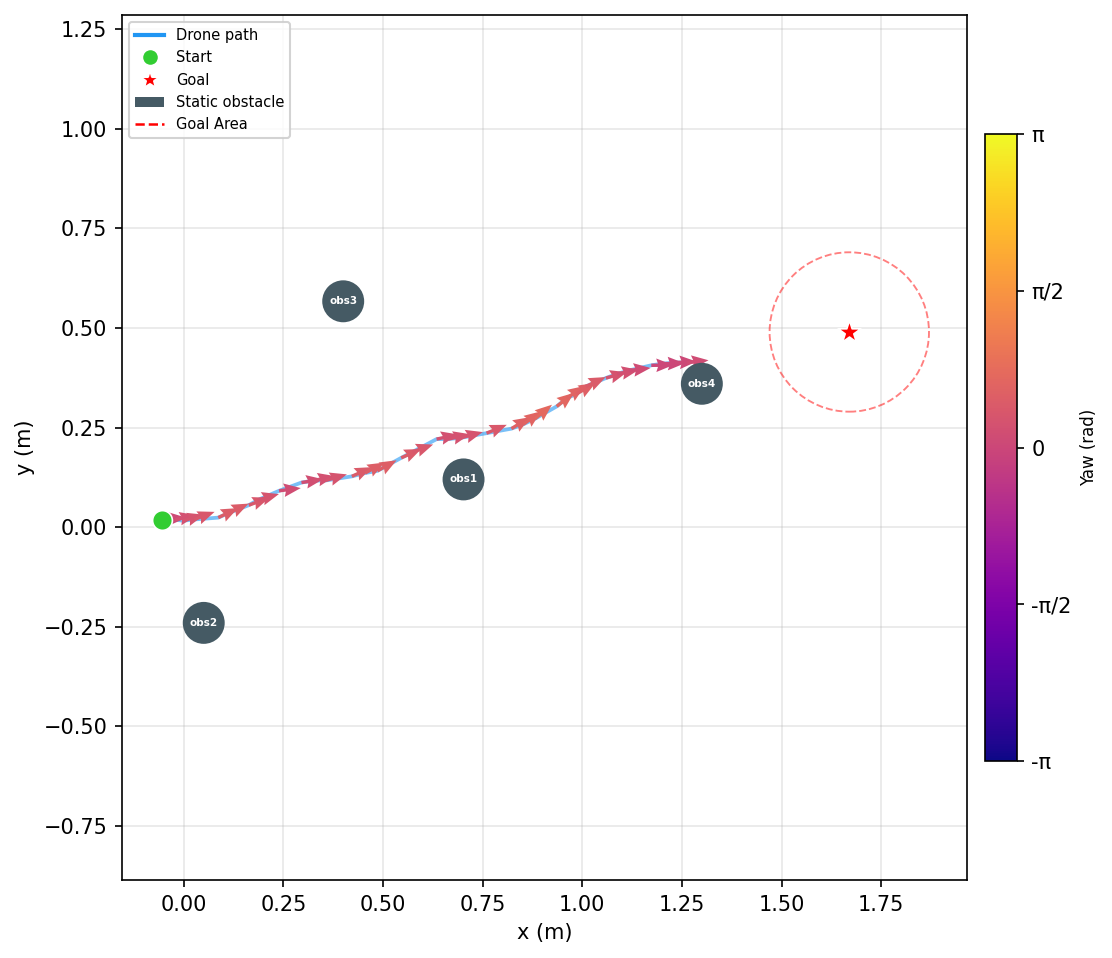} 
    \caption{}
  \end{subfigure}
  \begin{subfigure}{0.32\linewidth}
    \centering
    \includegraphics[width=\linewidth, height=0.75\linewidth]{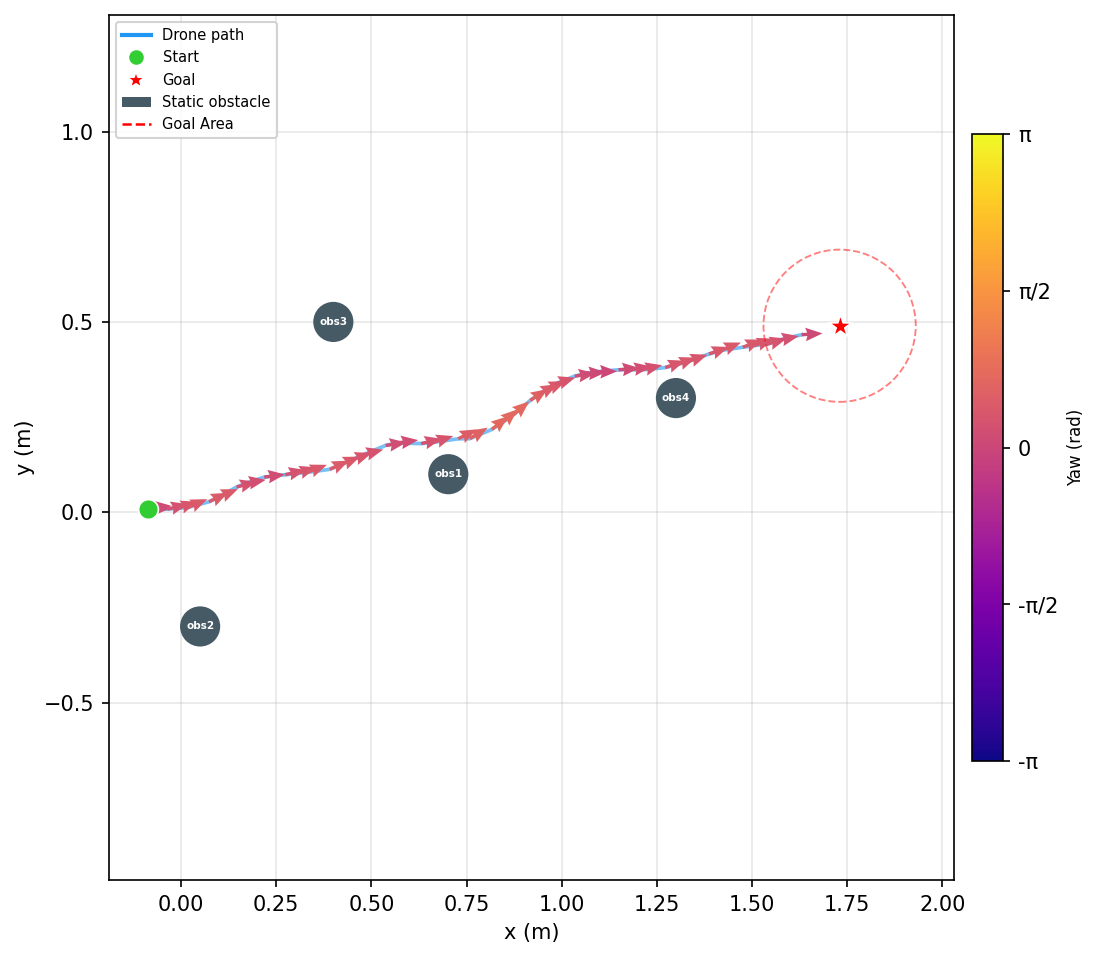} 
    \caption{}
  \end{subfigure}
  \caption{The plots of real-world UAV trajectories in the environment containing four static obstacles.}
  \label{fig:paths_episodes_four_static}
\end{figure}

Figure~\ref{fig:video_shot_four_static}a displays a composite snapshot of a successful flight navigating around four static obstacles, constructed by overlaying sequential video frames onto a fixed background to capture the full trajectory from takeoff to landing. Figure~\ref{fig:video_shot_four_static}b illustrates the corresponding trajectory data recorded for this episode.

\begin{figure}
  \centering
  \begin{subfigure}{0.48\linewidth}
    \centering
    \includegraphics[width=\linewidth, height=0.75\linewidth]{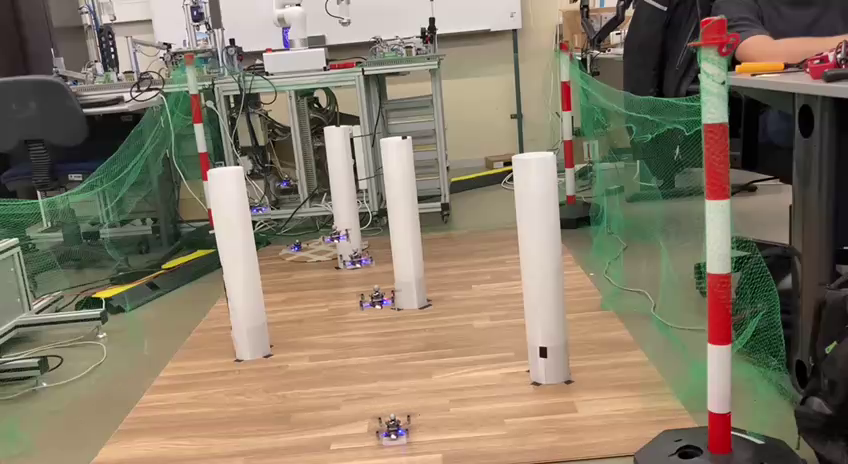} 
    \caption{}
  \end{subfigure}
  \begin{subfigure}{0.48\linewidth}
    \centering
    \includegraphics[width=\linewidth, height=0.75\linewidth]{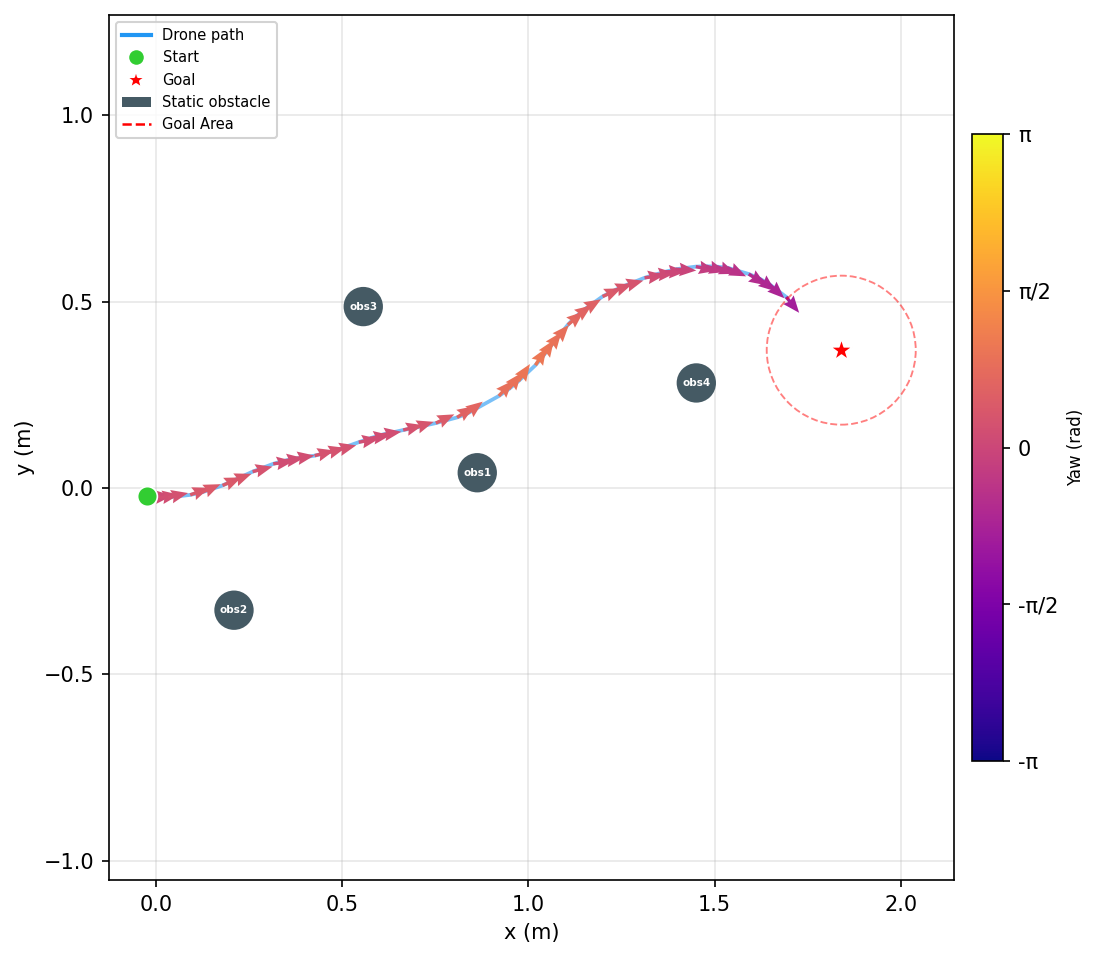} 
    \caption{}
  \end{subfigure}
 \caption{Experimental results with four static obstacles: (a) Composite snapshot tracking the drone's sequential positions from takeoff to landing against a fixed background, and (b) the corresponding 2D trajectory plot of the full episode.}
  \label{fig:video_shot_four_static}
\end{figure}

\textbf{Dynamic Obstacle:}
The final experiment evaluated the proposed approach in an environment containing a moving obstacle. The quantitative results are summarized in Table~\ref{tab:real_test_info_one_dynamic_obstacle}, and the corresponding trajectories are shown in Figure~\ref{fig:paths_episodes_one_dynamic}. In each subplot of this figure, the dynamic obstacle is designated by an orange-filled circle, and its workspace is represented by the surrounding dotted red region.

In episode~a, the target was located toward the negative $y$-axis, prompting the drone to pass along the right side of the obstacle. The drone initially rotated to minimize yaw angle error. As the front sensor detected the approaching obstacle, the UAV executed a sharper right turn. Under the subsequent activation of the left-side sensor, it maintained this rightward evasion before swinging back left to correct its heading error and successfully reach the goal.

In episode~b, the target was positioned along the positive $y$-axis. The drone initially rotated counter-clockwise to align its heading. As the dynamic obstacle (which exhibits a pendulum-like motion) swung directly into the drone's path, the drone sensed the imminent hazard. While a standard reactive model might continue turning left, the trained model dynamically initiated a sharp right turn. This decision was crucial: continuing along the leftward path would have placed the drone on an intercept course with the moving obstacle. This sudden, intelligent maneuver highlights the trained model's capacity to predict dynamic movements and successfully execute proactive collision avoidance.

In episode~c, as the drone bypassed the obstacle, the obstacle was swinging toward the positive $y$-axis. This movement naturally increased the separation distance between the two bodies, minimizing the activation of the left sensor and resulting in a less aggressive turning maneuver compared to episode~a. Furthermore, the larger initial distance between the start position and the target and the obstacle in this episode allowed for a noticeably smoother overall trajectory.

Due to the continuously changing environment, the processing overhead exhibited larger variations, with planning times ranging from 32.96~ms to 42.48~ms. Nonetheless, these results confirm that the proposed method delivers reliable, real-time obstacle avoidance in the presence of dynamic hazards.

\begin{table}[h]
\centering
\caption{Quantitative results of real-world UAV navigation experiments in the presence of one dynamic obstacle.}
\label{tab:real_test_info_one_dynamic_obstacle}
\begin{tabularx}{\linewidth}{
>{\hsize=0.1\hsize}X
|>{\hsize=0.2\hsize}X
|>{\hsize=0.2\hsize}X
|>{\hsize=0.2\hsize}X
|>{\hsize=0.2\hsize}X
}
\hline
\textbf{Episode} & \textbf{Path Length (m)} & \textbf{Avg. Clearance (m)} & \textbf{Avg. Plan Time (ms)} & \textbf{Smoothness} \\
\hline
a & 1.336 & 0.400 & 38.39 & 5.234 \\
b & 0.900 & 0.288 & 42.48 & 5.660 \\
c & 1.393 & 0.420 & 32.96 & 3.977 \\

\hline
\end{tabularx}
\end{table}

\begin{figure}
  \centering
  \begin{subfigure}{0.32\linewidth}
    \centering
    \includegraphics[width=\linewidth, height=0.75\linewidth]{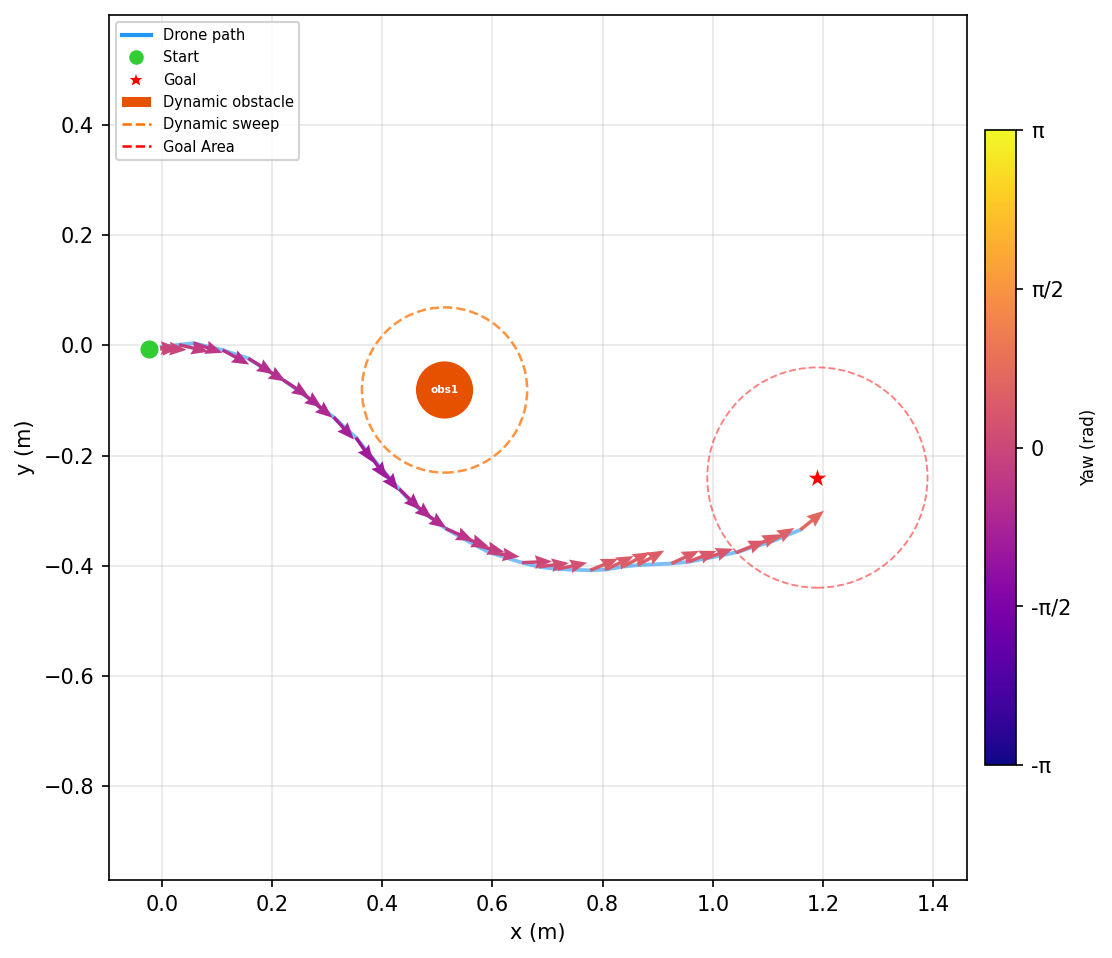} 
    \caption{}
  \end{subfigure}
  \begin{subfigure}{0.32\linewidth}
    \centering
    \includegraphics[width=\linewidth, height=0.75\linewidth]{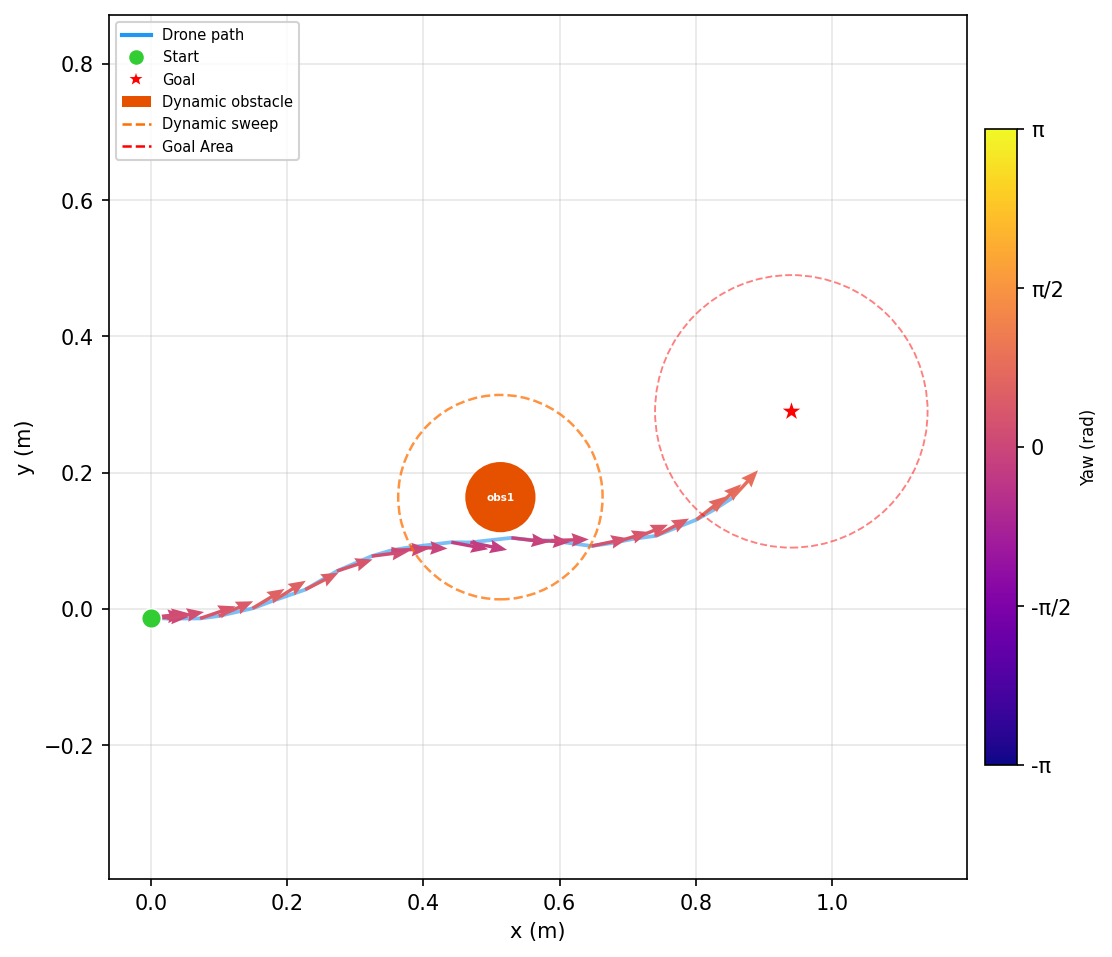} 
    \caption{}
  \end{subfigure}
  \begin{subfigure}{0.32\linewidth}
    \centering
    \includegraphics[width=\linewidth, height=0.75\linewidth]{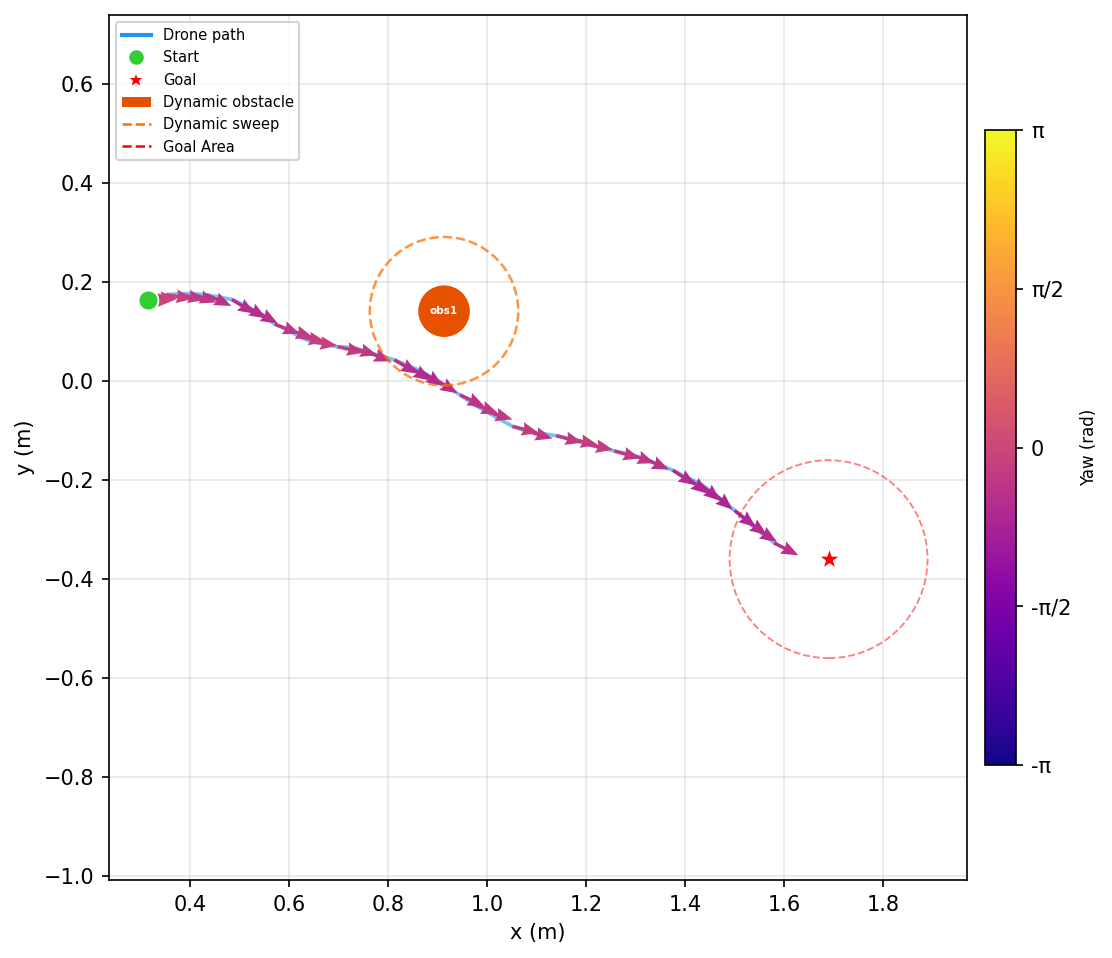} 
    \caption{}
  \end{subfigure}
 \caption{The plots of real-world UAV trajectories in the environment containing one dynamic obstacle.}
  \label{fig:paths_episodes_one_dynamic}
\end{figure}

Figure~\ref{fig:video_shot_one_dynamic}a displays a composite snapshot of a successful flight navigating around the dynamic obstacle. The snapshot is constructed by overlaying sequential video frames onto a fixed background to capture the full trajectory from takeoff to landing, with a blue curve marking the pendular workspace of the suspended obstacle. Figure~\ref{fig:video_shot_one_dynamic}b illustrates the corresponding trajectory data recorded for this episode.

\begin{figure}
  \centering
  \begin{subfigure}{0.48\linewidth}
    \centering
    \includegraphics[width=\linewidth, height=0.75\linewidth]{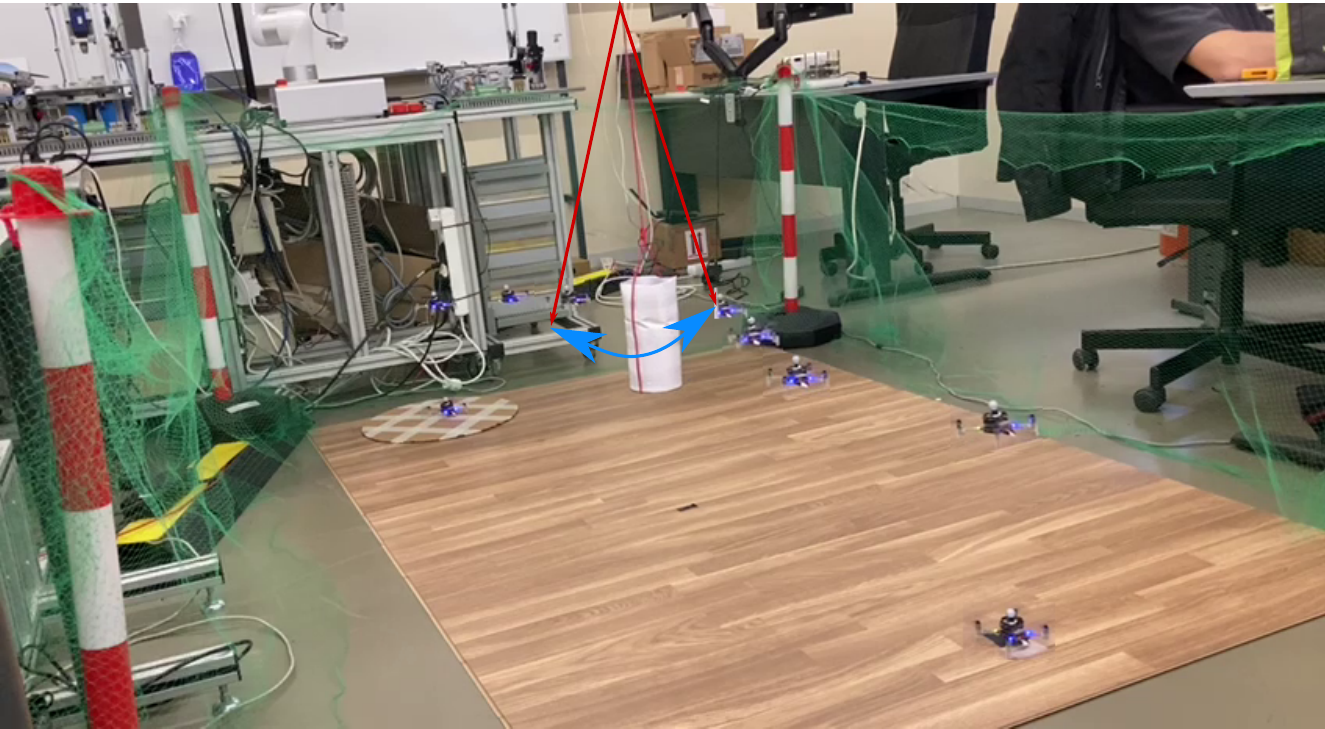} 
    \caption{}
  \end{subfigure}
  \begin{subfigure}{0.48\linewidth}
    \centering
    \includegraphics[width=\linewidth, height=0.75\linewidth]{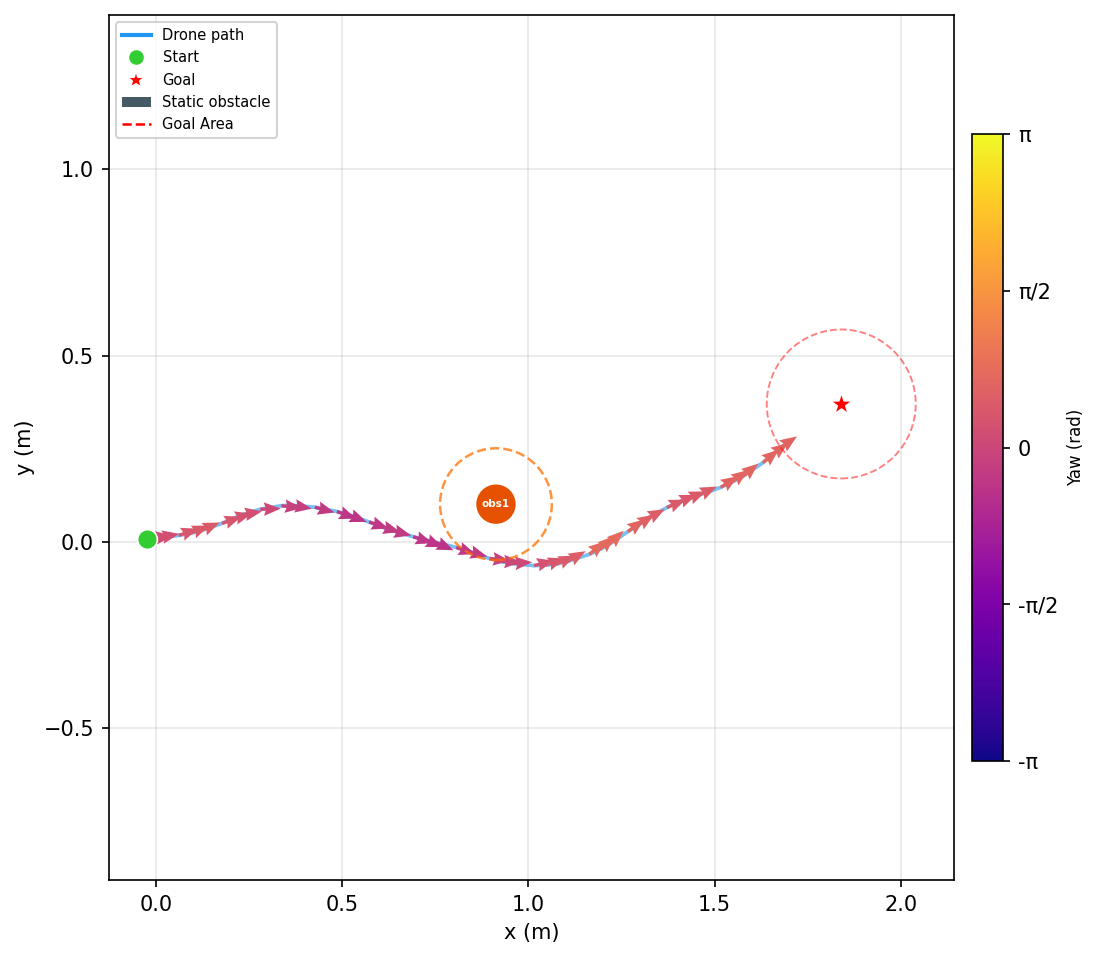} 
    \caption{}
  \end{subfigure}
 \caption{Experimental results with a dynamic obstacle: (a) Time-lapse composite snapshot tracking the drone's position relative to the obstacle's pendular workspace (indicated by the blue curve), and (b) the corresponding measured trajectory plot.}
  \label{fig:video_shot_one_dynamic}
\end{figure}

\textbf{Energy Consumption Comparison:}
Figure~\ref{fig:battary_usage_real} compares the battery voltage discharge profiles of the standard method and the enhanced method during real-world experiments.

The standard method produced a voltage range of 0.1478~V, while the enhanced method resulted in a larger voltage range of 0.2639~V, which is approximately 1.79 times higher. This indicates that the enhanced method consumed more energy during operation. The increased energy usage is related to the additional maneuvering required for increasing the environmental perception.

\begin{figure}
    \centering
    \includegraphics[width=0.75\linewidth]{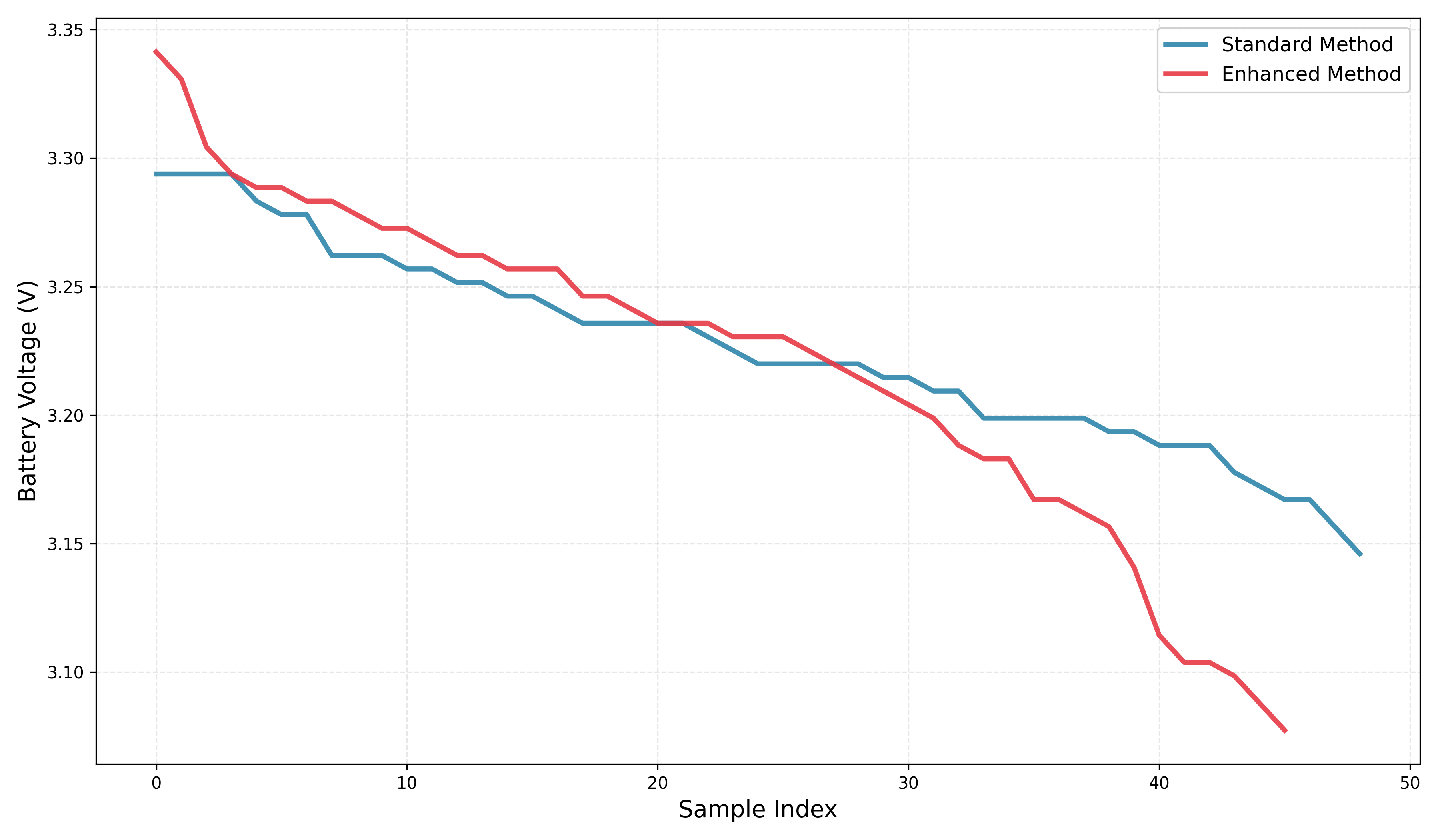}
    \caption{Drone Battery Voltage Discharge: Standard vs Enhanced Method}
    \label{fig:battary_usage_real}
\end{figure}

\subsection{Limitations and Future Work}\label{Limitations_and_Future_Work}

Although the proposed framework demonstrates strong performance in dynamic obstacle avoidance under sparse sensing, several limitations remain.

First, the current implementation is restricted to planar motion with constant altitude. Extending the framework to fully three-dimensional navigation would improve applicability in more complex aerial environments and enable operation in scenarios requiring vertical maneuvering.

Second, to match the constraints of the experimental hardware, the sensing configuration is limited to three forward-facing range measurements, resulting in a perception blind spot in the rear direction. While this is not an inherent limitation of the proposed framework, but rather of the sensing setup, it may affect safety in scenarios involving rear-approaching obstacles. As future work, additional sensors (e.g., a rear rangefinder) can be integrated, and the behavior grid map extended accordingly. Retraining the PPO policy under this augmented perception configuration would enable full directional awareness without modifying the underlying architecture.

Third, the real-world validation was conducted in relatively small indoor environments with limited obstacle diversity. Although these experiments demonstrate feasibility, larger-scale evaluations under more diverse and realistic conditions, including varying levels of environmental complexity, remain necessary to further assess generalization.

Fourth, while the yaw-dominated control structure is motivated by fixed-wing UAV constraints, the experimental validation was performed using a Crazyflie quadrotor platform. To bridge this gap, the quadrotor was intentionally constrained to yaw-based motion, enforcing nonholonomic behavior. While this provides a controlled validation setting, further experiments on platforms with inherently constrained dynamics (e.g., fixed-wing UAVs) would strengthen the applicability of the approach.

Finally, the current reward function relies on fixed empirically selected coefficients, which may not optimally balance navigation efficiency and safety across different environments. In particular, highly cluttered scenarios may require stronger emphasis on obstacle avoidance to maintain larger clearance. As future work, adaptive reward weighting mechanisms could be explored to dynamically adjust the relative importance of target progress, heading alignment, and obstacle avoidance based on environmental conditions, improving both safety and trajectory efficiency.

\end{color}




\section{Conclusion}\label{sec:conclusion}

This paper proposed a learning-based framework for UAV online motion planning under severe sensing constraints, combining a behavior grid map representation with a PPO-based control policy. The approach enables real-time navigation in dynamic environments using only sparse directional range measurements, without relying on dense perception systems such as LiDAR or cameras. This design makes the framework particularly suitable for resource-constrained UAV platforms.

\textcolor{black}{
The effectiveness of the proposed approach was demonstrated through systematic simulation studies across multiple environment sizes ($5 \times 5$ and $20 \times 20$ m$^2$) and congestion levels, as well as real-world experiments in four distinct scenarios. The results show that the framework consistently achieves reliable target-reaching behaviour while maintaining safe interaction with both static and dynamic obstacles. Compared with MPC and PPO-based baselines, the proposed method improves success rates in challenging environments while preserving real-time execution and maintaining a lightweight computational footprint. Ablation studies further highlight the role of key design components, including the grid representation, temporal update mechanism, and reward formulation.
}

\textcolor{black}{
Overall, the results indicate that effective navigation can be achieved under extreme partial observability through the combination of structured state representation and learning-based control. Future work will focus on extending the framework to fully three-dimensional motion, incorporating additional sensing to address directional limitations, and evaluating performance in larger-scale and multi-agent environments.
}

\bibliography{sn-bibliography}

\section*{Statements and Declarations}
\subsection*{Funding}
This work was conducted under the FCT doctoral scholarship (reference number 2023.01812.BD) and forms part of ongoing R\&D activities. Financial support was provided by the R\&D Unit SYSTEC – Systems and Technologies Center (UID/00147) and the Associated Laboratory ARISE – Advanced Production and Intelligent Systems (ref. LA/P/0112/2020), both funded by FCT/MECI through national funds.

\subsection*{Competing Interests}
The authors have no relevant financial or non-financial interests to disclose.

\subsection*{Author Contributions}
All authors contributed to the study conception and design. All authors read and approved the final manuscript.

\subsection*{Data Availability}
The data that support the findings of this study are available from the corresponding author upon reasonable request.

\end{document}